\documentclass{article}

\usepackage{microtype}
\usepackage{graphicx}
\usepackage{subfigure}
\usepackage{booktabs} 

\usepackage{hyperref}

\usepackage[accepted]{icml2024}

\usepackage{amsmath}
\usepackage{amssymb}
\usepackage{mathtools}
\usepackage{amsthm}

\usepackage[capitalize,noabbrev]{cleveref}

\usepackage{mylivemacros}
\usepackage{xcolor}
\usepackage{dsfont}
\usepackage{subfiles} 
\usepackage{bold-extra}
\usepackage{bm}

\theoremstyle{plain}
\newtheorem{theorem}{Theorem}[section]
\newtheorem{proposition}[theorem]{Proposition}
\newtheorem{lemma}[theorem]{Lemma}
\newtheorem{corollary}[theorem]{Corollary}
\theoremstyle{definition}
\newtheorem{definition}[theorem]{Definition}

\theoremstyle{remark}
\newtheorem{remark}[theorem]{Remark}

\usepackage[textsize=tiny]{todonotes}
\newcommand{\rs}{\ensuremath{\text{\tt RS}}}

\newcommand{\rspe}{\ensuremath{\text{\tt RS-PE}}}
\newcommand{\dsmpe}{\ensuremath{\text{\tt SIPS}}}
\newcommand{\ips}{\ensuremath{\text{\tt IPS}}}
\newcommand{\rsbpi}{\ensuremath{\text{\tt RS-BPI}}}
\newcommand{\grsbpi}{\ensuremath{\text{\tt G-RS-BPI}}}
\newcommand{\dsmbpi}{\ensuremath{\text{\tt SBPI}}}
\newcommand{\rsslin}{\ensuremath{\text{\tt RS-RMIN}}}

\newcommand{\floor}[1]{\left\lfloor #1 \right\rfloor}
\newcommand{\ceil}[1]{\left\lceil #1 \right\rceil}
\DeclareFontFamily{U}{mathx}{\hyphenchar\font45}
\DeclareFontShape{U}{mathx}{m}{n}{<-> mathx10}{}
\DeclareSymbolFont{mathx}{U}{mathx}{m}{n}
\DeclareMathAccent{\widebar}{0}{mathx}{"73}

\icmltitlerunning{Low-Rank Bandits via Tight Two-to-Infinity Singular Subspace Recovery}

\begin{document}

\twocolumn[
\icmltitle{Low-Rank Bandits via Tight Two-to-Infinity Singular Subspace Recovery}



\icmlsetsymbol{equal}{*}

\begin{icmlauthorlist}
\icmlauthor{Yassir Jedra}{equal,yyy}
\icmlauthor{William Réveillard}{equal,xxx}
\icmlauthor{Stefan Stojanovic}{equal,xxx}
\icmlauthor{Alexandre Proutiere}{xxx}
\end{icmlauthorlist}

\icmlaffiliation{yyy}{Laboratory for Information and Decision Systems, MIT, Cambridge, MA, USA}
\icmlaffiliation{xxx}{Division of Decision and Control Systems, KTH, Stockholm, Sweden}

\icmlcorrespondingauthor{Yassir, Jedra}{jedra@mit.edu}

\icmlkeywords{low-rank bandits, entrywise recovery, misspecified bandits}

\vskip 0.3in
]



\printAffiliationsAndNotice{\icmlEqualContribution} 

\begin{abstract}
We study contextual bandits with low-rank structure where, in each round, if the (context, arm) pair $(i,j)\in [m]\times [n]$ is selected, the learner observes a noisy sample of the $(i,j)$-th entry of an unknown low-rank reward matrix. Successive contexts are generated randomly in an i.i.d. manner and are revealed to the learner. For such bandits, we present efficient algorithms for policy evaluation, best policy identification and regret minimization. For policy evaluation and best policy identification, we show that our algorithms are nearly minimax optimal. For instance, the number of samples required to return an $\varepsilon$-optimal policy with probability at least $1-\delta$ typically\footref{fn:typically} scales as ${r(m+n)\over \varepsilon^2}\log(1/\delta)$. Our regret minimization algorithm enjoys minimax guarantees typically\footnote{\label{fn:typically}By "typically" here, we mean when the matrix $M$ and the (context, arm) distribution are homogeneous (refer to Definition \ref{def:homogeneous_M} for details). Note that under homogeneity of $M$, $m = \Theta(n)$.}
scaling as $r^{5/4}(m+n)^{3/4}\sqrt{T}$, which improves over existing algorithms. All the proposed algorithms consist of two phases: they first leverage spectral methods to estimate the left and right singular subspaces of the low-rank reward matrix. We show that these estimates enjoy tight error guarantees in the two-to-infinity norm. This in turn allows us to reformulate our problems as a misspecified linear bandit problem with dimension roughly $r(m+n)$ and misspecification controlled by the subspace recovery error, as well as to design the second phase of our algorithms efficiently.
\end{abstract}

\section{Introduction}

Stochastic Multi-Armed Bandits (MABs) \cite{lai1985} provide an efficient and natural framework for the analysis of the exploration-exploitation trade-off in sequential decision making and have been used in a variety of applications ranging from recommendation systems to clinical trials. When designing bandit algorithms, it is crucial to leverage as much as possible all information initially available to the learner. This information often comes as structural properties satisfied by the mapping of the arms to their rewards, see e.g. linear \cite{dani2008stochastic}, convex \cite{agarwal2011}, or unimodal \cite{Combes2014}. In this paper, we investigate bandit problems where the arm-to-reward mapping exhibits a {\it low-rank structure}. Such a structure has received a lot of attention recently \cite{kveton2017stochastic,jun2019bilinear,jang2021improved,bayati2022speed,kang2022efficient,pal2022online,lee2023context}. Despite these efforts, it remains unclear how one may optimally exploit low-rank structures and how much gains in terms of performance they bring. We make progress towards answering these questions.

We consider the following contextual low-rank bandit model, similar to that studied in \cite{bayati2022speed,pal2022online,lee2023context}. In each round, the learner first observes a random context $i\in [m]$, and then selects an arm $j\in [n]$. When the pair $(i,j)$ is selected, the learner observes a noisy sample of the $(i,j)$-th entry of a rank-$r$ matrix $M$, where $r$ is typically assumed to be much smaller than $m$ and $n$. For this low-rank bandit model, we address three learning tasks: policy evaluation, best policy identification, and regret minimization. Our contributions are as follows:

{\it (a)} We propose a generic method to design two-phase algorithms for low-rank bandit problems. In the first phase, we estimate the singular subspaces of the reward matrix $M$ using simple spectral methods. We establish tight upper bounds on the subspace recovery error in the two-to-infinity norm. In turn, these new error bounds allow us to reformulate our low-rank bandits as a misspecified contextual linear bandit problem with dimension $r(m+n)-r^2$ and controlled misspecification. In the second phase, the algorithm solves the resulting misspecified contextual linear bandit problem. The main contribution of this paper is to demonstrate that the above method, based on tight subspace recovery guarantees in the two-to-infinity norm, yields computationally efficient algorithms that either are nearly minimax optimal or outperform existing algorithms for the three considered learning tasks. 

{\it (b) Policy Evaluation (PE).} For this task, the objective is to estimate the average reward of a given {\it target} policy based on data generated from a fixed behavior policy. We derive instance-specific and minimax lower bounds on the sample complexity satisfied by any $(\varepsilon,\delta)$-PAC algorithms\footnote{These algorithms return an $\varepsilon$-accurate estimate of the value of the policy with certainty level $1-\delta$.}. The latter typically\footref{fn:typically} scales as ${m+n\over \varepsilon^2}\log(1/\delta)$. We leverage our method to devise \dsmpe\ (Spectral Importance Propensity Score) and \rspe \ (Recover Subspace for Policy Evaluation), two PE algorithms with nearly minimax optimal sample complexity. The second phase of \rspe\ refines the estimate of the reward matrix, using a regularized least-square-estimator applied to the misspecified linear bandit model obtained in the first phase. 

{\it (c) Best Policy Identification (BPI).} Here, the goal is to return an approximately optimal policy based on data gathered using a fixed or adaptive sampling strategy. Using the same estimates of the reward matrix as those used in the PE algorithms, we devise \dsmbpi\ (Spectral Best Policy Identification) and \rsbpi\ (RS for Best Policy Identification), two $(\varepsilon,\delta)$-PAC algorithms with nearly minimax optimal sample complexity, again typically scaling as ${m+n\over \varepsilon^2}\log(1/\delta)$. This significantly improves over existing algorithms whose sample complexity scales as ${m+n\over \varepsilon^{2+r}}\log(1/\delta)$ \cite{lee2023context}.   

{\it (d) Regret minimization.} The two-phase design method yields \rsslin \ (RS for Regret MINimization),  an algorithm with minimax regret guarantees typically scaling as $(m+n)^{3/4}\sqrt{T}$ over $T$ rounds. Surprisingly and to the best of our knowledge, \rsslin \ is the first algorithm enjoying minimax guarantees that are always strictly tighter than those achieved in unstructured bandits (in this case, the best guarantees scale as $\sqrt{mnT}$). In its second phase, \rsslin\ uses an extension of \textsc{SupLinUCB} \citep{chu2011contextual,takemura2021parameter} to solve the misspecified linear bandit problem derived at the end of the first phase.

\paragraph{Notation.} For a given matrix $M \in \RR^{m \times n}$, we denote its $i$-th row  by $M_{i,:}$, its $j$-th column by $M_{:, j}$, and its $(i,j)$-th entry by $M_{i,j}$. We denote by $\Vert M \Vert_\op$ its Euclidean operator norm, by $\Vert M \Vert_\F$ its Frobenius norm, by $\Vert M \Vert_{2\to \infty} = \max_{i \in [m]} \Vert M_{i,:} \Vert_2$ its two-to-infinity norm, by $\Vert M \Vert_{\max} = \max_{(i,j) \in [m]\times[n]} \vert M_{i,j} \vert$ its max-norm, by $\lambda_{\max}\left(M\right)$ (resp. $\lambda_{\min}\left(M\right)$) its maximal (resp. minimal) eigenvalue. For a given vector $x$, we denote its Euclidean norm by $\Vert x \Vert_2$, and by $\Vert x \Vert_\Lambda=\sqrt{x^\top\Lambda x}$ its Euclidean norm weighted by some positive definite matrix $\Lambda$. The notation $f(x) \lesssim g(x)$ (resp. $f(x) \gtrsim g(x)$) means that there exists a universal
constant $C > 0$ such that $f(x) \leq Cg(x)$ (resp. $f(x) \geq Cg(x)$) for all $x$. We write  $f(x) = \Theta\left(g(x)\right)$  when $f(x) \lesssim g(x)$ and $f(x) \gtrsim g(x)$.
We write $\operatorname{poly}\left(x\right)$ to denote a quantity that is upper bounded by a polynomial function of $x$. We also use $a\wedge b = \min(a,b)$ and $a\vee b = \max(a,b)$. Finally, for any $p, q \in (0,1)$, we define $\operatorname{kl}(p, q) = p\log(p/q) + (1-p)\log((1-p)/(1-q))$ as the KL-divergence between two Bernoulli distributions with mean $p$ and $q$, respectively.
 

\section{Related Work}

In this section, we discuss existing results for the three investigated learning tasks in contextual low-rank bandits, as well as for singular subspace recovery for low-rank matrices. Additional related work can be found in Appendix \ref{app:rel} (there, we discuss misspecified linear bandits and other models related to low-rank bandits).

\paragraph{Contextual low-rank bandits.} There has been some interest in settings similar to ours \citep{gentile2014online, gopalan2016low, pal2022online, lee2023context, pal2023optimal}, although mostly in the context of regret minimization. Nonetheless, progress on minimax guarantees has remained surprisingly limited. For example, \cite{pal2022online} proposed \textsc{OCTAL}, an algorithm that achieves a regret of order $O(\mathrm{polylog}(m+n)\sqrt{T})$ for rank-$1$ bandits, and a simple ETC-based algorithm achieving a regret of $O(\mathrm{polylog}(m+n)T^{3/4})$ for general rank $r$. However, their results assume that the learner observes $m$ entries per round, while in our setting, only one entry is observed. Recently, \cite{lee2023context} considered a setting closely related to ours, and referred to as \emph{context-lumpable bandits}. There, the rows of the reward matrix can be clustered into $r$ groups, within which the rows are identical. They establish minimax regret upper bounds of order $\widetilde{O}(\sqrt{r^3(m+n) T})$. They further extend their results to contextual low-rank bandits, but for these bandits, prove a regret upper bound of order $\widetilde{O}((m+n)^{\frac{1}{3r+2}} T^{\frac{3r+1}{3r+2}})$ (see Theorem 26 in \cite{lee2023context}). A cluster-like structure comparable to that of \cite{lee2023context} was also considered in \cite{pal2023optimal}, where the authors proposed an algorithm that attains a regret of order $\tilde{O}(\sqrt{\mathrm{poly}(r)(m+n)T})$. These works leave the existence of an algorithm with minimax regret of order $\widetilde{O}(\sqrt{\mathrm{poly}(r)(m+n) T})$ in contextual low-rank bandit as an open question. We do not fully answer this question, but propose an algorithm with regret scaling as  $\widetilde{O}(r^{5/4}(m+n)^{3/4}\sqrt{T})$, which improves over existing algorithms.


For the BPI task, \cite{lee2023context} exhibits sample complexity guarantees of order $\widetilde{O}(r\left(m+n\right)/\varepsilon^{2})$ in the context-lumpable case. For contextual low-rank bandits, the authors present an algorithm with sample complexity guarantees of order $\widetilde{O}\left(\left(m+n\right)/\varepsilon^{2+r}\right)$ (see their Theorem 25 and the discussion thereafter). We get guarantees similar to those they obtain for context-lumpable bandits, but for more general low-rank bandits. 
We finally mention the work of \cite{xi2023matrix} addressing the PE task in low-rank MDPs. When applied to contextual low-rank bandits (see their Subsection 4.3.2), the authors obtain a sharp estimation error term with an additional bias that does not vanish unless the behavior and target policy are identical; an assumption that we do not require.

\textbf{Singular subspace recovery in the two-to-infinity norm.}
It is not surprising that estimating the singular subspaces of $M$ is useful in the context of bandit problems with low-rank structure as it has been showcased in \citep{jun2019bilinear,lale2019stochastic, lu2021low, kang2022efficient}. However, to the best of our knowledge, only Frobenius norm guarantees have been used, and guarantees in the two-to-infinity norm have remained largely unexplored. The latter guarantees are often harder to obtain and progress towards obtaining them have only emerged recently \citep{eldridge2018unperturbed,fan2018eigenvector,cape2019two,abbe2020entrywise}.  Typically, under standard assumptions, such as bounded incoherence \citep{candes2010power, recht2011simpler}, these guarantees state that, for a matrix of size $m\times n$, the subspace estimation error in the two-to-infinity norm is smaller by a factor of $\sqrt{m+n}$ than the achievable bounds on the estimation error in the Frobenius norm. For this reason, we say that the subspace recovery error exhibits a \emph{delocalization} phenomenon \citep{rudelson2015delocalization}, i.e., the estimation error is spread out across $m+n$ directions (see also the survey \citep{chen2021spectral} and references therein).

The major difficulty towards using two-to-infinity norm guarantees is that their derivation requires stringent independence assumptions. Recently, \cite{stojanovic2024spectral} provided tools to relax such assumptions and to accommodate scenarios that are suitable for sequential decision making problems such as Markov decision processes and bandits. We make use of these tools to obtain tight guarantees on the singular subspace recovery in the two-to-infinity norm for low-rank bandits. These guarantees enable an effective reduction of the contextual low-rank bandit problem to the misspecified contextual linear bandit problem.

\section{Model and Learning Objectives}

We consider a stochastic bandit problem with low-rank structure. Specifically, we assume that the expected rewards are parametrized by a matrix $M\in \mathbb{R}^{m\times n}$ which has low rank $r\ll \min(m,n)$. Its SVD is $M = U \Sigma V^\top$, where $U \in \RR^{m \times r}$ (resp. $V \in \RR^{n\times r}$) contains the left (resp. right) singular vectors, and where $\Sigma = \mathrm{diag}(\sigma_1, \dots, \sigma_r)$ with $\sigma_1\ge \sigma_2\ge \ldots \ge \sigma_r$. The incoherence parameters of $M$ are defined as  $\mu(U) = \sqrt{m/r} \Vert U \Vert_{2 \to \infty}$ and $\mu(V) =  \sqrt{n/r} \Vert V \Vert_{2 \to \infty}$. We define $\mu = \mu(U) \vee \mu(V)$, and denote by $\kappa = \sigma_1/\sigma_r$ the condition number of $M$.


\noindent
 {\bf Contextual low-rank bandits.} In each round $t\ge 1$, the learner first observes a context $i_t$, selected in an i.i.d. manner and with distribution $\rho$ over $[m]$. She then selects an arm $j_t\in [n]$ potentially based on previous observations, and receives a reward $r_t=M_{i_t,j_t}+\xi_t$. The reward matrix $M$ is a priori unknown,
 and $(\xi_t)_{t\ge 1}$ is a sequence of i.i.d. zero-mean and $\sigma$-subgaussian random variables. We denote the (context, arm) distribution by $\omega.$

For a given randomized policy $\pi$, we define its policy value as $v^{\pi}:=\sum_{i,j}\omega^{\pi}_{i,j}M_{i,j}$ with $\omega^{\pi}_{i,j}:=\rho_i\pi\left(j|i\right)$ ($\pi(j|i)$ is the probability to select arm $j$ when observing context $i$). We further define the optimal policy $\pi^\star := \argmax_\pi v^\pi$ and denote its value by $v^\star$. We consider three learning tasks:
\begin{itemize}
    \item[\it (i)] \emph{Policy evaluation.} Given a target policy $\pi$, and assuming data is gathered under a policy\footnote{The arm $j$ is selected with probability $\pi^b(j|i)$ when the context $i$ is observed.} $\pi^b$,  we aim at designing an efficient estimator of its policy value $v^{\pi}$. We say that a PE estimator $\hat{v}^{\pi}$ is $(\varepsilon,\delta)$-PAC if $\mathbb{P}_{M}\left(|v^{\pi}-\hat{v}^\pi| \le \varepsilon \right) \ge 1- \delta$ for every rank-$r$ matrix $M \in \mathbb{R}^{m \times n}$. 
    \item[\it(ii)] \emph{Best policy identification.} 
    We aim to design an efficient algorithm to identify an $\varepsilon$-optimal policy. We say that a BPI algorithm is $\left(\varepsilon,\delta\right)$-PAC if it outputs a policy $\hat{\pi}: [m] \to [n] $ such that $\mathbb{P}_{M}\left( v^{\star} - v^{\hat{\pi}} \le \varepsilon\right) \ge 1- \delta$ for every rank-$r$ matrix $M \in \mathbb{R}^{m \times n}$.
\end{itemize}
We define the \textit{sample complexity} of a PE estimator or a BPI algorithm as the number of samples required to achieve an $\left(\varepsilon,\delta\right)$-PAC guarantee. 

\begin{itemize}
\item [\it(iii)]\emph{Regret minimization.} Here, the objective is to minimize the regret. The regret up to round $T$ of a sequential decision algorithm $\pi$ is defined as $R^\pi(T)=\sum_{t=1}^T \mathbb{E}[M_{i_t,\pi^\star(i_t)} - M_{i_t,j^\pi_t}]$ where $j^\pi_t$ is the arm selected under algorithm $\pi$ in round $t$.
\end{itemize}

Throughout the paper, we make the assumption that the learner is aware of upper bounds on $\kappa, \mu, \Vert M \Vert_{\max}$ and $1/(m\min_{i\in [m]}\rho_i) > 0$. Our results exhibit precise dependencies on all these parameters. Sometimes, to simplify the exposition of our results, we consider that the reward matrix $M$, the context distribution $\rho$ or the (context, arm) distribution $\omega$ are \textit{homogeneous} in the following sense:
\begin{definition}[Homogeneity]
\label{def:homogeneous_M}
The reward matrix $M$ is \textit{homogeneous} when 
$m=\Theta(n)$, $r=\Theta(1)$, $\mu =\Theta(1)$, $\kappa=\Theta(1)$. A distribution $p$ on a finite set ${\cal I}$ is  \textit{homogeneous} when $p_{\min}=\Theta\left(p_{\max}\right)$ for $p_{\min}=\min_{i\in {\cal I}} p_i$,  $p_{\max} = \max_{i\in {\cal I}} p_i$ \footnote{This implies $p_{\min}=\Theta\left(1/\vert \cal I \vert\right)$,  $p_{\max}=\Theta\left(1/\vert \cal I \vert\right)$. } .
\end{definition}

The notion of homogeneity allows us to have meaningful discussions but also constitutes a reasonable assumption. In particular, assuming that the matrix $M$ has bounded incoherence entails that the matrix $M$ can be recovered using fewer samples than $mn$ \cite{candes2008exact}. The bounded incoherence property is also related to the notion of spikiness of a matrix, which has to do with assuming that $\Vert M \Vert_{\max} / \sigma_1 = \Theta(1/\sqrt{mn})$ \cite{negahban2012restricted, mackey2015distributed}). Indeed, we can verify (see Lemma \ref{lem:spikiness}) that
       $\frac{1}{\sqrt{mn}}\le \frac{\Vert M \Vert_{\max}}{\sigma_1} \le \frac{\mu^2r}{\sqrt{mn}}$.
Note that this inequality does not mean that $\Vert M \Vert_{\max} = \Theta(1/\sqrt{mn})$ when $\mu = \Theta(1)$, because for most interesting settings $\sigma_{1}$ would scale as $\sqrt{mn}$. A naive example would be to take a matrix $M$ that has all entries equal to $1$. Observe then that $\textup{rank}(M) = 1$, $\Vert M \Vert_{\max} = 1$,  $\sigma_1 = \sqrt{mn}$, and $M$ satisfies the bounded incoherence assumption.


\section{Singular Subspace Recovery}\label{sec:subspacerecovery}

In the first phase of our algorithms for low-rank bandits, the singular subspaces of $M$ are estimated. In this section, we study the performance of this estimation procedure. We assume that the learner observes $T$ samples of noisy entries of $M$ chosen in an i.i.d. manner according to some distribution $\omega$ over $[m]\times [n]$, and that from these observations, she builds estimates of the singular subspaces of $M$. For simplicity, we assume that $\omega$ is known to the learner. This is without loss of generality in our settings\footnote{Indeed, we can estimate the context distribution $\rho$ using only $\widetilde{O}(1/\rho_{\min})$ samples, and select $j$ uniformly at random (refer to Appendix \ref{subsec:knowledge-of-context})}. We also assume that $\omega_{\min}=\min_{(i,j)\in [m]\times [n]}\omega_{i,j}>0$. 

 
\subsection{Subspace Estimation}\label{subsec:subest} 
To estimate the singular subspaces of $M$, we construct $\widetilde{M}$ as: for all $(i,j)\in [m]\times [n]$,
\begin{align}
\label{eq:def_Mtilde_uniform}
    \widetilde{M}_{i,j} =  \frac{1}{T \omega_{i,j}} \sum_{t=1}^T (M_{i_t, j_t} + \xi_t) \indicator_{\lbrace (i_t, j_t) = (i,j)\rbrace}.
\end{align}
Then, we let $\widehat{M}$ be the best $r$-rank approximation of $\widetilde{M}$. More precisely, we write from the SVD of $\widetilde{M}$:
\begin{align}\label{eq:SVD_hat_M}
\widehat{M} = \widehat{U}\widehat{\Sigma} \widehat{V}^\top,    
\end{align}
where $\widehat{\Sigma}$ is the diagonal matrix containing the $r$ largest singular values of $\widetilde{M}$. $\widehat{U} \in \RR^{m\times r}$ contains their corresponding $r$ left singular vectors, and $\widehat{V} \in \RR^{n\times r}$ their corresponding $r$ right singular vectors. By Eckart-Young-Mirsky's theorem, the matrix $\widehat{M}$ is the best $r$-rank approximation of $\widetilde{M}$. We use $\widehat{U}$ and $\widehat{V}$ as our estimates of the singular subspaces spanned by $U$ and $V$, respectively.

In what follows, we denote by $\widehat{U}_{\perp} \in\RR^{m\times (m-r)}$ a matrix made of orthonormal vectors completing $\widehat{U}$ so as to obtain an orthonormal basis of $\mathbb{R}^m$. $\widehat{V}_\perp \in \RR^{n\times (n-r)}$ is constructed similarly.

\subsection{Guarantees in the Two-to-Infinity Norm} 
\label{subsec:main_guarantees_two_to_inf}


We wish to examine the performance of our subspace recovery method in the two-to-infinity norm. More precisely, we denote the left and right singular subspace recovery errors in this norm by 
\begin{align*}
    d_{2\to\infty}(U, \widehat{U}) = \Vert U U^\top - \widehat{U} \widehat{U}^\top \Vert_{2 \to \infty}, \\
    d_{2\to\infty}(V, \widehat{V}) = \Vert V V^\top - \widehat{V} \widehat{V}^\top \Vert_{2 \to \infty}.
\end{align*}
Now, we present the main guarantee on these errors, which is adapted from \cite{stojanovic2024spectral}.  
\begin{theorem}[Subspace recovery in $\Vert \cdot \Vert_{2 \to \infty}$]\label{thm:recovery-two-to-infinity-norm}
    Let us define $\epsilon_{\textup{Sub-Rec}} := \max( d_{2\to\infty}(U,\widehat{U}), d_{2\to\infty}(V,\widehat{V}))$ and $L := \Vert M \Vert_{\max} \vee \sigma$. For any $\delta \in (0,1)$, the following event: 
    \begin{align}\label{eq:error-rate}
    \epsilon_{\textup{Sub-Rec}} \lesssim  \sqrt{\frac{L^2\mu^2\kappa^2 r (m+n) }{\sigma_r^2 T\omega_{\min}(m\wedge n)}\log^3\left(\frac{(m+n)T}{\delta}\right)}
    \end{align}
    holds with probability at least $1- \delta$, provided that 
    \begin{align}\label{eq:sc}
    T \gtrsim \frac{L^2 (m+n)}{\sigma_r^2 \omega_{\min}} \log^3\left( \frac{(m+n)T}{\delta}\right).
    \end{align}  
\end{theorem}

The proof of Theorem \ref{thm:recovery-two-to-infinity-norm} is sketched in Appendix \ref{sec:app_sing_sub_rec}. The analysis is technically involved and relies on combining the so-called leave-one-out argument (see \cite{chen2021spectral} and references therein) together with a Poisson approximation technique \cite{stojanovic2024spectral}.

From Theorem \ref{thm:recovery-two-to-infinity-norm}, we can immediately verify that if the matrix $M$ and the sampling distribution $\omega$ are homogeneous, then the subspace recovery error bound in the two-to-infinity norm is smaller by a factor of $\sqrt{m+n}$ than any achievable subspace recovery error bound in the Frobenius norm with our estimates of the singular subspaces.
Indeed, assume additionally and only for simplicity that
$\sigma \lesssim \Vert M\Vert_{\max}$, then after only $\widetilde{\Omega}(m+n)$ observations, we can recover the singular subspaces with an error rate in the $\Vert \cdot \Vert_{2 \to \infty}$ norm that scales as $\widetilde{O}(1/\sqrt{T})$ with high probability. To see that, note that when $\sigma \lesssim \Vert M\Vert_{\max}$, then according to Lemma \ref{lem:spikiness}, we have $L /\sigma_r = \Theta(1/\sqrt{mn})$. The recovery error in $\Vert \cdot \Vert_\F$ would typically scale as $\widetilde{O}(\sqrt{(m+n)/T})$ (see Lemma \ref{lemma:subspace_Frobenius} which follows from classical arguments as in \cite{jun2019bilinear}).
This suggests that the subspace recovery error (seen as a matrix) is {\it delocalized}, i.e., spread out along $m+n$ directions. Thus, the provided two-to-infinity norm guarantees offer a finer and more precise control over the subspace recovery error than the Frobenius norm guarantees. We take advantage of this when devising algorithms for contextual low-rank bandits.

It is worth mentioning that the obtained error rate \eqref{eq:error-rate}  in Theorem \ref{thm:recovery-two-to-infinity-norm} exhibits a dependence on $L = \Vert M \Vert_{\max} \vee \sigma$ instead of only $\sigma$. This means that  
 the proposed guarantees do not suggest exact subspace recovery when $\sigma \to  0$.  This is to be expected because we rely on a spectral method that consists in truncating the SVD of $\widetilde{M}$. Indeed, such a method typically suffers from a bias that comes from setting unobserved entries to zero, as can be remarked in \eqref{eq:def_Mtilde_uniform} (see also \cite{chen2021spectral}). As we shall see 
 in the coming sections, 
 this affects the final bounds we obtain in all three learning tasks. Nonetheless, we provide a discussion in Appendix \ref{app:debiased_regret} on how we may improve this dependence, which may be of independent interest. Indeed, one may resort to ideas based on nuclear norm penalization which have been shown recently to enjoy tight entry-wise matrix estimation guarantees \cite{chen2020noisy}. However, one has to take further care to ensure that the results based on such ideas would extend to our dependent noise setting. Finally, let us also mention that the dependence on $\sigma$, through $L$, in the sample complexity \eqref{eq:sc} provided in Theorem \ref{thm:recovery-two-to-infinity-norm} is to be expected and cannot be further improved.

\section{Reduction to Misspecified Linear Bandits}
\label{subsec:reductions}

In this section, we assume that we use $T_1$ reward observations to estimate the singular subspaces $U$ and $V$ of $M$ as described in the previous section. These $T_1$ first observations will constitute the first phase of our low-rank bandit algorithms. Based on the estimates of $U$ and $V$, one could, as in \citep{jun2019bilinear, kang2022efficient}, reduce the problem faced by the learner in the remaining $(T-T_1)$ rounds to an \emph{almost low-dimensional linear bandit}; see Appendix \ref{subsec:main_reduction_almost_lowd} for a detailed discussion. The second phase of low-rank bandit algorithms would then just consist in applying an algorithm for such linear bandits. We explain in Appendix \ref{subsec:main_reduction_almost_lowd} that for regret minimization, this approach would however not lead to the tightest regret guarantees possible. Instead, we establish that the second phase of low-rank bandit algorithms can be reduced to solving a {\it misspecified} linear bandit with misspecification controlled using our subspace recovery error guarantees in the two-to-infinity norm. This new approach yields improved regret upper bounds, but also allows us to devise efficient algorithms for PE and BPI. In what follows, we make the reduction to a misspecified linear bandit precise.



First observe that, following a similar transformation of the so-called left and right features as in \citep{jun2019bilinear}, we may express any entry of the matrix $M$ as: 
\begin{align}
     &M_{ij} = e_i^\top \widehat{U} (\widehat{U}^\top M \widehat{V} ) \widehat{V}^\top e_j + e_i^\top \widehat{U}(\widehat{U}^\top M \widehat{V}_\perp  ) \widehat{V}^\top_\perp  e_j \nonumber \\ &+ e_i^\top \widehat{U}_\perp(\widehat{U}^\top_\perp M \widehat{V} ) \widehat{V}^\top  e_j + e_i^\top \widehat{U}_\perp(\widehat{U}^\top_\perp M \widehat{V}_\perp ) \widehat{V}^\top_\perp  e_j.
     \label{eq:M_decomposition}
\end{align}

From (\ref{eq:M_decomposition}), we can write that for any $(i,j)\in [m]\times [n]$, 
\begin{align} \label{eq:miss-LB}
    M_{i,j}= \phi_{i,j}^{\top} \theta + \epsilon_{i,j},
\end{align}
where $\epsilon_{i,j} = e_i^\top \widehat{U}_\perp(\widehat{U}^\top_\perp M \widehat{V}_\perp ) \widehat{V}^\top_\perp  e_j$, and 
\begin{align}\label{eq:phi}
    \phi_{i,j} = \begin{bmatrix}\textup{vec}( \widehat{U}^\top e_i e_j^\top\widehat{V} ) \\
    \textup{vec}( \widehat{U}^\top e_i e_j^\top\widehat{V}_\perp )\\
        \textup{vec}(\widehat{U}_\perp^\top e_i e_j^\top\widehat{V} ) \\
    \end{bmatrix}\!, \theta    = \begin{bmatrix}
        \textup{vec}(\widehat{U}^\top M \widehat{V} ) \\
        \textup{vec}(\widehat{U}^\top M  \widehat{V}_\perp ) \\
        \textup{vec}(\widehat{U}_\perp^\top  M \widehat{V})
    \end{bmatrix}\!.
\end{align} 

With this interpretation, we obtain a misspecified linear bandit of dimension $d:=r(m+n)-r^2$. The benefit of this reduction is that we can characterize the misspecification $\epsilon_{\max}:=\max_{(i,j) \in [m]\times [n]} \vert \epsilon_{i,j} \vert$ using our subspace recovery guarantees in the two-to-infinity norm. Indeed, we establish:

\begin{corollary}[Misspecification error]
    Let $\delta \in (0,1)$, $T_1 > 0$ be the number of samples observed by the sampling distribution $\omega$. 
    Then, the event 
    \begin{align*}
         \epsilon_{\max} \lesssim  \frac{L^2\mu^2\kappa^3 r (m+n) }{\sigma_r T_1\omega_{\min}(m\wedge n)}\log^3\left(\frac{(m+n)T_1}{\delta}\right) 
    \end{align*}
    holds with probability at least $1-\delta$,
    provided that 
    \begin{align*}
    T_1 \gtrsim \frac{L^2 (m+n)}{\sigma_r^2 \omega_{\min}} \log^3\left( \frac{(m+n)T_1}{\delta}\right).
    \end{align*} 
 \label{corr:misspecification}
\end{corollary}

Corollary \ref{corr:misspecification}  follows immediately from Theorem \ref{thm:recovery-two-to-infinity-norm}. To see that, observe that we have, for all $(i,j)\in [m]\times[n]$,
\begin{align}
     \vert \epsilon_{i,j}\vert &  = \vert e_i^\top \widehat{U}_\perp(\widehat{U}^\top_\perp M \widehat{V}_\perp ) \widehat{V}^\top_\perp  e_j \vert \nonumber \\
    & \overset{(a)}{=}  \vert e_i^\top ( UU^\top - \widehat{U}\widehat{U}^\top)M (VV^\top - \widehat{V}\widehat{V}^\top) e_{j} \vert  \nonumber \\ 
        &  \overset{(b)}{\le} d_{2\to\infty}(U, \widehat{U})  d_{2\to\infty}(V, \widehat{V}) \Vert M \Vert_\op  \nonumber \\
        &\overset{(c)}{\le} \epsilon_{\textup{Sub-Rec}}^2 \kappa \sigma_r(M) 
        \label{eq:eps_bound_pert_bound}
\end{align}
where equality (a) follows from the fact that $ \widehat{U}_\perp\widehat{U}_\perp^\top M = (I - \widehat{U}\widehat{U}^\top) M$ since $\widehat{U}\widehat{U}^\top + \widehat{U}_\perp\widehat{U}_\perp^\top = I$, and similarly for $M \widehat{V}_\perp \widehat{V}_\perp^\top$, inequality (b) follows from $\vert x^\top M y \vert \le \Vert x\Vert_2 \Vert y\Vert_2 \Vert M \Vert_{\textup{op}}$ for any vectors $x \in \mathbb{R}^{m}, y \in \mathbb{R}^{n}$, and inequality (c) follows from the definition of the condition number $\kappa$. 

We remark that when $M$ and $\omega$ are homogeneous, we have $\epsilon_{\max} = \widetilde{O}((m+n)/T_1)$ due to the quadratic dependence on $\epsilon_{\textup{Sub-Rec}}$ appearing in \eqref{eq:eps_bound_pert_bound}.

\section{Policy Evaluation}\label{sec:pe}

In this section, we present algorithms for PE and analyze their performance. The first algorithm, referred to as \dsmpe\ (Spectral Importance Propensity Score), consists of a single phase and exploits all samples to construct a rank-$r$ estimate of $M$. The value of the target policy is directly obtained using this estimate. We show that, when $M$ and $\omega$ are homogeneous, \dsmpe\ exhibits a minimax optimal sample complexity w.r.t. the matrix size $(m,n)$ and the accuracy level $\varepsilon$, but not w.r.t. the confidence level $\delta$. The design of the second algorithm, \rspe\ (Recover Subspace for Policy Evaluation), leverages our two-phase approach. In its second phase, \rspe\ refines the estimate of $M$ via a regularized least squares estimator applied to the misspecified linear bandit model obtained at the end of the first phase.  Applying this two-phase approach has significant benefits: (i) \rspe\ enjoys instance-dependent sample complexity guarantees; 
(ii) When $M$ and $\omega$ are homogeneous, it is nearly minimax optimal w.r.t. the parameters $m,n$, $\varepsilon$ and $\delta$; 
(iii) as shown in Appendix \ref{subsec:pe-experiments}, it outperforms \dsmpe\ in numerical experiments. We conclude this section by deriving instance-dependent and minimax sample complexity lower bounds.

Throughout this section, we assume that the observations are gathered in an i.i.d. manner using a behavior policy $\pi^b$, and we wish to estimate the value $v^{\pi}$ of a target policy $\pi$. To simplify the notation, we abbreviate the (context, arm) distribution $\omega^{\pi^b}$ by $\omega$.


\subsection{Spectral Importance Propensity Score}\label{subsec:dsm-pe}


A natural approach to PE consists in directly constructing $\widehat{M}$, a rank-$r$ estimate of $M$ obtained via the spectral method described in \eqref{eq:SVD_hat_M}, using all of the $T$ available samples gathered under the (context, arm) distribution $\omega$. We estimate the value of the target policy $\pi$ by:
\begin{align}\label{eq:est-IP-SM}
    \hat{v}^\pi_{\dsmpe} :=\sum_{(i,j) \in [m]\times [n]}\omega^{\pi}_{i,j}\widehat{M}_{i,j}.
\end{align}
The estimator $\hat{v}^\pi_{\dsmpe}$ can be interpreted as an Importance Propensity Score (IPS) estimator that takes into account the low-rank structure of $M$. Indeed, in the absence of such low-rank structure, i.e., $r = m\wedge n$, we see that our policy value estimator reduces to the classical \ips\ estimator \cite{wang2017optimal}. 

\begin{theorem}\label{thm:dsm-pe-error}
Let $\varepsilon \in \left(0,\Vert M \Vert_{\max}\right)$ and $\delta \in \left(0,1\right)$. For any target policy $\pi$, $\vert v^{\pi}-\hat{v}^{\pi}_{\dsmpe} \vert \leq \varepsilon$ with probability larger than $1-\delta$ as soon as 
\begin{align*}
    T \gtrsim \frac{L^2\mu^{6} \kappa^4 r^3  (m+n)}{\omega_{\min} (m\wedge n)^2 \varepsilon^2}\log^3\!\left(  \frac{(m+n)T}{\delta}\right)\!.
\end{align*}
\end{theorem}
The proof of this theorem, presented in Appendix \ref{subsec:proof-dsm-pe-error}, relies on deriving entry-wise guarantees for the estimate $\widehat{M}$. These guarantees follow from the two-to-infinity guarantees provided in Theorem \ref{thm:recovery-two-to-infinity-norm}. Observe that when both $M$ and $\omega$ are homogeneous, the sample complexity guarantees derived in Theorem \ref{thm:dsm-pe-error} scale as ${L^2(m+n)\over \varepsilon^2}\log^3(1/\delta)$. The dependence in $L$ is a consequence of the bias induced by setting the unobserved entries of $M$ to $0$ in the definition \eqref{eq:def_Mtilde_uniform} of $\widetilde{M}$, as explained in Section \ref{subsec:main_guarantees_two_to_inf}.

Note that a naive control by the Frobenius norm would not be sufficient to obtain an error guarantee that scales with $\sqrt{m+n}$. Indeed, by the Cauchy-Schwartz inequality, $\vert v^{\pi}-\hat{v}^{\pi}_{\dsmpe} \vert \leq \Vert \omega^{\pi} \Vert_{F} \Vert M- \widehat{M} \Vert_{F}.$  If $\omega$ is homogeneous, $\Vert M- \widehat{M} \Vert_{F}$ typically scales with $\sqrt{mn\left(m+n\right)/T}$, and $\Vert \omega^{\pi} \Vert_{F}=\Theta\left(1/\sqrt{m}\right)$ when $\pi$ is deterministic. We consequently obtain $\vert v^{\pi}-\hat{v}^{\pi}_{\dsmpe} \vert \lesssim \sqrt{n\left(m+n\right)/T}$, which does not improve over existing PE error bounds for contextual bandits with no structure \cite{yin2020asymptotically}.

\subsection{Algorithm via the Two-Phase Approach}\label{subsec:rs-pe}

Next, we present $\rspe$ (Recover Subspace for Policy Evaluation), a two-phase algorithm that proceeds as follows. (i) In the first phase, we recover the subspaces by using the first $T_1$ samples to construct the estimates $\widehat{U}$ and $\widehat{V}$ as described in \eqref{eq:M_decomposition}. (ii) In the second phase, using $\widehat{U}$ and $\widehat{V}$, we reduce the problem to a {misspecified linear bandit} as described in \eqref{eq:miss-LB}, and then run least squares estimation with the remaining $T - T_1$ samples to construct  
\begin{equation}\label{eq:rspe-lse}
    \hat{\theta} = \left( \sum_{t=t_1}^T \phi_{i_t, j_t}\phi_{i_t, j_t}^{\top} + \tau I_d \right)^{-1} \left( \sum_{t=t_1}^{T} r_t  \phi_{i_t,j_t}\right),
\end{equation}
where $\tau>0$ is a regularization parameter, $t_1 = T_1 + 1$, and $d = r(m+n) - r^2$.
Finally, the estimated value of the target policy $\pi$ is: 
\begin{align}\label{eq:rspe-pv}
    \hat{v}_{\rspe}^\pi :=\sum_{(i,j) \in [m]\times [n]} \omega^{\pi}_{i,j} \phi_{i,j}^\top \hat{\theta}.
\end{align}


To analyze $\rspe$, we introduce the following instance-dependent quantities. Define, for all $(i,j) \in [m]\times [n]$, 
$$
\psi_{i,j}  = \begin{bmatrix}\textup{vec}( U^\top e_i e_j^\top V ) \\
    \textup{vec}( U^\top e_i e_j^\top V_\perp )\\
        \textup{vec}(U_\perp^\top e_i e_j^\top V ) \\
    \end{bmatrix},
$$
and $\psi_\pi = \sum_{(i,j) \in [m]\times[n]} \omega_{i,j}^\pi \psi_{i,j}$. 

\begin{theorem}\label{thm:rs-pe-error} 
Let $\varepsilon \in \left(0,\Vert M \Vert_{\max}\right)$ and $ \delta \in \left(0,1\right)$. With the choices $T_1 = \lfloor T/2 \rfloor$ and $\tau \le r(m\wedge n)^{-1} (\omega_{\min}/\omega_{\max}) \log(16d/\delta)$, for any target policy $\pi$,  $\vert v^{\pi}-\hat{v}^\pi_{\rspe} \vert \leq \varepsilon$ with probability larger than $1-\delta$ as soon as
\begin{align*}
    T \gtrsim \frac{\sigma^2\Vert \psi_\pi \Vert^2_2}{\omega_{\min}\,\varepsilon^2}\log\left(\frac{e}{\delta}\right) + \frac{K}{\varepsilon^{4/3}}\log^{3}\!\left( \frac{(m+n)T}{\delta}\right)\!, 
\end{align*}
where $K$ depends polynomially on the model parameters\footnote{Refer to \eqref{eq:def_of_K1} in Appendix \ref{subsec:proof-rs-pe-error} for a complete expression.} $\Vert M\Vert_{\max},\sigma, \mu, \kappa, r, m, n, (\omega_{\min} mn )^{-1}$. 
\end{theorem} 

We present in Appendix \ref{subsec:proof-rs-pe-error} a stronger version of Theorem \ref{thm:rs-pe-error} and its proof. The latter relies on intermediate results pertaining to the PE task in contextual linear bandits, for which we present in Appendix \ref{app:linear} a complete analysis: we derive an instance-dependent sample complexity lower bound and an algorithm matching this limit. 


Note that the guarantees for \rspe\ are instance-dependent through $\psi_\pi$. By slightly altering the proof of Theorem \ref{thm:rs-pe-error} and leveraging the upper bound $\Vert \psi_{\pi} \Vert^2_{2} \leq \mu^2r\frac{m+n}{mn}$ (see Appendix \ref{subsec:useful-lemmas} for a proof), we can also derive minimax guarantees with a remainder term that scales with $1/\varepsilon$ instead of $1/\varepsilon^{4/3}$. We refer to Theorem \ref{thm:rs-pe-error-minimax} for the precise statement. 

Asymptotically, when $\varepsilon$ tends to 0, these guarantees exhibit a better dependence in $\mu$, $\kappa$, $r$, $\Vert M \Vert_{\max}$ and $\delta$ (and even in $(m,n)$ for non-homogeneous matrices $M$) compared to those of \dsmpe: for $\varepsilon$ small enough\footnote{Refer to \eqref{cond-on-epsilon} in Appendix \ref{subsec:proof-rs-pe-error} for a precise condition.}, the sample complexity guarantees of \rspe\ scale as $\frac{\sigma^2\mu^2r(m+n)}{\omega_{\min}mn \varepsilon^2} \log(1/{\delta})$. When, in addition, $M$ and $\omega$ are homogeneous, they scale as $\frac{\sigma^2(m+n)}{\varepsilon^2} \log(1/{\delta})$. 

For large $\varepsilon$, the sample complexity guarantees of \dsmpe\ presented in Theorem \ref{thm:dsm-pe-error} may be better than those of \rspe\ (we believe however that the analysis of \rspe\ can be improved for such $\varepsilon$). Yet, as our experiments (see Appendix \ref{subsec:pe-experiments}) suggest, \rspe\ outperforms \dsmpe\ numerically even in the large $\varepsilon$ regime.

Observe further that in any case, the first term in the sample complexity guarantees of \rspe\ scales with $\sigma^2$. In contrast, the guarantees of \dsmpe\ scale with $L^2=\left(\Vert M \Vert_{\max} \vee \sigma \right)^2$ (see Theorem \ref{thm:dsm-pe-error}). The remainder term, however, still depends on $L$ through the factor $K$.


We finally note that without accounting for the low-rank structure, the minimax sample complexity of any PE algorithm would scale as ${mn\over \varepsilon^2}\log(1/\delta)$ \cite{yin2020asymptotically}. Exploiting the low-rank structure allows us to replace the factor $mn$ by $(m+n)$.

\subsection{Sample Complexity Lower Bounds}\label{subsec:lowerbounds-pe}

We first derive an instance-dependent lower bound on the sample complexity. Consider the following factorization of $M$: $M=PQ^\top$ with $P \in \RR^{m\times r}$ and $Q \in \RR^{n\times r}$ (this factorization is not unique). Let us define for every  $i \in [m]$, $$\Lambda^{i}_Q=\sum_{j=1}^{n}\omega_{i,j}Q_jQ_j^\top \quad 
\text{ and } \quad Q^{i}_{\pi}=\sum_{j=1}^{n}\omega^{\pi}_{i,j}Q_j.$$
Similarly, for every $j \in [n]$, let 
$$\Lambda^{j}_P=\sum_{i=1}^{m}\omega_{i,j}P_iP_i^\top \quad  \text{ and }  \quad P^{j}_{\pi}=\sum_{i=1}^{m}\omega^{\pi}_{i,j}P_i.$$ 
In the above definitions, $Q_j$ (resp. $P_i$) denotes the $j$-th row vector of $Q$ (resp. $i$-th row vector of $P$). 

\begin{theorem}\label{thm:instance-lower-bound-pe}
      Let $\varepsilon >0, \delta \in (0,1/2)$ and assume that $\xi_t \sim \mathcal{N}\left(0,\sigma^2\right)$. The sample complexity of any $(\epsilon,\delta)$-PAC estimator of $v^{\pi}$ must satisfy 
      $$
      T \geq \frac{\sigma^2L_{M,\pi} }{2\varepsilon^2}\operatorname{kl}\left(\delta,1-\delta\right)
      $$ 
      for $L_{M,\pi}:=\max\big(\sum_i \Vert Q^{i}_{\pi}\Vert ^2_{(\Lambda^{i}_Q)^{-1}},\sum_j \Vert P^{j}_{\pi}\Vert ^2_{\left(\Lambda^{j}_P\right)^{-1}}\big)$. Furthermore, $L_{M,\pi}$ is independent of the choice of rank factorization $M=PQ^\top$. 
\end{theorem}


From Theorem \ref{thm:instance-lower-bound-pe}, we can deduce a minimax lower bound. For some constants $c=(c_1,c_2)$ and $c'$, we introduce the set of matrices ${\cal M}(c,c')=\{M\in \mathbb{R}^{m\times n}: \mathrm{rank}(M)=r\le c', c_1m \le n\le c_2m\}$. 

\begin{proposition}\label{corr:worst-case-lower-bound-pe}
 Let $\varepsilon > 0, \delta \in (0,1/2)$ and assume that $\xi_t \sim \mathcal{N}\left(0,\sigma^2\right)$. There exists a target policy $\pi$ and a homogeneous matrix $M\in {\cal M}(c,c')$ such that the sample complexity of any $\left(\varepsilon,\delta\right)-$PAC estimator of $v^{\pi}$ must satisfy  $$
      T \gtrsim \frac{\sigma^2\left(m+n\right)}{\omega_{\max}mn\varepsilon^2}\operatorname{kl}\left(\delta,1-\delta\right).
      $$  
\end{proposition}
Note that when $\omega$ is homogeneous, the sample complexity lower bound reduces to $T \gtrsim \frac{\sigma^2\left(m+n\right)}{\varepsilon^2}  \operatorname{kl}\left(\delta,1-\delta\right)$.
Since $\operatorname{kl}\left(\delta,1-\delta\right) \geq \log\left(1/\left(2.4\delta\right)\right)$  \cite{kaufmann2016complexity}, this lower bound nearly matches the sample complexity guarantees of \rspe\ when $M$ and $\omega$ are homogeneous. If in addition, $\sigma \gtrsim \Vert M \Vert_{\max},$ it also matches the sample complexity guarantees of \dsmpe.


\section{Best Policy Identification}\label{sec:bpi}

In this section, following the same approach as that used to address the PE task, we devise two BPI algorithms with nearly minmax optimal sample complexity.
In what follows, we assume that the context distribution $\rho$ is homogeneous. We explain how this condition can be dropped at the reasonable cost of additional logarithmic factors in our bounds in Appendix \ref{subsec:general-context}.

As for the PE learning task, our first algorithm is directly based on the estimate $\widehat{M}$, and the second follows our two-phase approach. These algorithms, referred to as \dsmbpi\ and \rsbpi, output the policy $\hat{\pi}_{\dsmbpi}$ and $\hat{\pi}_{\rsbpi}$ respectively, defined by, for all $i \in [m]$, 
\begin{align*}
    \hat{\pi}_{\dsmbpi}(i)&:= \argmax_{1\leq j \leq n} \widehat{M}_{i,j}, \\
    \hat{\pi}_{\rsbpi}(i) &:=\argmax_{1\leq j \leq n} \phi_{i,j}^T\hat{\theta},
\end{align*}
where $\widehat{M}$, $\phi$ and $\hat{\theta}$ are defined in \eqref{eq:SVD_hat_M}, \eqref{eq:phi} and \eqref{eq:rspe-lse}, respectively. The pseudo-code of \rsbpi\ is presented in Algorithm \ref{algo:RSBPI}.

\begin{algorithm}[t]
\caption{\textsc{\textbf{R}ecover \textbf{S}ubspace for \textbf{B}est \textbf{P}olicy \textbf{I}dentification} $(\rsbpi)$ }\label{algo:RSBPI}
\begin{algorithmic}
\STATE \textbf{Input}: Budgets of rounds $T$ and $T_1$, context distribution $\rho$, regularization parameter $\tau$ 
\STATE \emph{\color{purple}\underline{Phase 1:} Subspace recovery}
\STATE Collect $T$ samples according to the uniform policy
\STATE Use the first $T_1$ samples to construct 
$\widehat{U},\widehat{V}$ as in \eqref{eq:SVD_hat_M}
\STATE \emph{\color{purple}\underline{Phase 2:} Solving a misspecified linear bandit}
    \STATE Construct $\lbrace \phi_{i,j}: (i,j) \in [m]\times[n] \rbrace$ as in \eqref{eq:phi}
    \STATE Use the remaining $T-T_1$ samples to construct the least square estimator $\hat{\theta}$ as in  \eqref{eq:rspe-lse}
\STATE Set $\hat{\pi}_{\rsbpi}\left(i\right) = \displaystyle \arg \max_{1 \leq j \leq n} \phi_{i,j}^\top \hat{\theta}$ for all $i \in \left[m\right]$
\STATE \textbf{Output:} $\hat{\pi}_{\rsbpi}$ 
\end{algorithmic}
\end{algorithm}

\begin{theorem}\label{thm:bpi-bound}
Let $\varepsilon \in \left(0,\Vert M \Vert_{\max}\right)$, $\delta \in \left(0,1\right)$ and assume that $\rho$ is homogeneous.\\
(i) $v^\star-v^{\hat{\pi}_{\dsmbpi}} \leq \varepsilon$ with probability larger than $1-\delta$ as soon as
\begin{align*}
T \gtrsim \frac{L^2\mu^{6} \kappa^4 r^3  (m+n)^3}{(m\wedge n)^2 \varepsilon^2}\log^3\!\left(  \frac{(m+n)T}{\delta}\right)\!.
\end{align*} 
    
(ii) Choosing $\tau \le r(m\wedge n)^{-1} (\rho_{\min}/\rho_{\max}) \log(16d/\delta)$ and $T_1 = \lfloor T/2 \rfloor$, $v^\star-v^{\hat{\pi}_{\rsbpi}} \leq \varepsilon$ with probability larger than $1-\delta$ as soon as 
\begin{align*}
 T  \gtrsim \frac{\sigma^2\mu^2r\!\left(m+n\right)\! }{\varepsilon^2}\log\!\left(\!\frac{m+n}{\delta}\!\right) \!+\!  \frac{K}{\varepsilon}\log^{3}\!\left( \!\frac{(m+n)T}{\delta}\!\right)\!, 
\end{align*}
where $K$ depends polynomially on the model parameters \footnote{Refer to \eqref{eq:def_of_K0_2} in Appendix \ref{app:bpi} for a complete expression.} $\Vert M \Vert_{\max},\sigma, \mu, \kappa, r, m, n$.  
\end{theorem} 

Note that under our homogeneity assumption on $\rho$, the sample complexity guarantees of \rsbpi\  scale as $\frac{\sigma^2\mu^2r\left(m+n\right)}{\varepsilon^2}\log\left(1/{\delta}\right)$ when $\varepsilon$ is sufficiently small. In this case, it exhibits a better dependence on the model parameters than the sample complexity guarantees of \dsmbpi\ when $M$ is not homogeneous. For large $\varepsilon$, the guarantees of \dsmbpi\ can be better than those of \rsbpi. Nonetheless, our experiments (see Appendix \ref{subsec:bpi-experiment}) suggest that \rsbpi\ still outperforms \dsmbpi\ in the large $\varepsilon$ regime.

A known minimax sample complexity lower bound for BPI in contextual low-rank bandits scales as ${r(m+n)\over \varepsilon^2}\log(1/\delta)$ (this is essentially proved by \cite{lee2023context} for context-lumpable bandits, a subclass of contextual low-rank bandits). Hence, in view of Theorem \ref{thm:bpi-bound}, both \dsmbpi\ and \rsbpi\ are nearly minimax optimal when the reward matrix and context distribution are homogeneous; the sample complexity of \dsmbpi\ has an additional $\log^2(1/\delta)$ term, and \rsbpi\ is minimax optimal when $\varepsilon$ is sufficiently small. 

Finally, it is worth noting that for the two proposed algorithms, data is simply gathered using the uniform sampling policy (there is no need to be adaptive to achieve a minimax optimal sample complexity when $M$ and $\rho$ are homogeneous).  


\section{Regret Minimization}\label{sec:regret}







In this section, we introduce \rsslin, our algorithm designed for regret minimization in contextual low-rank bandits. Presented as Algorithm \ref{algo:ESRED},  \rsslin\ runs in two phases. In the first phase, we collect $T_1$ samples by simply selecting arms uniformly at random. Based on the corresponding reward observations, we estimate the singular subspaces of $M$ using the procedure outlined in Subsection \ref{subsec:subest}. The problem is then reformulated as a misspecified contextual linear bandit, as described in Section \ref{subsec:reductions}.  

In the second phase, to solve the misspecified linear bandit for the remaining $T-T_1$ rounds, \rsslin\ uses a variant of \textsc{SupLinUCB}  \cite{takemura2021parameter}, outlined in Algorithm \ref{algo:LowSupLinUCB_reg} in Appendix \ref{subsec:misspecified_linear_regret}. In each round $t$, the set of feature vectors ${\cal X}_t=\{ \phi_{i_t,j}, j\in [n]\}$ is constructed using the estimated singular subspaces as defined in \eqref{eq:phi}, with features of dimension $d=r(m+n)-r^2$. 

The choice of $T_1$ will be specified later. Let $T_2 = T-T_1$. We define the threshold $\beta(\delta)$ used in \textsc{SupLinUCB} as: 
\begin{align}
\label{alg1:beta}
    \beta(\delta) =   \sigma ( 1 + \sqrt{2\log \left( \frac{T_2 mn}{\delta} \ceil{\frac{1}{2}\log\left(\frac{T_2}{d}\right) } \right) } ).
\end{align}



\begin{theorem}
\label{thm:main_ESRED_context}
Suppose\footnote{Refer to \eqref{eq:T1_def_regret_app} in Appendix \ref{app:regret_min} for an explicit expression.} $T_1=\widetilde{\Theta}(f(T,M,\rho))$, where 
\begin{align*}
    f(T,M,\rho) := L \mu^2 \kappa^2 r^{5/4}  \frac{(m+n)^{3/4}(mn)^{1/4} }{\sqrt{m\rho_{\min}}(m\wedge n)^{1/2}} \sqrt{T }. 
\end{align*}
Then, the regret of $\pi=\rsslin$ satisfies:\\  $R^\pi(T) = \widetilde{O} (f(T,M,\rho))$, for all $T \ge 1$.
\end{theorem}
Here, $\widetilde{\Theta}(\cdot), \widetilde{O}(\cdot)$ may hide logarithmic factors in $T$, $m$, $n$ and $L/\sigma$. Theorem \ref{thm:main_ESRED_context} is proved in Appendix \ref{app:regret_min}. When the reward matrix $M$ and context distribution $\rho$ are homogeneous, the regret upper bound  of $\rsslin$ simplifies to:
$$
R^\pi(T) = \widetilde{O}\left(Lr^{5/4} n^{3/4}\sqrt{T} \right).
$$

Observe that if we were to use an algorithm meant for contextual bandits with no structure, then the best regret guarantee we can obtain is $\widetilde{O}(\sqrt{mnT})$ \cite{lattimore2020bandit}. In contrast, our regret upper bound scales as $\widetilde{O}(n^{3/4}\sqrt{T})$, and clearly $n^{3/4}\sqrt{T} \ll n\sqrt{T} = \Theta(\sqrt{mnT})$ for all $T \ge 1$ whenever $m = \Theta(n)$. In other words, our algorithm truly takes advantage of the low-rank structure. We finally remark that \citet{lee2023context} also provides non-trivial regret bounds for general low-rank reward matrices, as discussed in the introduction, but only when $T = \widetilde{O}((m+n)^{2(1+(3r)^{-1})})$.

\begin{algorithm}[t]
   \caption{\textsc{\textbf{R}ecover \textbf{S}ubspace for \textbf{R}egret \textbf{MIN}imization (\rsslin)}}
   \label{algo:ESRED}
\begin{algorithmic}
   \STATE {\bfseries Input:} Budget of rounds $T$ and $T_1$, 
 context distribution $\rho$ 
   \STATE \emph{\color{purple}\underline{Phase 1:} Subspace recovery}
   \STATE Collect $T_1$ samples according to the uniform policy
    \STATE Construct $\widehat{U},\widehat{V}$ as described in \eqref{eq:SVD_hat_M}
    \STATE \emph{\color{purple}\underline{Phase 2:} Solving a misspecified linear bandit}
    \STATE Construct $\lbrace \phi_{i,j}: (i,j) \in [m]\times[n] \rbrace$ as in \eqref{eq:phi}
    \STATE Set $T_2\leftarrow T-T_1$, $\delta\leftarrow 1/T$, $\beta(\delta)$ as in \eqref{alg1:beta}, 
    \STATE \qquad $\Lambda \leftarrow \sigma^2/(\|M\|_{\max}^{2}mn) I_d$
    \STATE Run \textsc{SupLinUCB} for the remaining rounds with inputs $T_2$, $\beta(\delta)$ and $\Lambda$.
\end{algorithmic}
\end{algorithm}


We also explain in Appendix \ref{subsec:main_reduction_almost_lowd} why reducing the problem to almost-low-dimensional linear bandits, a framework employed for regret minimization in previous works such as \cite{jun2019bilinear,lu2021low,kang2022efficient}, gives suboptimal regret guarantees compared to our reduction to misspecified linear bandits. Additionally, we present an algorithm based on the reduction to almost-low-dimensional linear bandits that leverages our entrywise guarantees, with regret $R^\pi(T) = \widetilde{O}((m+n)\sqrt{T})$ in the homogeneous case. Our results complement observations made in \cite{kang2024efficient}, where reducing to misspecified linear bandits yields $\Theta(r^{1/4}(m+n)^{1/4})$ higher regret than reducing to almost-low-dimensional linear bandits.

As mentioned in Section \ref{subsec:main_guarantees_two_to_inf}, using spectral methods introduces in the subspace recovery error bound a dependence on $L = \Vert M\Vert_{\max} \vee \sigma$ instead of $\sigma$ (see also \cite{keshavan2009}). This may affect the regret bounds one may obtain in settings where $\sigma \ll \Vert M\Vert_{\max}$. Indeed, should this improvement be possible, one may expect a regret upper bound of order  $\widetilde{O}(\sigma (n+m)^{3/4} \sqrt{T} + \Vert M \Vert_{\max} (n+m) )$ when $M$ and $\omega$ are homogeneous (see Appendix \ref{app:debiased_regret} for more details) which is better than ours for $\sigma \ll \Vert M \Vert_{\max}$. 
We believe that the dependence on $\sigma$ can be improved using either the method proposed in Appendix \ref{subsec:max-norm-improvements} or methods based on nuclear norm regularization \cite{chen2020noisy} as discussed in Appendix \ref{app:debiased_regret}.

\begin{proof}[Proof sketch of Theorem \ref{thm:main_ESRED_context}]
For the sake of simplicity, we assume that $\mu,\kappa=\Theta(1)$ and $m = \Theta(n)$ in the proof sketch. The regret accumulated during the exploration phase (Phase 1) is upper bounded by $ 2T_1 \Vert M\Vert_{\max}$. To upper bound the regret of the second phase, we observe that we face a misspecified linear bandit with misspecification $\epsilon_{\max} = \widetilde O \left( (L^2 r)/(T_1 \sigma_r \omega_{\min} ) \right)$ for $T_1$ sufficiently large according to Corollary \ref{corr:misspecification}, and we apply the regret bound for misspecified linear bandits from Theorem 1 in \cite{takemura2021parameter} (see Theorem \ref{thm:mainthm_LowSupLinUCB_misspec}). Putting all of this together, we get:
\begin{align*}
    R^\pi(T) = \widetilde{O}\left( T_1 \Vert M \Vert_{\max} + \frac{L^2}{T_1} \frac{r^{3/2}\sqrt{n}T }{\sigma_r \omega_{\min}} +  L\sqrt{rnT}  \right).
\end{align*}
Choosing $T_1$ that minimizes this expression while using that $\Vert M\Vert_{\max}/\sigma_r \lesssim r/\sqrt{mn}$ (see Lemma \ref{lem:spikiness}) yields the claimed regret.

\end{proof}



\section{Conclusion}




We devised new algorithms achieving state-of-the-art guarantees in several learning tasks for contextual low-rank bandits. A key observation behind our results is that subspace recovery guarantees in the two-to-infinity norm can be leveraged to perform a succinct reduction from a low-rank bandit to a \emph{misspecified} linear bandit problem. This motivates a two-stage approach as an algorithm design principle, and constitutes a core message of our work. 

For PE and BPI, we followed such a two-stage approach and proposed $\rspe$ and $\rsbpi$ whose minimax sample complexity typically scales as $\widetilde{O}(r(m+n)/\varepsilon^2)$. We also considered a single-stage approach to propose $\dsmpe$ and $\dsmbpi$ with a similar minimax sample complexity. However, experimentally, $\rspe$ and $\rsbpi$ exhibit far superior performance in comparison with $\dsmpe$ and $\dsmbpi$. Shedding some light on this discrepancy is of future interest.  

For regret minimization, the benefits of the two-stage approach are even more pronounced. Indeed, the proposed \rsslin~ achieves a non-trivial regret upper bound typically scaling as $\widetilde{O}(r^{5/4}(m+n)^{3/4}\sqrt{T})$. It is unclear if achieving  $\widetilde{O}(\sqrt{\mathrm{poly}(r)(m+n)T})$ is possible for general low-rank reward matrices, but we believe this to be an exciting research question.





\newpage
\section*{Acknowledgements}
This research was supported by the Wallenberg AI, Autonomous Systems and Software Program
(WASP) funded by the Knut and Alice Wallenberg Foundation, the Swedish Research Council (VR), and Digital Futures.

\section*{Impact Statement} 
This paper presents work whose goal is to advance the field of Machine Learning. There are many potential societal consequences of our work, none which we feel must be specifically highlighted here.

\bibliography{references, bandit}
\bibliographystyle{icml2024}

\newpage
\appendix
\onecolumn

\section{Additional Related Work}\label{app:rel}

In this appendix, we provide additional related work. Most specifically, we survey some of the recent algorithms for \emph{misspecified} contextual linear bandits.  Next, we also discuss bilinear bandits from which the idea of reducing a low-rank bandit problem to an \emph{almost low-dimensional linear bandit} originated. 

\paragraph{Misspecified contextual linear bandits.} 
Misspecified contextual linear bandits have recently received a lot of attention \citep{gopalan2016low, ghosh2017misspecified,lattimore2020learning,foster2020adapting, zanette2020learning, takemura2021parameter}. The best achievable minimax regret bounds for this setting are  $\widetilde{O}(d\sqrt{T} + \epsilon  \sqrt{d} T )$ where $d$ is the ambient dimension, $\epsilon$ is the misspecification, and $T$ is the time horizon. This guarantee is achievable even if $\epsilon$ is not known as proven by \cite{foster2020adapting}. Interestingly, \cite{lattimore2020learning} show that one cannot hope to improve the dependence on $d$ in the term $\epsilon \sqrt{d}$ due to misspecification for most interesting regimes. However, the first term, that is of order $d\sqrt{T}$, is due to the assumption that the set of arms per context is continuous or infinite. In fact, \cite{lattimore2020learning} also shows that a minimax regret of order $\widetilde{O}(\sqrt{d\log(K)T} + \epsilon \sqrt{d}T)$ is achievable when the misspecified linear bandit has a finite number of arms $K$. Furthermore, in the case of misspecified contextual linear bandits with a finite number of arms per context, say $K$, \cite{takemura2021parameter} show that \textsc{SupLinUCB}, initially proposed by \cite{chu2011contextual},  achieves a regret guarantee of order $\widetilde{O}(\sqrt{d\log\phantom{)\!\!\!}^2(K)T} + \epsilon \sqrt{d \log(K)} T)$ without knowledge of $\epsilon$. They further propose a variant of \textsc{SupLinUCB} that has an improved regret upper bound of order $\widetilde{O}(\sqrt{d\log(K)T} + \epsilon \sqrt{d}T)$.

\paragraph{Bilinear bandits with low-rank structure.} In this setting, originally introduced by \cite{jun2019bilinear}, in each round, the learner selects a pair of features $(x, y)$ in $\cX \times \cY \subseteq \RR^{m} \times \RR^{n}$, and then observes a reward with expected value $x^\top M y$, where $M$ is assumed to be of rank $r \ll m + n$. This setting is closely related to ours, because if one assumes that $\cX$ and $\cY$ correspond to the canonical basis in $\RR^{m}$ and $\RR^{n}$, respectively, then one recovers our reward model. To solve this problem, \cite{jun2019bilinear} proposed \textsc{ESTR}, a two-stage algorithm that first estimates the singular subspaces of $M$ using spectral methods, and then solves an almost low-dimensional linear bandit problem, using an adaptation of \textsc{OFUL} \citep{abbasi2011improved} referred to as \textsc{LowOFUL}.  The key insight is that subspace recovery enables a reformulation of bilinear bandit problems as low-dimensional linear bandits. This lead to the first regret guarantees of order $\widetilde{O}((m+n)^{3/2}\sqrt{T})$. Since then, there has been many subsequent work with new algorithms, improvements, and various generalizations \citep{jang2021improved, lu2021low, kang2022efficient, cai2023doubly}. Notably, \cite{kang2022efficient} proposes an algorithm, \textsc{G-ESTT}, to solve generalized bilinear bandit problems. Their algorithm also achieves a regret guarantee of order $\widetilde{O}((m+n)^{3/2}\sqrt{T})$, and relies on the same two-stage algorithm design idea as \cite{jun2019bilinear}, but uses instead a novel subspace estimator that is based on Stein's method. A major limitation of existing state-of-the-art algorithms for bilinear bandits is that when $\vert \cX \times \cY \vert = \widetilde{O}(mn)$, as is the case in our setting, the obtained regret guarantees are no better than the ones achievable by minimax optimal algorithms for unstructured bandits. Indeed, an algorithm such as \textsc{UCB} that does not use the low-rank structure achieves a regret rate of order $\widetilde{O}(\sqrt{mnT})$  \citep{lattimore2020bandit}).

It is worth noting that in a concurrent work, \cite{jang2024efficient} derive a regret bound that scales with $\tilde{O}((m+n)^{1/4}\sqrt{B T})$ where $B$ is a constant that depends on the action set geometry. This constant appears in their bounds because they use an experimental design approach. It is unclear how this constant would scale with our action set, but for the case of the unit sphere, which would also include our action set, they obtain a regret bound of order $\widetilde{O}( (m+n)^{5/4}\sqrt{T})$.  This still does not attain a sub-linear dependence on $m+n$ as we do in our work.

\newpage

\section{Singular Subspace Recovery}
\label{sec:app_sing_sub_rec}

In this appendix, we give the proofs of our results on the singular subspace recovery in the two-to-infinity norm. In Appendix \ref{app:subspace-recovery-1}, we provide the proof of Theorem \ref{thm:recovery-two-to-infinity-norm} which in fact follows from Theorem \ref{thm:recovery-two-to-infinity-norm-2}. The latter theorem is also a guarantee on the two-to-infinity norm but where the estimation error is evaluated in an alternative way. The proof of Theorem \ref{thm:recovery-two-to-infinity-norm-2} is provided in Appendix \ref{app:subspace-recovery-2} and features two key ideas, namely the leave-one-out analysis \citep{abbe2020entrywise, chen2021spectral} and the Poisson approximation argument \citep{stojanovic2024spectral}.


\subsection{ Proof of Theorem \ref{thm:recovery-two-to-infinity-norm} -- Recovery in $\Vert \cdot \Vert_{2 \to \infty}$} \label{app:subspace-recovery-1}

We now present the two key results from which Theorem \ref{thm:recovery-two-to-infinity-norm} follows. We defer the proofs of these two claims to Appendices \ref{app:subspace-recovery-2} and \ref{subsec:app_proof_lemma_recovery-two-to-infinity-norm}, respectively.

\begin{theorem}\label{thm:recovery-two-to-infinity-norm-2}
    Let $\delta \in (0, 1)$. Then, the following event 
    \begin{align*}
    \max\left( \Vert U - \widehat{U} (\widehat{U}^\top U)\Vert_{2 \to \infty}, \Vert V - \widehat{V} (\widehat{V}^\top V)\Vert_{2 \to \infty} \right)\le C \frac{L }{\sigma_r(M)} \sqrt{\frac{\mu^2\kappa^2 r (m+n) }{T\omega_{\min}(m\wedge n)}\log^3\left(\frac{(m+n)T}{\delta}\right)},\
    \end{align*}
    holds with probability at least $1- \delta$ as long as 
    \begin{align*}
    \quad T \geq c \frac{L^2 (m+n)}{\sigma_r^2(M) \omega_{\min}} \log^3\left( \frac{(m+n)T}{\delta}\right)
    \end{align*}  
    for some universal constants $C,c>0$. 
\end{theorem}


\begin{lemma}\label{lem:recovery-two-to-infinity-norm}
    Under the event $\Vert \widetilde{M} - M \Vert_\op \le \sigma_r(M)/2$,  we have  
    \begin{align*}
        d_{2\to\infty}(U, \widehat{U}) = \Vert   UU^\top - \widehat{U}\widehat{U}^\top \Vert_{2 \to \infty} & \le 2 \Vert U - \widehat{U} \widehat{U}^\top U \Vert_{2 \to \infty}  + \frac{4 \Vert U \Vert_{2 \to \infty} \Vert \widetilde{M} - M \Vert_\op}{\sigma_r(M)},\\ 
        d_{2\to\infty}(V, \widehat{V}) = \Vert   VV^\top - \widehat{V}\widehat{V}^\top \Vert_{2 \to \infty} & \le 2 \Vert V - \widehat{V} \widehat{V}^\top V \Vert_{2 \to \infty}  + \frac{4 \Vert V \Vert_{2 \to \infty} \Vert \widetilde{M} - M \Vert_\op}{\sigma_r(M)}.
    \end{align*}
\end{lemma}

Now the proof of Theorem \ref{thm:recovery-two-to-infinity-norm} follows straightforwardly. First, we note that under the event that $\Vert \widetilde{M} - M\Vert_\op \le \sigma_r(M)/2$, by Lemma \ref{lem:recovery-two-to-infinity-norm} we only need to upper bound the terms $\Vert U - \widehat{U} \widehat{U}^\top U \Vert_{2 \to \infty},\Vert V - \widehat{V} \widehat{V}^\top V \Vert_{2 \to \infty},\Vert \widetilde{M} - M \Vert_\op$. 
Now, observe that we may use Theorem  \ref{thm:recovery-two-to-infinity-norm-2} to upper bound $\Vert U - \widehat{U} \widehat{U}^\top U \Vert_{2 \to \infty}$ and $\Vert   V - \widehat{V}\widehat{V}^\top V \Vert_{2 \to \infty}$ with high probability, and use Proposition \ref{prop:reward-concentration-1} to upper bound the term $\Vert \widetilde{M} - M \Vert_\op$ with high probability. The obtained upper bounds combined with the fact that 
$\Vert U\Vert_{2 \to \infty} \vee \Vert V \Vert_{2\to \infty}  \le \mu \sqrt{r/(m\wedge n)}$ yield the final result.

\subsection{Proof of Theorem \ref{thm:recovery-two-to-infinity-norm-2}} \label{app:subspace-recovery-2}

The proof of Theorem \ref{thm:recovery-two-to-infinity-norm-2} is rather involved but follows the same steps as the proof of Lemma 30 in \cite{stojanovic2024spectral}, with the slight difference that in \cite{stojanovic2024spectral}, it is assumed that $\omega_{i,j}=1/(mn)$ for all $i,j$. For the sake of completeness, we highlight the key steps of the proof next.

\paragraph{Step 1: Poisson approximation.} For the sake of the analysis, it is rather convenient to describe the estimate $\widetilde{M}$ given in \eqref{eq:def_Mtilde_uniform} with an alternative random matrix which has the same distribution. For all $(i,j) \in [m]\times [n]$, define $Z_{i,j} = \sum_{t=1}^T \indicator_{\lbrace (i_t, j_t) = (i , j)\rbrace}$ and let $(\xi_{i,j,t}')$ be a sequence of i.i.d. random variables that have the same distribution as $\xi_1$. Now, observe that our matrix $\widetilde{M}$ has the same distribution as the random matrix $\widetilde{M}'$ defined as: 
\begin{align}\label{eq:def_Mtilde_alt}
     \forall (i,j) \in [m]\times [n], \qquad   \widetilde{M}'_{i,j} = \frac{1}{T \omega_{i,j}} \sum_{t=1}^{Z_{i,j}} (M_{i,j} + \xi_{i,j,t}' ).
\end{align}
Let $\PP$ denote the joint probability distribution of $(Z_{i,j})_{(i,j) \in [m]\times[n]}$ and $(\xi'_{i,j,t})_{(i,j)\in [m]\times [n], t \ge 1}$.

Next, we describe a compound Poisson random matrix model that will serve as an approximation of the random model \eqref{eq:def_Mtilde_alt}. Let $Y \in \RR^{m\times n}$ be a random matrix with independent entries, such that $Y_{i,j} \sim \textrm{Poisson}(T\omega_{i,j})$, for all $(i,j) \in [m]\times [n]$. Let us further define the random matrix $X$ taking values $\RR^{m \times n}$ as follows:
\begin{align}\label{eq:def_M_poisson}
        \forall (i,j) \in [m] \times [n], \qquad X_{i,j} = \frac{1}{T\omega_{i,j}}\sum_{t=1}^{Y_{i,j}} (M_{i,j} + \xi'_{i,j,t}). 
\end{align}
We observe that the entries of the matrix $X$ are independent and distributed according to compound Poisson distributions. 

Let $\PP'$ denote the joint probability distribution of $(Y_{i,j})_{(i,j) \in [m] \times [n]}$ and $(\xi_{i,j,t}')_{(i,j) \in [m]\times [n], t\ge 1}$. The Poisson approximation argument is formalized in the next lemma.

\begin{lemma}[Poisson Approximation]\label{lem:poisson-approx-rewards}
    Let $(\Omega, \mathcal{F} , \PP)$ (resp.  $(\Omega, \cF , \PP')$) be the probability space under the random matrix model \eqref{eq:def_Mtilde_alt}  (resp. \eqref{eq:def_M_poisson}). Then for any event $\mathcal{E} \in \mathcal{F}$, we have 
    \begin{align*}
        \PP \left( \mathcal{E}  \right)  \le e \sqrt{T} \; \PP'\left(\mathcal{E} \right).
    \end{align*}
\end{lemma}

The proof of Lemma \ref{lem:poisson-approx-rewards} follows from Lemma \ref{lem:poisson-approx} and is similar to the proof of Lemma 20 in \cite{stojanovic2024spectral}.

\begin{proof}[Proof]
Let $\mathcal{E} \in \mathcal{F}$ be an event and for convenience, introduce the notation $X = (\xi_{i,j,t})_{(i,j)\in [m]\times [n],t \ge 1}$. First, we note that $(Y_{i,j})_{(i,j) \in [m]\times [n]} \sim \textrm{Multinomial}(T, (\omega_{i,j})_{(i,j)\in m \times n})$. Let $f(Y, X) = \indicator_{\lbrace \mathcal{E}\rbrace}$. Then, by Lemma \ref{lem:poisson-approx}, we have 
\begin{align*}
    \EE[f(X, Y) \vert X] \le e \sqrt{T} \EE'[f(X, Z) \vert X].  
\end{align*}
By taking the expectation with respect to $X$, we get $\EE[f(X, Y)] \le \EE[f(X, Z)]$, which means that $\PP(\mathcal{E}) \le e\sqrt{T}\PP'(\mathcal{E})$.
\end{proof}

We borrow the following Lemma from \cite{stojanovic2024spectral} (see their Lemma 14) which itself is a mild generalization of Theorem 5.7 in \cite{mitzenmacher2017probability}

\begin{lemma}\label{lem:poisson-approx}
    Let $Y_i^{(t)} \sim \mathrm{Poisson}(t p_i)$, $i=1,\dots,n$, be independent random variables with $\sum_{i=1}^{n} p_i = 1$. Moreover, let $(Z_1^{(t)},Z_2^{(t)},\dots,Z_n^{(t)})\sim\allowbreak \mathrm{Multinomial}(t,(p_1,\dots,p_n))$. Let $f:\mathbb{R}^p\to\mathbb{R}_+$ be any non-negative function. Then
    \begin{align*}
        \EE [f(Z_1^{(t)},\dots,Z_{p}^{(t)})] \leq e\sqrt{t} \EE [f(Y_1^{(t)},\dots,Y_{p}^{(t)})].
    \end{align*}
    \label{lemma:hetero_poisson_multi}
\end{lemma}

\paragraph{Step 2: Key concentration inequalities.}

The following propositions are direct consequences of a truncated matrix Bernstein inequality (see e.g., Proposition A.7 in \cite{hopkins2016fast}). 

\begin{proposition}\label{prop:reward-concentration-1}
    Under the random matrix model \eqref{eq:def_M_poisson} with compound Poisson entries, for all $\delta \in (0,1)$, for all $
        T  \ge   13/(\omega_{\min} (m \wedge n)) \log^3\left( (m+n)/\delta\right)$, we have:
    \begin{align*}
         \PP\left(\Vert \widetilde{M} - M \Vert_\op  \le  36\sqrt{2}L  \sqrt{\frac{1}{ T \omega_{\min}}}\left( \sqrt{ (m+n)  \log\left(\frac{m+n}{\delta}\right)  }  + \log^{3/2}\left(\frac{m+n}{\delta}\right)   \right)   \right) \ge 1 - \delta 
    \end{align*}
    with $L = \Vert M \Vert_{\max} \vee \sigma$.
\end{proposition}

\begin{proposition}\label{prop:reward-concentration-2}
    Let $A$ be a $m \times 2r$ deterministic matrix, and $B$ be a $n \times 2r$ deterministic matrix. Then, under the random matrix model \eqref{eq:def_M_poisson} with compound Poisson entries, and denoting $L = \Vert M\Vert_{\max} \vee \sigma$, we have: 
    \begin{itemize}
        \item [(i)] for all $\ell \in [m]$, for all $\delta\in (0,1)$, for all $T \ge (1/(n\omega_{\min})) \log^3(en/\delta)$, the event   
        \begin{align}\label{prop:eq:statement}
            \Vert (\widetilde{M}_{\ell,:} - M_{\ell, :})  A \Vert_2 \le 73\sqrt{2} L \Vert A \Vert_{2\to \infty} \sqrt{\frac{1}{T \omega_{\min}}}\left(\sqrt{ n \log\left(\frac{en}{\delta}\right)} +  \log^{3/2}\left( \frac{en}{\delta}\right) \right) 
        \end{align}
        holds with probability at least $1-\delta$;
        \item[(ii)]  for all $k \in [n]$, for all $\delta\in (0,1)$, for all $T \ge (1/(m\omega_{\min})) \log^3(em/\delta)$, the event 
        \begin{align}\label{prop:eq:statement-transpose}
            \!\! \Vert (\widetilde{M}_{:,k} - M_{:,k})^\top  B \Vert_2 \le 73\sqrt{2} L \Vert B \Vert_{2\to \infty}\sqrt{\frac{1}{T \omega_{\min}}}\left(\sqrt{ m \log\left(\frac{em}{\delta}\right)} +  \log^{3/2}\left( \frac{em}{\delta}\right) \right) 
        \end{align}
        holds with probability at least $1-\delta$.
    \end{itemize}    
\end{proposition}

The proofs of Proposition \ref{prop:reward-concentration-1} and Proposition  \ref{prop:reward-concentration-2} are essentially the same as those of Proposition 26  and Proposition 27 in \cite{stojanovic2024spectral}, and are omitted. Note that both Propositions in \cite{stojanovic2024spectral} were derived under an i.i.d. sub-gaussian noise assumption.

\paragraph{Step 3: Leave-one-out analysis.} This part of the analysis follows the steps of the proof of Lemma 30 in \cite{stojanovic2024spectral}, but we repeat the main results for the sake of completeness.

(i) First, a dilation trick is needed to reduce the analysis to that of symmetric matrices. We introduce: 
\begin{align*}
    S = \begin{bmatrix}
        0 & M \\
        M^\top & 0
    \end{bmatrix} 
\end{align*}
and recalling the SVD of $M$, $M = U \Sigma V^\top$, we can express the SVD of $S$ as follows:
\begin{align*}
    S = \frac{1}{\sqrt{2}} \begin{bmatrix}
        U & U \\
        V & - V 
    \end{bmatrix} \begin{bmatrix}
        \Sigma & 0 \\
        0 & - \Sigma
    \end{bmatrix} \frac{1}{\sqrt{2}} \begin{bmatrix}
        U & U \\
        V & - V
    \end{bmatrix}^\top : = Q D Q^\top. 
\end{align*}
We can define $\widetilde{S}$ in a similar way, and denote $\widehat{Q} \in \RR^{(m+n)\times 2r}$, the matrix containing the $2r$ eigenvectors of the best $2r$-rank approximation of  $\widetilde{S}$. Here, we observe that:
\begin{align}
    \Vert Q - \widehat{Q} (\widehat{Q}^\top Q)\Vert_{2 \to \infty} = \max \left\lbrace \Vert U - \widehat{U} (\widehat{U}^\top U)\Vert_{2 \to \infty},   \Vert V - \widehat{V} (\widehat{V}^\top V)\Vert_{2 \to \infty} \right\rbrace. 
\end{align}
Additionally, we note that $\sigma_{2r}(S) = \sigma_r(M)$. To simplify the notation, we denote $E= \widetilde{S} - S$ and note that $\Vert E \Vert_\op = \Vert \widetilde{M} - M \Vert_\op$.

(ii) Second, it can be shown, under the event $$\mathcal{E}_1 = \lbrace \Vert E \Vert_\op \le C_1 \sigma_r(M) \rbrace$$ where $C_1$ is a universal and sufficiently small positive constant, that: 
\begin{align} \label{eq:loo-1}
    \| Q - \widehat{Q} (\widehat{Q}^\top Q) \|_{2\to\infty} 
    \leq \frac{1}{\sigma_{2r}(S)} \bigg( \frac{5\| S Q  \|_{2\to\infty} \| E\|_\op}{\sigma_{2r}(S)} + 3\| E Q \|_{2\to\infty} + 2\| E ( Q - \widehat{Q} (\widehat{Q}^\top Q)  \|_{2\to\infty} \bigg).
\end{align}
We can easily control  $\Vert  E \Vert_\op $ using Proposition \ref{prop:reward-concentration-1} and $\Vert E Q\Vert_{2\to \infty}$ using Proposition \ref{prop:reward-concentration-2}. On the other hand, the term $\| E ( Q - \widehat{Q} (\widehat{Q}^\top Q ))\|_{2\to\infty}$ is the bottleneck of the analysis because $E$ and $( Q - \widehat{Q} (\widehat{Q}^\top Q))$ are dependent on each other in a non-trivial way. The leave-one-out analysis is used to deal with this term.

(iii) We make the leave-one-out analysis more precise. We introduce for all $\ell \in [m+n]$, the matrix $\widetilde{S}^{(\ell)}$: for all $(i,j) \in [m+n]\times[m+n]$, 
\begin{align*}
    \widetilde{S}^{(\ell)}_{i,j} = \begin{cases}
        \widetilde{S}_{i,j},\quad &\text{if } i\neq \ell \text{ or } j \neq \ell, \\
        S_{i,j}, \quad &\text{otherwise.}
    \end{cases}
\end{align*}
We further define $\widehat{Q}^{(\ell)} \in \RR^{(m+n)\times 2r}$ as the matrix that contains the $2r$ eigenvector of the best $2r$-rank approximation of $\widetilde{S}^{(\ell)}$. It can be shown that, again under $\mathcal{E}_1$ where $C_1$ is sufficiently small, that
 \begin{align}
    \Vert E (Q - \widehat{Q}W_{\widehat{Q}})\Vert_{2 \to \infty} & \le \max_{\ell \in [m+n]} \Vert  E_{\ell,:} ( Q - \widehat{Q}^{(\ell)} W_{\widehat{Q}^{(\ell)}} ) \Vert_2 + \Vert E \Vert_\op \Vert \widehat{Q} W_{\widehat{Q}}  - \widehat{Q}^{(\ell)} W_{\widehat{Q}^{(\ell)}}\Vert_F,  \label{eq:loo-2}  \\
    \Vert \widehat{Q} W_{\widehat{Q}}  - \widehat{Q}^{(\ell)} W_{\widehat{Q}^{(\ell)}}\Vert_F & \le  \frac{16}{\sigma_r(M)}( \Vert E_{\ell, :} ( Q - \widehat{Q}^{(\ell)} W_{\widehat{Q}^{(\ell)}} ) \Vert_2 + \Vert E_{\ell,:} Q \Vert_2 + \Vert E \Vert_\op \Vert Q - \widehat{Q} W_{\widehat{Q}}\Vert_{2 \to \infty}  + \Vert E \Vert_\op \Vert Q \Vert_{2 \to \infty} )  , \label{eq:loo-3}
\end{align}
where we denote $W_{\widehat{Q}} = \widehat{Q}^\top Q$ and $W_{\widehat{Q}^{(\ell)}} = \widehat{Q}^{(\ell)\top} Q$. 

Now, the key observation that underpins the leave-one-out analysis is that $E_{\ell,:}$ and $Q - \widehat{Q}^{(\ell)} W_{\widehat{Q}^{(\ell)}}$ are independent provided that the entries of the error matrix $E$ are independent. This will be ensured thanks to the Poisson approximation (using Lemma \ref{lem:poisson-approx-rewards}) by controlling the quantities of interest under the random matrix model \eqref{eq:def_M_poisson}. 
Define 
\begin{align*}
\mathcal{B}(\delta) = \sqrt{\frac{m+n}{T \omega_{\min}}\log^3\left(\frac{m+n}{\delta}\right)}
\end{align*}
and the events 
\begin{align*}
    \mathcal{E}_2 & = \left\lbrace  \Vert E \Vert_\op \le C_2 L \mathcal{B}(\delta) \right\rbrace\\
    \forall \ell \in [m+n], \qquad \mathcal{E}_3^{(\ell)} & = \left\lbrace   \Vert E_{\ell,:} (Q - \widehat{Q}^{(\ell)} W_{\widehat{Q}^{(\ell)}})\Vert_2  \le C_3 L\Vert Q - \widehat{Q}^{(\ell)} W_{\widehat{Q}^{(\ell)} }\Vert_{2\to \infty} \mathcal{B}(\delta) \right\rbrace
\end{align*}
where $C_2$ and $C_3$ are sufficiently large positive constants. Then, applying Proposition \ref{prop:reward-concentration-1} and Proposition \ref{prop:reward-concentration-2}, we have that $\PP(\mathcal{E}_2) \ge 1-\delta$ and $\PP(\mathcal{E}_3^{(\ell)}) \ge 1-\delta$ for all $\ell \in [m+n]$, provided that 
$$
T \ge \frac{c}{\omega_{\min}(m \wedge n)} \log^3\left(\frac{m+n}{\delta}\right)
$$
for $c>0$ a sufficiently large universal constant. Additionally, using the fact that $L/\sigma_r(M) \ge 1 / \sqrt{mn}$ from Lemma \ref{lem:spikiness}, we can deduce that $\mathcal{E}_1 \supseteq \mathcal{E}_2 $ and $\PP(\mathcal{E}_2) \ge 1-\delta$ as well as $\PP(\mathcal{E}_3^{(\ell)}) \ge 1-\delta$  for all $\ell \in [m+n]$ when 
\begin{align}\label{eq:condition-1}
    T \ge \frac{c' L^2(m+n)}{\sigma_r^2(M)\omega_{\min}} \log^3\left(\frac{m+n}{\delta}\right)
\end{align}
for a positive universal constant $c'$ sufficiently large. At this point, what remains to be done is to deal with the inequalities \eqref{eq:loo-1}, \eqref{eq:loo-2} and \eqref{eq:loo-3} under the intersection of the events $\mathcal{E}_2, \mathcal{E}_3^{(1)}, \dots, \mathcal{E}_3^{(m+n)}$. We spare the readers the details of this tedious task and write directly the resulting bound
\begin{align*}
    \Vert Q - \widehat{Q}(\widehat{Q}^\top Q)\Vert_{2 \to \infty} \le C'\frac{L}{\sigma_r(M)}  \frac{\mu \kappa r^{1/2}\mathcal{B}(\delta) }{\sqrt{m\wedge n}}
\end{align*}
for some universal constant $C' > 0$. Thus, by a union bound, the above bound holds with probability at least $1 - (1+m+n)\delta$, provided that condition \eqref{eq:condition-1} holds, under the random matrix model \eqref{eq:def_M_poisson}. Using Lemma \ref{lem:poisson-approx-rewards}, we conclude that the same event holds with probability $1 - (1+m+n)e\sqrt{T}\delta$, provided that condition \eqref{eq:condition-1} holds, under the random matrix model \eqref{eq:def_Mtilde_alt}. Re-parametrizing by $\delta' = e(m+n+1)\sqrt{T} \delta$ concludes the proof.

\subsection{Proof of Lemma \ref{lem:recovery-two-to-infinity-norm}}
\label{subsec:app_proof_lemma_recovery-two-to-infinity-norm}
We start with the first inequality. Let $W\in\mathbb{R}^{r\times r}$ be an arbitrary orthogonal matrix and note that   
    \begin{align}\label{ineq:lem:1}
        \Vert   UU^\top - \widehat{U}\widehat{U}^\top \Vert_{2 \to \infty} & \le   \Vert   U U^\top - U (\widehat{U} W )^\top  \Vert_{2 \to \infty} +  \Vert U (\widehat{U} W )^\top - \widehat{U} W  (\widehat{U} W)^\top   \Vert_{2 \to \infty}   \nonumber \\
        & \le \Vert U \Vert_{2 \to \infty}  \Vert U - \widehat{U} W\Vert_\op  + \Vert U - \widehat{U} W \Vert_{2 \to \infty} 
    \end{align}
    where the first inequality follows from the triangular inequality and the fact that $WW^\top = I_r$, and the second inequality follows by the inequality $\Vert A B \Vert_{2 \to \infty} \le \Vert A \Vert_{2 \to \infty} \Vert B \Vert_\op $ and the fact that $\Vert (\widehat{U} W)^\top \Vert_\op = 1$. 
    
    Next, we choose $W = \textrm{sgn}(\widehat{U}^\top U)$ where $\text{sgn}(A)$ denotes the sign matrix of $A$ (see \cite{chen2021spectral}). First, according to Lemma 4.15 in \cite{chen2021spectral} and Davis-Kahan inequality, we know that 
    \begin{align*}
        \Vert \widehat{U}^\top U - \textrm{sgn}(\widehat{U}^\top U) \Vert_\op = \Vert \sin(\widehat{U}, U) \Vert_\op^2 \le  \frac{2 \Vert \widetilde{M} - M \Vert^2_\op}{\sigma_r(M)^2} \le \frac{\Vert \widetilde{M} - M \Vert_\op}{\sigma_r(M)} \le \frac{1}{2}
    \end{align*}
    where the last two inequalities hold under the event $\Vert \widetilde{M} - M \Vert_\op \le \sigma_r(M)/2$.
    Thus, under this event,  we have 
    \begin{align*}
        \Vert U  - \widehat{U} W \Vert_{2 \to \infty} & \le \Vert U - \widehat{U} \widehat{U}^\top U \Vert_{2 \to \infty} + \Vert \widehat{U} \Vert_{2 \to \infty} \Vert \widehat{U}^\top U -  \textrm{sgn}(\widehat{U}^\top U)\Vert_{\op }  \\
        & \le \Vert U - \widehat{U} \widehat{U}^\top U \Vert_{2 \to \infty} + (\Vert U - \widehat{U} W \Vert_{2 \to \infty} + \Vert U \Vert_{2 \to \infty} ) \Vert \widehat{U}^\top U -  \textrm{sgn}(\widehat{U}^\top U)\Vert_{\op } \\
        & \le \Vert U - \widehat{U} \widehat{U}^\top U \Vert_{2 \to \infty} + \frac{\Vert U - \widehat{U} W \Vert_{2 \to \infty}}{2} + \frac{\Vert U \Vert_{2 \to \infty} \Vert \widetilde{M} - M \Vert_\op}{\sigma_r(M)}   
    \end{align*}
    which implies
    \begin{align}\label{ineq:lem:2}
        \Vert U  - \widehat{U} W \Vert_{2 \to \infty} \le 2 \Vert U - \widehat{U} \widehat{U}^\top U \Vert_{2 \to \infty}  + \frac{2 \Vert U \Vert_{2 \to \infty} \Vert \widetilde{M} - M \Vert_\op}{\sigma_r(M)}.   
    \end{align}

    We also have by Davis-Kahan inequality and relations between the subspace distances (see Lemma 2.5 and Lemma 2.6 in \cite{chen2021spectral}) 
    \begin{align}\label{ineq:lem:3}
        \Vert U  - \widehat{U} W \Vert_{\op} \le \sqrt{2} \Vert U U^\top - \widehat{U}\widehat{U}^\top \Vert_\op = \sqrt{2} \Vert \sin(\widehat{U},U)\Vert_\op \le\frac{\sqrt{2} \Vert \widetilde{M} - M \Vert_\op }{\sigma_r(M)}. 
    \end{align}

    Finally, combining the inequalities \eqref{ineq:lem:1}, \eqref{ineq:lem:2} and \eqref{ineq:lem:3}, we conclude that under the event $\Vert \widetilde{M} - M \Vert_\op \le \sigma_r(M)/2$, we have 
    \begin{align*}
        \Vert   UU^\top - \widehat{U}\widehat{U}^\top \Vert_{2 \to \infty} & \le 2 \Vert U - \widehat{U} \widehat{U}^\top U \Vert_{2 \to \infty}  + \frac{4 \Vert U \Vert_{2 \to \infty} \Vert \widetilde{M} - M \Vert_\op}{\sigma_r(M)}.  
    \end{align*}
    This concludes the proof of the first inequality. The second one follows similarly.

\subsection{Subspace recovery in Frobenius norm}

Deriving such guarantees is by now well understood under general noise assumptions (see e.g., \cite{keshavan2009, chen2021spectral}). Indeed, the classical proof follows by first deriving upper bounds on $\| \widetilde{M}-M\|_\op$ with high probability, then using Davis-Kahan's or Wedin's theorem to obtain upper bounds on the subspace recovery error in the Frobenius norm. The following lemma provides such upper bounds:

\begin{lemma}
[Subspace recovery in $\Vert \cdot \Vert_\F$] \label{lemma:subspace_Frobenius}
Let $\delta \in (0,1)$ and consider the setting of Theorem \ref{thm:recovery-two-to-infinity-norm}. Then, with probability at least $1-\delta$, it holds that 
\begin{align*}
    \max\left( \Vert UU^\top - \widehat{U}\widehat{U}^\top \Vert_\F, \Vert V V^\top - \widehat{V} \widehat{V}^\top  \Vert_\F \right) 
       \lesssim   \frac{L\sqrt{r }}{\sigma_r(M)} \sqrt{\frac{m+n}{T\omega_{\min}}\log^3\left( \frac{(m+n)T}{\delta}\right)},
\end{align*}
provided that $T \gtrsim \frac{1}{\omega_{\min}(m\wedge n)}\log^3\left(\frac{(m+n)T}{\delta}\right)$.
\end{lemma}

We remark that after $\widetilde{\Omega}(m+n)$ observations, when $m+n = O(m\wedge n)$, we can recover the singular subspaces with an error in $\Vert \cdot \Vert_\F$  of order $\widetilde{O}(\sqrt{(m+n)/T})$. This contrasts with the error obtained from Theorem \ref{thm:recovery-two-to-infinity-norm} that is of order $\widetilde{O}(\sqrt{1/T})$ (in the homogeneous case) with the same amount of observations. This suggests that the subspace recovery error (seen as a matrix) is {\it delocalized}, i.e., spread out along $m+n$ directions. 


\begin{proof}[Proof of Lemma \ref{lemma:subspace_Frobenius}]
Define $\sin\Theta(U,\widehat{U})$ as a diagonal matrix with entries $\{\sin \theta_i\}_{i=1}^{r}$ where $\theta_i = \arccos \sigma_i(\widehat{U}^\top U)$. Applying first Lemma 1 in \cite{cai2018rate}, and then Davis-Kahan's theorem (Corollary 2.8 in \cite{chen2021spectral}), we obtain:
    \begin{align*}
     \Vert UU^\top - \widehat{U}\widehat{U}^\top \Vert_{F} \le \sqrt{2r}\Vert UU^\top - \widehat{U}\widehat{U}^\top \Vert_{\op} \le 2\sqrt{r}\Vert \sin \Theta(U,\widehat{U})\Vert_\op \leq \frac{2\sqrt{2r} \Vert M - \widetilde{M}\Vert_{\op}}{\sigma_r(M)}. 
\end{align*}
An analogous bound holds for $\Vert VV^\top - \widehat{V}\widehat{V}^\top \Vert_{F}$. Lastly, using Proposition \ref{prop:reward-concentration-1} to bound $\Vert M - \widetilde{M}\Vert_\op$ and Lemma \ref{lem:spikiness} to bound $\Vert M\Vert_{\max}/\sigma_r(M)$ concludes the proof.
\end{proof}

\subsection{Additional Lemmas}

The following result is immediate (see Lemma 17 in \cite{stojanovic2024spectral}).
\begin{lemma}\label{lem:spikiness}
    Let $M$ be an $m\times n$ matrix, with rank $r$, incoherence parameter $\mu > 0$, and condition number $\kappa > 0$. Then, 
    \begin{align*}
           \frac{1}{\sqrt{mn}}  \le \frac{\Vert M \Vert_{\max}}{\sigma_1(M)} \le  \frac{\mu^2 r}{\sqrt{mn}},   
    \end{align*}
    Consequently, we also have 
    \begin{align*}
        \frac{1}{\sqrt{mn}}\le \frac{\Vert M \Vert_{\max}}{\sigma_r(M)}\le \frac{\mu^2 \kappa r}{\sqrt{mn}}.
    \end{align*}
\end{lemma}




\newpage
\section{Policy Evaluation for Contextual Linear Bandits}\label{app:linear}

In this appendix, we study the PE task for the simpler contextual linear bandit problem, whose error bound analysis will be useful to handle the PE task for the contextual low-rank bandit problem. We also demonstrate that the linear bandit error bound is instance-optimal by deriving a sample complexity lower bound of independent interest.

We observe $T$ i.i.d. samples of the form $(i_t,j_t,r_t)_{1 \leq t \leq T}$ where the context $i_t \in [m]$ is drawn independently from a known distribution $\rho$ and the arm $j_t \in [n]$ is sampled from $\pi^b(.|i_t)$ for a known behavioral policy $\pi^b$. The reward distribution is assumed to have a linear structure with respect to a known feature map $\psi : [m] \times [n] \to \mathbb{R}^d $ : there exists an unknown parameter $\theta$ in $\mathbb{R}^d$ such that $r_t=\psi_{i_t,j_t}^\top \theta+\xi_t$ where $\left(\xi_t\right)_{t \leq 1}$ is an i.i.d. sequence of $\sigma$-subgaussian noise, independent from all the other random variables. The overall bandit model described above will be denoted as $\mathcal{M}_{\theta}$. For a given target policy $\pi$, we use the previously defined notations $\omega_{i,j}=\rho(x)\pi^b(j|i)$, $\omega^{\pi}_{i,j}=\rho(x)\pi(j|i)$ and we define  $$\Lambda=\displaystyle \sum_{(i,j) \in [m] \times [n]} \omega_{i,j}\psi_{i,j}\psi_{i,j}^\top \quad \text{and} \quad v^{\pi}_{\theta}=\sum_{i,j} \omega^{\pi}_{i,j}\psi_{i,j}^\top \theta.$$ 
For any $\delta \in (0,1)$ and any $\varepsilon >0$,  we say that $\hat{v}^{\pi}$ is $(\varepsilon,\delta)$-PAC\footnote{We will often omit the superscript $\pi$ to ease the notations.}  if $\mathbb{P}_{\theta}\left(|v^{\pi}_{\theta}-\hat{v}^{\pi}| \leq \varepsilon \right) \geq 1-\delta$ for every $\theta \in \mathbb{R}^d$. We will write $v^{\pi}$ for $v^{\pi}_{\theta}$ when there is no ambiguity about the underlying parameter. We also write $N \succeq M$ (resp.  $N \preceq M$) when $N-M$ (resp. $M-N$) is positive semi-definite, and denote the Moore-Penrose pseudoinverse \cite{barata2011moore} of $N$ by $N^{\dagger}$.
\subsection{Instance-dependent sample complexity lower bound} 

In this section only, we assume that the noise follows a centered Gaussian distribution: $\xi_t \sim \mathcal{N}\left(0,\sigma^2\right)$.  
\begin{lemma}\label{lemma:log-likelihood}
    Denote by $q_{\theta}(.|i,j)$ the distribution of the reward when choosing action $j$ under context $i$, and by $N_{i,j}$ the number of times $(i,j)$ has been observed in the sample set. Let $L_T=L_T(i_1,j_1,r_1,\dots,i_T,j_T,r_T)$ be the log-likelihood ratio between two models $\mathcal{M}_{\theta}$ and $\mathcal{M}_{\theta'}$. Then, $$\mathbb{E}_{\theta}[L_{T}]=\sum_{(i,j) \in [m] \times [n]} \mathbb{E}_{\theta}\left[N_{i,j}\right]\operatorname{KL}(q_{\theta}(.|i,j),q_{\theta'}(.|i,j)).$$
\end{lemma}

\begin{proof}[Proof of Lemma \ref{lemma:log-likelihood}]
We can write $$ \begin{aligned}
L_{T}&=\sum_{t=1}^{T}\sum_{i,j} \indicator_{i_t=i,j_t=j}\operatorname{log}\left(\frac{q_{\theta}(r_t|i,j)}{q_{\theta'}(r_t|i,j)}\right) \\
&= \sum_{i,j} \sum_{t=1}^{N_{i,j}} \operatorname{log}\left(\frac{q_{\theta}(r^{i,j}_t|i,j)}{q_{\theta'}(r^{i,j}_t|i,j)}\right) 
\end{aligned}
$$

where $r^{i,j}_t$ is the reward observed after the $t$-th time that arm $j$ has been selected under context $i$. Since the $r^{i,j}_t$ are i.i.d. of distribution $q_{\theta}(|i,j)$ (resp. $q_{\theta'}(|i,j)$) under $\mathcal{M}_{\theta}$ (resp. $\mathcal{M}_{\theta'}$), Wald's Lemma yields
the result.
\end{proof}

\begin{proposition}\label{prop:sampling-lower-bound-1}
Let $\operatorname{Alt}_\varepsilon\left(\theta\right)=\left\{\theta' \in \mathbb{R}^d, |v^{\pi}_{\theta'}-v^{\pi}_{\theta}| \geq 2\varepsilon\right\}$.
If $\hat{v}$ is $(\varepsilon,\delta)$-PAC with $\delta \in (0,1/2]$, then for every $\theta \in \mathbb{R}^d$,
$$T \geq \frac{\operatorname{kl}\left(\delta,1-\delta\right)}{\displaystyle \inf_{\theta' \in \operatorname{Alt}_{\varepsilon}\left(\theta\right)}\frac{1}{2\sigma^2} \Vert \theta'-\theta\Vert^2_{\Lambda}}.$$

\end{proposition}

\begin{proof}[Proof of Proposition \ref{prop:sampling-lower-bound-1}]
The information-theoretic arguments used in the proof of Lemma 19 from \cite{kaufmann2016complexity} ensure that for any $\theta$ and $\theta'$ and any measurable set $E$ with respect to $\sigma(i_1,j_1,r_1,\dots,i_{T},j_{T},r_{T}),$ $$\mathbb{E}_{\theta}[L_{T}] \geq \operatorname{kl}(\mathbb{P}_{\theta}\left(E\right),\mathbb{P}_{\theta'}\left(E\right)).$$
The choices $E=\left\{|v^{\pi}_{\theta}-\hat{v}| \geq \varepsilon \right\}$ and $\theta' \in \operatorname{Alt}_{\varepsilon}\left(\theta\right)$
directly yield $\mathbb{P}_{\theta}(E) \leq \delta$. Furthermore, by the triangular inequality, $|v^{\pi}_{\theta}-\hat{v}| \geq |v^{\pi}_{\theta}-v^{\pi}_{\theta'}| - |v^{\pi}_{\theta'}-\hat{v}|$. Hence, $\mathbb{P}_{\theta'}(E) \geq \mathbb{P}_{\theta'}(|v^{\pi}_{\theta'}-\hat{v}| \geq \varepsilon) \geq 1-\delta$. Since $0 < \delta \leq 1/2$, we have in particular $\delta \leq 1-\delta$ and $\mathbb{P}_{\theta}(E) \leq \mathbb{P}_{\theta'}(E)$. When $x \leq y$, $x \mapsto \operatorname{kl}(x,y)$ is decreasing and $y \mapsto \operatorname{kl}(x,y)$ is increasing. Thus, $\operatorname{kl}(\mathbb{P}_{\theta}\left(E\right),\mathbb{P}_{\theta'}\left(E\right)) \geq \operatorname{kl}\left(\delta,1-\delta\right),$ so that $$\mathbb{E}_{\theta}[L_{T}] \geq \operatorname{kl}\left(\delta,1-\delta\right). $$ Since the reward follows a Gaussian distribution, Lemma \ref{lemma:log-likelihood} yields $$\mathbb{E}_{\theta}[L_{T}]=\frac{T}{2\sigma^2}\sum_{(i,j) \in [m] \times [n]} \omega_{i,j}((\theta-\theta')^\top \psi_{i,j})^2,$$ where we have used that in our setting, $\mathbb{E}_{\theta}\left[N_{i,j}\right]=T\omega_{i,j}$. Combining this result with the inequality above gives $$T \geq \frac{\operatorname{kl}\left(\delta,1-\delta\right)}{\frac{1}{2\sigma^2}(\theta'-\theta)^\top \Lambda(\theta'-\theta)}, $$ and taking the infimum over every $\theta' \in \operatorname{Alt}_{\varepsilon}\left(\theta\right)$ completes the proof.

\end{proof}

This infimum can be computed explicitly with the use of Lemma \ref{lemma:psd-infimum}. It yields the following result.

\begin{theorem}\label{thm:sampling-lower-bound-2}

Let $\delta \in (0,1/2]$ and assume that $\psi_{\pi}:=\sum_{i,j} \omega^{\pi}_{i,j} \psi_{i,j} \in \operatorname{Im}(\Lambda)$. The sample complexity of any $\left(\epsilon,\delta\right)$-PAC estimator of $v^{\pi}$ must satisfy $$ T \geq \frac{\sigma^2\operatorname{kl}\left(\delta,1-\delta\right)}{2\varepsilon^2}\Vert \psi_{\pi}\Vert^2_{\Lambda^{\dagger}}.$$
\end{theorem}

\begin{proof}[Proof of Theorem \ref{thm:sampling-lower-bound-2}]

From Proposition \ref{prop:sampling-lower-bound-1}, we know that $$T \geq \frac{\sigma^2\operatorname{kl}\left(\delta,1-\delta\right)}{\displaystyle \inf_{\theta' \in \operatorname{Alt}_{\varepsilon}\left(\theta\right)}\frac{1}{2} \Vert \theta'-\theta\Vert^2_{\Lambda}}.$$ Let us show that the infimum can be rewritten as the infimum in Lemma \ref{lemma:psd-infimum} for $n=1$, $\Lambda_1=\Lambda$ and $\nu_1=\psi_{\pi}$. First, note that 
since $
v^{\pi}_{\theta}=\theta^\top \left(\sum_{i,j} \omega^{\pi}_{i,j}  \psi_{i,j}\right)$, the constraint $\theta' \in \operatorname{Alt}_{\varepsilon}\left(\theta\right)$ can be rewritten as $|(\theta'-\theta)^\top  \psi_{\pi}| \geq 2\varepsilon$. Since replacing $\theta'-\theta$ by $\theta-\theta'$ yields the same objective, the absolute value can be removed. Moreover, the inequality constraint can be simplified to an equality constraint: if we had $(\theta'-\theta)^\top \psi_{\pi} > 2\varepsilon, $ we could normalize $\theta'-\theta$ to make the objective strictly smaller. Finally, note that the constrained problem only depends on $\theta'$ through $\mu:=\theta'-\theta$. Computing the infimum consequently amounts to solving the quadratic program 
\begin{align}
    \inf_{\mu \in \mathbb{R}^d} & \frac{1}{2}\Vert \mu\Vert^2_{\Lambda} \\
    \text{s.t. } & \mu^\top \psi_{\pi}=2\varepsilon.
\end{align}
Since $\psi_{\pi} \in \operatorname{Im}(\Lambda)$, either $\Vert \psi_{\pi}\Vert_{\Lambda^{\dagger}}=0$ and the bound is trivial, or $\Vert \psi_{\pi}\Vert_{\Lambda^{\dagger}} \neq 0$ and Lemma \ref{lemma:psd-infimum} yields the bound.
    
\end{proof}

\begin{remark}
    The proof shows that interestingly, the bound of Proposition \ref{prop:sampling-lower-bound-1} does not depend on $\theta$.
\end{remark}

\begin{remark}
    If $\psi_{\pi} \notin \operatorname{Im}(\Lambda)$, Lemma \ref{lemma:psd-infimum} entails $T \geq \infty$, and there is no $\left(\epsilon,\delta\right)$-PAC estimator of $v^{\pi}$.
\end{remark}

This sample complexity lower bound can be rewritten as a bound on the optimal estimation accuracy for $T$ samples. When $\psi_{\pi} \in \operatorname{Im}(\Lambda)$ and $\hat{v}$ is an $(\epsilon,\delta)$-PAC estimator of $v^{\pi}$, then 

$$\varepsilon \geq \Vert \psi_{\pi}\Vert_{\Lambda^{\dagger}}\frac{\sigma\sqrt{\operatorname{kl}\left(\delta,1-\delta\right)}}{\sqrt{2T}}.$$ In other words, for any $\delta \in (0,1/2]$, there is no $\left(\varepsilon,\delta\right)-$PAC estimator with $$\varepsilon < \Vert \psi_{\pi}\Vert_{\Lambda^{\dagger}}\frac{\sigma\sqrt{\operatorname{kl}\left(\delta,1-\delta\right)}}{\sqrt{2T}}.$$

\begin{remark}
    Contrary to the lower bound of Theorem 5 in \cite{duan2020minimax} for linear Markov decision processes, our lower bound does not require additional assumptions on the set of contexts.
\end{remark}
\subsection{Matching finite-sample error bound}\label{subsec:linear-error-bound}

We assume in this section that $\Lambda$ is invertible, and the noise is now only assumed $\sigma$-subgaussian. Recall that $v^{\pi}= \theta^\top \psi_{\pi}$ for $\psi_{\pi}=\sum_{i,j}\omega^{\pi}_{i,j}\psi_{i,j}.$ Let $\tau>0.$ Since $\psi_{\pi}$ is known, a natural estimator of $v^{\pi}$ is $\hat{v}=\hat{\theta}^\top \psi_{\pi}$ where $$\hat{\theta}=\left(\sum_{t=1}^{T}\psi_{i_t,j_t}\psi_{i_t,j_t}^\top +\tau I_d\right)^{-1}\left(\sum_{t=1}^T r_t\psi_{i_t,j_t}\right)$$ is a regularized least squares estimator of $\theta$. Note that $\sum_{t=1}^{T}\psi_{i_t,j_t}\psi_{i_t,j_t}^\top+\tau I_d$ is invertible for any $\tau>0$ since $\sum_{t=1}^{T}\psi_{i_t,j_t}\psi_{i_t,j_t}^\top $ is a positive semi-definite matrix with nonnegative eigenvalues. Let $\widehat{\Lambda}_{\tau}:=\frac{1}{T}\left(\sum_{t=1}^{T}\psi_{i_t,j_t}\psi_{i_t,j_t}^\top + \tau I_d\right)$, $L_{\psi}:= \max_{i,j} \psi_{i,j}^\top \Lambda^{-1}\psi_{i,j}$ and $L_{\theta}:=\max_{i,j} \vert \psi_{i,j}^\top \theta \vert.$

\begin{theorem}\label{thm:linear-error-bound}
    Assume that $T \geq 18L_{\psi}\log\left(8d/\delta\right)$ and $\tau \leq L_{\psi}\log(8d/\delta)\lambda_{\min}\left(\Lambda\right).$ With probability at least $1-\delta$, $$\left|v^{\pi}-\hat{v}\right| \leq \frac{1}{\sqrt{T}}E_1+\frac{1}{T}E_2+\frac{1}{T^{3/2}}E_3+\frac{1}{T^2}E_4$$ with $E_1:=\Vert \psi_{\pi}\Vert_{\Lambda^{-1}}\sqrt{2\sigma^2\log(8/\delta)}$, $E_2 :=\Vert \psi_{\pi}\Vert_{\Lambda^{-1}}\left(L_{\psi} L_{\theta}+\sigma(1+4\sqrt{2d})\sqrt{L_{\psi}}\right)\log\left(16(d+1)/\delta\right)$, $E_3:=\Vert \psi_{\pi}\Vert_{\Lambda^{-1}}\left(\frac{\sigma\sqrt{2}}{3}+4\sigma\sqrt{d}+4\sqrt{L_{\psi}}L_{\theta}\right)L_{\psi}\log^{3/2}\left(16(d+1)/\delta\right)$, $E_4:=\Vert \psi_{\pi}\Vert_{\Lambda^{-1}}\frac{4\sigma\sqrt{2}}{3}\sqrt{d}L_{\psi}^{3/2}\log^{2}\left(16(d+1)/\delta\right).$ In particular, $$\left|v^{\pi}-\hat{v}\right| \leq \Vert \psi_{\pi}\Vert_{\Lambda^{-1}}\sqrt{\frac{2\sigma^2\log(8/\delta)}{T}}+O\left(1/T\right)$$ with probability at least $1-\delta.$
\end{theorem}

This error bound matches the sampling lower bound of Theorem \ref{thm:sampling-lower-bound-2} up to an universal constant. This result is similar to the error bound from \cite{duan2020minimax}, but it requires the use of variants of the classical Freedman's inequality adapted to subgaussian random variables, which are stated in Appendix \ref{subsec:linear_lemmas}. This is due to the fact that we assume in our work that the noise is subgaussian instead of bounded. 

\begin{proof}[Proof of Theorem \ref{thm:linear-error-bound}]
Note that $$\begin{aligned}
    \hat{\theta}-\theta&=\left(\sum_{t=1}^{T}\psi_{i_t,j_t}\psi_{i_t,j_t}^\top + \tau I_d\right)^{-1}\left(\sum_{t=1}^{T} (r_t-\psi_{i_t,j_t}^\top \theta)\psi_{i_t,j_t}-\tau \theta\right) \\
    &= \frac{1}{T}\widehat{\Lambda}_{\tau}^{-1}\left(\sum_{t=1}^{T} \xi_t\psi_{i_t,j_t}-\tau \theta\right).
    \end{aligned}
    $$
Thus, we have $$ \begin{aligned}
    \hat{v}-v^{\pi}&=  \psi_{\pi}^\top\left(\hat{\theta}-\theta\right) \\
    &=\frac{1}{T}\psi_{\pi}^\top \widehat{\Lambda}_{\tau}^{-1}\left(\sum_{t=1}^{T} \xi_t\psi_{i_t,j_t}-\tau \theta\right)\\
    &= L_1 + L_2
\end{aligned}$$
for $L_1:=\frac{1}{T}\sum_{t=1}^{T} \xi_t\psi_{\pi}^\top \Lambda^{-1}\psi_{i_t,j_t} $ and $L_2:=\frac{1}{T}\sum_{t=1}^{T} \xi_t\psi_{\pi}^\top (\Lambda^{-1}-\widehat{\Lambda}_{\tau}^{-1})\psi_{i_t,j_t} + \frac{\tau}{T}\psi_{\pi}^\top \widehat{\Lambda}_{\tau}^{-1} \theta.$ \\

\textbf{Control of $L_1$.}
 To control this term, which turns out to be the dominant term in the bound, we will use Proposition \ref{prop:freedman-1}. Let $\mathcal{F}_{t-1}=\sigma(i_1,j_1,r_1,\dots,i_{t-1},j_{t-1},r_{t-1},i_t,j_t)$, $U_t=\psi_{\pi}^\top \Lambda^{-1}\psi_{i_t,j_t},$ $e_t=\xi_tU_t$ and $ V_t=\operatorname{Var}\left(e_t | \mathcal{F}_{t-1}\right).$ We first note that for any choice of $\gamma>0,$ $$\mathbb{P}\left(\left|L_1\right| \geq \varepsilon\right) \leq \mathbb{P}\left(\left|\sum_{t=1}^{T}e_t\right| \geq T\varepsilon \text{ and  }\sum_{t=1}^{T}V_t \leq \gamma^2\right) +\mathbb{P}\left(\sum_{t=1}^{T}V_t \geq \gamma^2\right).$$
 Since $\xi_t$ is independent of $\mathcal{F}_{t-1}$, we have  $$\mathbb{E}\left[e_t | \mathcal{F}_{t-1}\right]=U_t\mathbb{E}\left[\xi_t\right]=0,$$  so that $(e_t)$ is a martingale difference sequence with respect to $(\mathcal{F}_t).$ 
Consequently, $$ \begin{aligned} V_t&=\mathbb{E}\left[e_t^2 | \mathcal{F}_{t-1}\right] \\ &=U_t^2\mathbb{E}\left[\xi_t^2\right]\\ 
&= \sigma^2U_t^2,  \end{aligned} $$  and
$$ \begin{aligned} 
\mathbb{E}\left[e^{ce_t} | \mathcal{F}_{t-1}\right] &= \mathbb{E}\left[e^{c \xi_t U_t} | \mathcal{F}_{t-1}\right] \\
&= \mathbb{E}\left[e^{c\xi_tU_t} | U_t\right] \\
&\leq e^{c^2V_t/2}, \end{aligned} $$  where the second to last inequality follows from the $\mathcal{F}_{t-1}$-measurability of $U_t$ and the independence of $\xi_t$ from $\mathcal{F}_{t-1}$, and the last inequality follows from the fact that $\xi_t$ is $\sigma$-subgaussian and independent of $U_t$, and from $U_t^2=V_t.$ Hence, by Proposition \ref{prop:freedman-1},  $$\mathbb{P}\left(\left|\sum_{t=1}^{T}e_t\right| \geq T \varepsilon \text{ and  }\sum_{t=1}^{T}V_t \leq \gamma^2\right) \leq 2\exp\left(-\frac{T^2\varepsilon^2}{2\gamma^2}\right)$$ for any $\varepsilon,\gamma > 0$.
Furthermore, $\sum_{t=1}^{T} V_t = \sigma^2T\psi_{\pi}^\top \Lambda^{-1}\widehat{\Lambda}_0\Lambda^{-1}\psi_{\pi} \leq \sigma^2T\Vert \psi_{\pi}\Vert^2_{\Lambda^{-1}}\Vert \Lambda^{-1/2}\widehat{\Lambda}_0\Lambda^{-1/2}\Vert_{\op}$. Indeed, since $\Lambda^{-1/2}\widehat{\Lambda}_0\Lambda^{-1/2}$ is a symmetric positive semi-definite matrix, the min-max theorem entails $\Vert \Lambda^{-1/2}\widehat{\Lambda}_0\Lambda^{-1/2}\Vert_{\op}=\lambda_{\max}\left(\Lambda^{-1/2}\widehat{\Lambda}_0\Lambda^{-1/2}\right)\geq \frac{z^\top \Lambda^{-1/2}\widehat{\Lambda}_0\Lambda^{-1/2}z}{z^\top z}$ for $z=\Lambda^{-1/2}\psi_{\pi}.$
Lemma B.5 from \cite{duan2020minimax} consequently ensures that $\sum_{t=1}^{T} V_t \leq \gamma^2$ with probability larger than $1-\delta/4$ when $\gamma$ is chosen such that $$\gamma^2=\sigma^2T \Vert \psi_{\pi}\Vert^2_{\Lambda^{-1}}(1+K_{\delta/4}) $$ where $K_{\delta/4}:=\sqrt{\frac{2\log(8d/\delta)L_{\psi}}{T}}+\frac{2\log(8d/\delta)L_{\psi}}{3T}.$

Finally, we choose $\varepsilon$ adequately with $\varepsilon:=\frac{\gamma}{T}\sqrt{2\log\left(8/\delta\right)}=\Vert \psi_{\pi}\Vert_{\Lambda^{-1}}\sqrt{\frac{2\sigma^2\log(8d/\delta)}{T}}\sqrt{1+K_{\delta/4}}$ so that $$\mathbb{P}\left(\left|L_1\right| \geq \varepsilon\right) \leq  \delta/4+\delta/4=\delta/2.$$ Finally, $\sqrt{1+x} \leq 1+x/2$ for all $x \geq  0$ implies that  $\left|L_1\right| \leq \Vert \psi_{\pi}\Vert_{\Lambda^{-1}}\sqrt{\frac{2\sigma^2\log(8d/\delta)}{T}}\left(1+\frac{K_{\delta/4}}{2}\right)$  with probability larger than $1-\delta/2.$\\

\textbf{Control of $L_2$.} Let us see that $L_2=O(1/T)$ with probability larger than $1-\delta/2$ by adapting the proof of Lemma B.6 from \cite{duan2020minimax}.
Let $\Delta = \Lambda^{-1}-\widehat{\Lambda}_{\tau}^{-1}$ and $W=\frac{1}{T}\sum_{t=1}^{T}\xi_t\psi_{i_t,j_t},$ so that $\frac{1}{T}\sum_{t=1}^{T} \xi_t\psi_{\pi}^\top (\Lambda^{-1}-\widehat{\Lambda}_{\tau}^{-1})\psi_{i_t,j_t}=\psi_{\pi}^\top \Delta W=\left(\psi_{\pi}^\top \Lambda^{-1/2}\right)\left(\Lambda^{1/2}\Delta \Lambda^{1/2}\right)\left(\Lambda^{-1/2}W\right).$
Note also that $\frac{\tau}{T}\psi_{\pi}^\top \widehat{\Lambda}_{\tau}^{-1} \theta=\frac{\tau}{T}\left(\psi_{\pi}^\top \Lambda^{-1/2}\right)\left(\Lambda^{1/2}\widehat{\Lambda}_{\tau}^{-1}B ^{1/2}\Lambda^{-1}\right)\left(\Lambda^{1/2}\theta\right).$ By the Cauchy-Schwartz inequality and the definition of the operator norm, $$ \begin{aligned} \left|L_2\right| &\leq \Vert \psi_{\pi}^\top \Lambda^{-1/2}\Vert_2\Vert \Lambda^{1/2}\Delta \Lambda^{1/2}\Lambda^{-1/2}W\Vert_2 + \frac{\tau}{T}\Vert \psi_{\pi}^\top \Lambda^{-1/2}\Vert_2\Vert \Lambda^{1/2}\widehat{\Lambda}_{\tau}^{-1}\Lambda^{1/2}\Vert_{\op}\Vert \Lambda^{-1}\Vert_{\op}\Vert \Lambda^{1/2}\theta\Vert_2  \\
&\leq \Vert \psi_{\pi}\Vert_{\Lambda^{-1}}\Vert \Lambda^{1/2}\Delta \Lambda^{1/2}\Vert_{\op}\Vert \Lambda^{-1/2}W\Vert_2+\frac{\tau}{\lambda_{\min}\left(\Lambda\right)T}\Vert \psi_{\pi}\Vert_{\Lambda^{-1}}\Vert \Lambda^{1/2}\widehat{\Lambda}_{\tau}^{-1}\Lambda^{1/2}\Vert_{\op}\Vert \Lambda^{1/2}\theta\Vert_2 \\
&\leq \Vert \psi_{\pi}\Vert_{\Lambda^{-1}}\left(\Vert \Lambda^{1/2}\Delta \Lambda^{1/2}\Vert_{\op}\Vert \Lambda^{-1/2}W\Vert_2+\frac{\tau L_{\theta}}{\lambda_{\min}\left(\Lambda\right)T}\left(1+\Vert \Lambda^{1/2}\Delta \Lambda^{1/2}\Vert_{\op}\right)\right),\end{aligned} $$

where we have used  $\Vert \Lambda^{1/2}\theta\Vert^2_2=\theta^\top \Lambda\theta=\sum_{i,j}\omega_{i,j}\left(\psi_{i,j}^\top \theta\right)^2\leq L_{\theta}^2 $ for the last inequality. 

The proof of Lemma B.6. from \cite{duan2020minimax} ensures 
that as long as $T \geq 18L_{\psi}\log\left(8d/\delta\right)$ and $\tau \leq L_{\psi}\log(8d/\delta)\lambda_{\min}\left(\Lambda\right)$, we have $\Vert \Lambda^{1/2}\Delta \Lambda^{1/2}\Vert_{\op} \leq 4\sqrt{\frac{\log(8d/\delta)L_{\psi}}{T}} $ with probability larger than $1-\delta/4$. To control $\Vert \Lambda^{-1/2}W\Vert_2$, we first write  $\Lambda^{-1/2}W=\frac{1}{T}\sum_{t=1}^{T}W_t$ where $W_t:=\Lambda^{-1/2}\xi_t\psi_{i_t,j_t}$ defines a martingale difference sequence in $\mathbb{R}^d$ with respect to the filtration $(\mathcal{F}_t)_{t \geq 0}$ that was defined earlier. We will use the concentration inequality of Proposition \ref{corr:freedman-3} to perform the analysis. 

Let $P_t:=\Lambda^{-1/2}\psi_{i_t,j_t}$. Note that $P_t$ is $\mathcal{F}_{t-1}$-measurable, which implies that
$W_t=\xi_tP_t$ defines a conditionally symmetric martingale difference sequence with respect to $(\mathcal{F}_t)$. In our case, we have $V_{1,t}=\mathbb{E}\left[W_tW_t^\top  | \mathcal{F}_{t-1}\right]=\sigma^2 P_tP_t^\top  \in \mathbb{R}^{d \times d}$  and $V_{2,t}=\mathbb{E}\left[W_t^\top W_t | \mathcal{F}_{t-1}\right]=\sigma^2 P_t^\top P_t \in \mathbb{R}$. Let us choose $Z_{1,t}:=\xi_t\sqrt{P_tP_t^\top }$ and $Z_{2,t}:=\xi_t\sqrt{P_t^\top P_t}$. By Lemma 4.3 of \cite{tropp2011user},
$$\mathbb{E}\left[e^{c Z_{1,t}} | \mathcal{F}_{t-1}\right]=\mathbb{E}\left[e^{c \xi_t\sqrt{P_tP_t^\top }} | P_t\right] \preceq e^{c^2V_{1,t}/2}$$ and $$\mathbb{E}\left[e^{c Z_{2,t}} | \mathcal{F}_{t-1}\right]=\mathbb{E}\left[e^{c \xi_t\sqrt{P_t^\top P_t}} | P_t\right] \preceq e^{c^2V_{2,t}/2}.$$

Since $P_tP_t^\top $ and $P_t^\top P_t$ have the same non-zero eigenvalues, $$ \begin{aligned} \left\Vert \sum_{t=1}^{T}P_tP_t^\top \right\Vert_{\op} \leq \sum_{t=1}^{T}\Vert P_tP_t^\top \Vert_{\op} = \sum_{t=1}^{T}P_t^\top P_t, \end{aligned} $$ so that $$\max\left( \left\Vert \sum_{t=1}^{T}V_{1,t}\right\Vert_{\op},\left\Vert \sum_{t=1}^{T}V_{2,t}\right\Vert_{\op}\right)=\sigma^2\sum_{t=1}^{T}P_t^\top P_t.$$

Applying Corollary \ref{corr:freedman-3} yields
$$ \begin{aligned} \mathbb{P}\left( \Vert \Lambda^{-1/2}W\Vert_{2} \geq \varepsilon \right) &\leq \mathbb{P}\left(\left\Vert \sum_{t=1}^{T}W_t\right\Vert_{2} \geq T\varepsilon \text{ and  }\sum_{t=1}^{T}P_t^\top P_t\leq \gamma^2\right) + \mathbb{P}\left(\sigma^2\sum_{t=1}^{T}P_t^\top P_t \geq \gamma^2 \right) \\
&\leq 2(d+1)\exp\left(-\frac{T^2\varepsilon^2}{2\gamma^2}\right)+\mathbb{P}\left(\sigma^2\sum_{t=1}^{T}P_t^\top P_t \geq \gamma^2 \right) \end{aligned} $$ for any $\varepsilon,\gamma > 0$. We can control $\sum_{t=1}^{T}P_t^\top P_t$ by adapting the proof of Lemma B.9 from \cite{duan2020minimax}. We have $$ \begin{aligned} \sum_{t=1}^{T}P_t^\top P_t &=\sum_{t=1}^{T}\psi_{i_t,j_t}\Lambda^{-1}\psi_{i_t,j_t} \\
&= Td+T\operatorname{Tr}\left(\Lambda^{-1/2}\widehat{\Lambda}_{\tau}\Lambda^{-1/2}-I_d\right) \\
&\leq Td\left(1+\Vert \Lambda^{-1/2}\hat{\Lambda}_{\tau}\Lambda^{-1/2}-I_d\Vert_{\op}\right).\end{aligned} $$ 
By Lemma B.5. of \cite{duan2020minimax}, $\sigma^2\sum_{t=1}^{T} P_t^\top P_t \leq \gamma^2$ with probability larger than $1-\delta/8$ when $\gamma$ is such that $$\gamma^2=\sigma^2Td(1+K_{\delta/8}). $$ 

Then, we let $\varepsilon:=\frac{\gamma}{T}\sqrt{2\log\left(16(d+1)/\delta\right)}$ so that $\mathbb{P}\left(\Vert \Lambda^{-1/2}W\Vert_2 \geq \varepsilon\right) \leq  \delta/8+\delta/8=\delta/4.$
In particular, $\Vert \Lambda^{-1/2}W\Vert_2 \geq \sqrt{\frac{2d\sigma^2\log\left(16(d+1)/\delta\right)}{T}}\left(1+\frac{K_{\delta/8}}{2}\right)$ with probability smaller than $\delta/4$. To ease the notation, let $C_1:=\log(8d/\delta)$ and $C_2:=\log(16(d+1)/\delta)$. Combining the previous result with the bound on $\Vert \Lambda^{1/2}\Delta \Lambda^{1/2}\Vert_{\op}$ yields that with probability larger than $1-\delta/2,$ 
$$\begin{aligned}
    \left|L_2\right| &\leq \Vert \psi_{\pi}\Vert_{\Lambda^{-1}}\left(\Vert \Lambda^{1/2}\Delta \Lambda^{1/2}\Vert_{\op}\Vert \Lambda^{-1/2}W\Vert_2+\frac{\tau L_{\theta}}{\lambda_{\min}\left(\Lambda\right)T}\left(1+\Vert \Lambda^{1/2}\Delta \Lambda^{1/2}\Vert_{\op}\right)\right) \\
    &\leq \frac{\Vert \psi_{\pi}\Vert_{\Lambda^{-1}}}{T}\left(4\sqrt{2d\sigma^2C_2C_1L_{\psi}}\left(1+\frac{K_{\delta/8}}{2}\right)+\frac{\tau L_{\theta}}{\lambda_{\min}\left(\Lambda\right)}\left(1+4\sqrt{\frac{C_1L_{\psi}}{T}}\right)\right) \\
    &\leq \frac{\Vert \psi_{\pi}\Vert_{\Lambda^{-1}}}{T}\left(4\sqrt{2d\sigma^2C_2C_1L_{\psi}}\left(1+\frac{K_{\delta/8}}{2}\right)+L_{\psi}L_{\theta}C_1\left(1+4\sqrt{\frac{C_1L_{\psi}}{T}}\right)\right),
\end{aligned}$$ 
where we used $\tau \leq L_{\psi}C_1\lambda_{\min}\left(\Lambda\right)$ for the last inequality. In the end, $$ \begin{aligned} \left|v^{\pi}-\hat{v}\right| &\leq \left|L_1\right|+\left|L_2\right| \\ &\leq \Vert \psi_{\pi}\Vert_{\Lambda^{-1}}\sqrt{\frac{2\sigma^2C_1}{T}}\left(1+\frac{K_{\delta/4}}{2}\right)+\frac{\Vert \psi_{\pi}\Vert_{\Lambda^{-1}}}{T}4\sqrt{2d\sigma^2L_{\psi}}C_2\left(1+\frac{K_{\delta/8}}{2}\right)  \\
&+ \frac{\Vert \psi_{\pi}\Vert_{\Lambda^{-1}} L_{\psi} L_{\theta}C_1}{T}\left(1+4\sqrt{\frac{C_1L_{\psi}}{T}}\right) \\
&\leq \Vert \psi_{\pi}\Vert_{\Lambda^{-1}}\left(\sqrt{\frac{2\sigma^2C_1}{T}}+\frac{4\sqrt{2d\sigma^2L_{\psi}}}{T}C_2\right)\left(1+\sqrt{\frac{C_2L_{\psi}}{2T}}+\frac{C_2L_{\psi}}{3T}\right) \\
&+ \frac{\Vert \psi_{\pi}\Vert_{\Lambda^{-1}} L_{\psi} L_{\theta}C_1}{T}\left(1+4\sqrt{\frac{C_1L_{\psi}}{T}}\right) \\
& \leq \frac{1}{\sqrt{T}}E_1 + \Vert \psi_{\pi}\Vert_{\Lambda^{-1}}\frac{4\sqrt{2d\sigma^2L_{\psi}}}{T}C_2 + \frac{L_{\psi} L_{\theta}C_1}{T}\Vert \psi_{\pi}\Vert_{\Lambda^{-1}}\left(1+4\sqrt{\frac{C_1L_{\psi}}{T}}\right)\\ &+\Vert \psi_{\pi}\Vert_{\Lambda^{-1}}\left(\sqrt{\frac{2\sigma^2C_2}{T}}+\frac{4\sqrt{2d\sigma^2L_{\psi}}}{T}C_2\right)\left(\sqrt{\frac{C_2L_{\psi}}{2T}}+\frac{C_2L_{\psi}}{3T}\right)  \\
&\leq \frac{1}{\sqrt{T}}E_1+\frac{1}{T}E_2+\frac{1}{T^{3/2}}E_3+\frac{1}{T^2}E_4 \end{aligned} $$ with probability larger than $1-\delta.$

\end{proof}

\subsection{Additional Lemmas}\label{subsec:linear_lemmas}

The following result is an extension of Lemma 9 from \cite{taupin2022best} to positive semi-definite matrices.

\begin{lemma}\label{lemma:psd-infimum}
For any $\nu_1,\dots,\nu_n \in \mathbb{R}^d$ and any symmetric positive semi-definite matrices $\Lambda_1, \dots, \Lambda_n \in \mathbb{R}^{d \times d}$, we have
$$ \inf_{\substack{\mu \in \mathbb{R}^{n \times d}, \\ \sum_{i=1}^{n}\mu_i^\top \nu_i=2\varepsilon}} \frac{1}{2} \sum_{i=1}^{n}\Vert \mu_i\Vert^2_{\Lambda_i} = \begin{cases}
 \frac{2\varepsilon^2}{\sum_{i=1}^{n}\Vert \nu_i\Vert^2_{\Lambda_i^{\dagger}}} & \text{ if } \forall i \in [n], \nu_i \in \operatorname{Im}\left(\Lambda_i\right) \text{ and } \sum_{i=1}^{n}\Vert \nu_i\Vert^2_{\Lambda_i^{\dagger}} \neq 0\\
  0 & \text{ if } \exists j \in [n], \nu_j \notin \operatorname{Im}\left(\Lambda_j\right). 
\end{cases}.$$

\end{lemma}

\begin{proof}[Proof of Lemma \ref{lemma:psd-infimum}]
Let $L(\mu,\lambda):=\frac{1}{2}\sum_{i=1}^{n}\mu_i^\top \Lambda\mu_i - \lambda\left(\sum_{i=1}^{n}\mu_i^\top \nu_i-2\varepsilon\right)$ be the Lagrangian associated to the quadratic program.
The KKT conditions yield $\Lambda_i\mu_i=\lambda \nu_i$ and $\sum_{i=1}^{n}\mu_i^\top \nu_i=2\varepsilon.$

\begin{itemize}
   
\item If $\nu_i \in \operatorname{Im}\left(\Lambda_i\right)$ for every $i \in [n]$, $\lambda \Lambda_i^{\dagger}\nu_i$ is a solution of $\Lambda_i\mu_i=\lambda \nu_i$ for any $\lambda.$
Thus, for $\lambda^*=\frac{2\varepsilon}{\sum_{i=1}^{n}\nu_i^\top \Lambda_i^{\dagger}\nu_i}$ and $\mu_i^*=\lambda_*\Lambda_i^{\dagger}\nu_i$, we have 
$\Lambda_i\mu_i^*=\lambda^* \nu_i$ and 
$$ \begin{aligned}
    \sum_{i=1}^{n}{\mu_i^{*}}^\top \nu_i &=\frac{2\varepsilon}{\sum_{i=1}^{n}\nu_i^\top \Lambda_i^{\dagger}\nu_i}\sum_{i=1}^{n}\left(\Lambda_i^{\dagger}\nu_i\right)^\top \nu_i \\ 
    &= 2\varepsilon ,
\end{aligned} $$
where the last equality follows from the fact that $\Lambda_i^{\dagger}$ is symmetric (since $\Lambda_i$ is symmetric). It follows that $\mu^*$ is optimal, and the infimum equals $\frac{1}{2}\sum_{i=1}^{n}{\mu_i^*}^\top \Lambda_i\mu_i^*=\frac{1}{2}\lambda^*\sum_{i=1}^{n}{\mu_i^*}^\top \nu_i=\frac{2\varepsilon^2}{\sum_{i=1}^{n}\nu_i^\top \Lambda_i^{\dagger}\nu_i}.$ Note that $\Lambda_i^{\dagger}$ is positive semi-definite since $\Lambda_i$ is, so we have $\nu_i^\top \Lambda_i^{\dagger}\nu_i=\Vert \nu_i\Vert^2_{\Lambda_i^{\dagger}}.$ \\

\item If $\nu_j \notin \operatorname{Im}\left(\Lambda_j\right)$ for some $j \in [n]$, $\Lambda_j\mu_j=\lambda \nu_j$ implies $\lambda=0$. Thus, any optimal point $\mu^*$ must verify $\Lambda_i\mu_i^*=0$ for every $i \in [n]$ which implies that $\frac{1}{2}\sum_{i=1}^{n}{\mu_i^*}^\top \Lambda_i\mu_i^*=0$.
\end{itemize}
\end{proof}

The following concentration inequalities are useful in the analysis of the PE error bounds.
\begin{proposition}[Corollary of Theorem 2.6 from \cite{fan2015exponential}]\label{prop:freedman-1}
Let $\left(e_t\right)_{t=1,\dots,T}$ be a martingale difference sequence with respect to a filtration $(\mathcal{F}_t)_{t=0,\dots,T-1}$ (i.e., $\mathbb{E}\left[e_t | \mathcal{F}_{t-1}\right]=0$ for every $t$), and let $V_t=\operatorname{Var}\left(e_t | \mathcal{F}_{t-1}\right)$. If for all $c \geq 0$ and all $t=1,\dots,T$, we have  $\mathbb{E}\left[e^{c e_t} | \mathcal{F}_{t-1}\right] \leq e^{c^2V_t/2}$, then for all $\varepsilon,\gamma > 0$,
$$ \mathbb{P}\left(\left|\sum_{t=1}^{T}e_t\right| \geq \varepsilon \text{ and  }\sum_{t=1}^{T}V_t \leq \gamma^2\right) \leq 2\exp\left(-\frac{\varepsilon^2}{2\gamma^2}\right).$$
\end{proposition}

\begin{proof}[Proof of Proposition \ref{prop:freedman-1}]
    Applying Theorem 2.6 from \cite{fan2015exponential} with $V_t=\operatorname{Var}\left(e_t | \mathcal{F}_{t-1}\right)$, $f(c)=c^2/2$ yields, after optimizing the bound in $c,$ $$ \mathbb{P}\left(\sum_{t=1}^{T}e_t \geq \varepsilon \text{ and  }\sum_{t=1}^{T}V_t \leq \gamma^2\right) \leq \exp\left(-\frac{\varepsilon^2}{2\gamma^2}\right).$$

    The same bound applies to the martingale difference sequence $(-e_t)$, and a union bound concludes the proof.
\end{proof}

The following concentration inequality is an extension of Proposition \ref{prop:freedman-1} to martingale difference sequences of $\mathbb{R}^{d \times d'}$ for any $d' \geq 1$. It can be obtained as a corollary of Theorem 2.3 in \cite{tropp2011freedman}.

\begin{proposition}\label{prop:freedman-2}
Let $\left(W_t\right)_{t=1,\dots,T}$ be a conditionally symmetric\footnote{This means that $W_t$ and $-W_t$ have the same conditional distribution with respect to $\mathcal{F}_{t-1}$ for every $t$.} martingale difference sequence of $\mathbb{R}^{d \times d'}$ with respect to a filtration $(\mathcal{F}_t)_{t=0,\dots,T-1}$, and let $V_{1,t} \in \mathbb{R}^{d \times d}$, $V_{2,t} \in \mathbb{R}^{d' \times d'}$ be $\mathcal{F}_{t-1}$-measurable symmetric matrices. Denote by $Z_{1,t} \in \mathbb{R}^{d \times d}$ and $Z_{2,t} \in \mathbb{R}^{d' \times d'}$ any symmetric matrices such that $Z_{1,t}^2=W_tW_t^\top $ and $Z_{2,t}^2=W_t^\top W_t$. If for some nonnegative function $g$, for all $c \in \mathbb{R}$ and all $t=1,\dots,T$,  $\mathbb{E}\left[e^{c Z_{1,t}} | \mathcal{F}_{t-1}\right] \preceq  e^{g(c)V_{1,t}}$\footnote{Recall that $N \preceq M$ means that $M-N$ is positive semi-definite.} and $\mathbb{E}\left[e^{c Z_{2,t}} | \mathcal{F}_{t-1}\right] \preceq e^{g(c)V_{2,t}}$, then for all $\varepsilon,\gamma > 0$,
$$ \mathbb{P}\left(\lambda_{\max}\left(\sum_{t=1}^{T}W_t\right) \geq \varepsilon \text{ and  }\max\left(\lambda_{\max}\left(\sum_{t=1}^{T}V_{1,t}\right),\lambda_{\max}\left(\sum_{t=1}^{T}V_{2,t}\right)\right) \leq \gamma^2\right) \leq (d+d')\inf_{c \in \mathbb{R}} \exp\left(-c \varepsilon + g(c)\gamma^2\right).$$
\end{proposition}

\begin{proof}[Proof of Proposition \ref{prop:freedman-2}]
Let $Y_t:=\begin{pmatrix}
0 & W_t\\
W_t^\top  & 0
\end{pmatrix}$ and $V_t:=\begin{pmatrix}
V_{1,t} & 0\\
0 & V_{2,t}
\end{pmatrix}$. We will use Theorem 2.3 of \cite{tropp2011freedman} on the martingale difference sequence $(Y_t)_{t \geq 0}$ of $\mathbb{R}^{(d+d')^2}$. To do so, we must prove that for all $c \in \mathbb{R}$ , $\mathbb{E}\left[e^{c Y_t} | \mathbb{F}_{t-1}\right] \preceq e^{g(c)V_t}.$ It is straightforward to check that for all $k\geq 1$, $$Y_t^{2k}=\begin{pmatrix}
\left(W_tW_t^\top \right)^k & 0\\
0 & \left(W_t^\top W_t\right)^k
\end{pmatrix}$$ and $$Y_t^{2k+1}=\begin{pmatrix}
0 & \left(W_tW_t^\top \right)^kW_t\\
\left(W_tW_t^\top \right)^kW_t^\top  & 0
\end{pmatrix}.$$ Note also that the conditional symmetry of $W_t$ implies that $\mathbb{E}\left[Y_t^{2k+1} | \mathcal{F}_{t-1}\right]=0$. Consequently, $$ \begin{aligned}
    \mathbb{E}\left[e^{c Y_t} | \mathcal{F}_{t-1}\right]&=\mathbb{E}\left[\sum_{k=0}^{+\infty}\frac{c^{2k}}{(2k)!}Y_t^{2k}+\sum_{k=0}^{+\infty}\frac{c^{2k+1}}{(2k+1)!}Y_t^{2k+1} | \mathcal{F}_{t-1}\right] \\
&= \mathbb{E}\left[\sum_{k=0}^{+\infty}\frac{c^{2k}}{(2k)!}\begin{pmatrix}
\left(W_tW_t^\top \right)^k & 0\\
0 & \left(W_t^\top W_t\right)^k
\end{pmatrix} | \mathcal{F}_{t-1}\right]  \\
&= \mathbb{E}\left[\begin{pmatrix}
\operatorname{cosh}\left(c Z_{1,t}\right) & 0\\
0 & \operatorname{cosh}\left(c Z_{2,t}\right)
\end{pmatrix} | \mathcal{F}_{t-1}\right].
\end{aligned}$$

From the assumptions of Proposition \ref{prop:freedman-2}, we have $$ \begin{aligned} \mathbb{E}\left[\operatorname{cosh}\left(c Z_{1,t}\right) | \mathcal{F}_{t-1}\right]&=\frac{1}{2}\left(\mathbb{E}\left[\exp\left(c Z_{1,t}\right)| \mathcal{F}_{t-1}\right]  + \mathbb{E}\left[\exp\left(-c Z_{1,t}\right)| \mathcal{F}_{t-1}\right]\right)   \\ &\preceq e^{g(c)V_{1,t}} \end{aligned} $$ and similarly, $$\mathbb{E}\left[\operatorname{cosh}\left(c Z_{1,t}\right) | \mathcal{F}_{t-1}\right] \preceq e^{g(c)V_{2,t}}.$$

This finally guarantees that
 $$\begin{aligned} 
 \mathbb{E}\left[e^{c Y_t} | \mathcal{F}_{t-1}\right]
&\preceq \begin{pmatrix}
e^{g(c)V_{1,t}} & 0\\
0 & e^{g(c)V_{2,t}}
\end{pmatrix} \\
&= e^{g(c)V_t}.
\end{aligned}$$
We can now apply Theorem 2.3 of \cite{tropp2011freedman}. It ensures that
$$\mathbb{P}\left(\lambda_{\max}\left(\sum_{t=1}^{T}Y_t\right) \geq \varepsilon \text{ and  }\lambda_{\max}\left(\sum_{t=1}^{T}V_{t}\right) \leq \gamma^2\right) \leq (d+d')\inf_{c \in \mathbb{R}} \exp\left(-c \varepsilon + g(c)\gamma^2\right).$$
Note that $$\lambda_{\max}\left(\sum_{t=1}^{T}Y_t\right)=\lambda_{\max}\left(\sum_{t=1}^{T}W_t\right) \quad \text{and} \quad \lambda_{\max}\left(\sum_{t=1}^{T}V_{t}\right)=\max\left(\lambda_{\max}\left(\sum_{t=1}^{T}V_{1,t}\right),\lambda_{\max}\left(\sum_{t=1}^{T}V_{2,t}\right)\right),$$ which concludes the proof.
\end{proof}

The previous general result directly implies the following concentration inequality. 
\begin{corollary}\label{corr:freedman-3}
Let $\left(W_t\right)_{t=1,\dots,T}$ be a conditionally symmetric martingale difference sequence of $\mathbb{R}^{d \times d'}$ with respect to a filtration $(\mathcal{F}_t)_{t=0,\dots,T-1}$, and let $V_{1,t}:=\mathbb{E}\left[W_tW_t^\top  | \mathcal{F}_{t-1}\right]$, $V_{2,t}:=\mathbb{E}\left[W_t^\top W_t | \mathcal{F}_{t-1}\right]$. Denote by $Z_{1,t} \in \mathbb{R}^{d \times d}$ and $Z_{2,t} \in \mathbb{R}^{d' \times d'}$ any symmetric matrices such that $Z_{1,t}^2=W_tW_t^\top $ and $Z_{2,t}^2=W_t^\top W_t$. If for all $c \in \mathbb{R}$ and all $t=1,\dots,T$,  $\mathbb{E}\left[e^{c Z_{1,t}} | \mathcal{F}_{t-1}\right] \preceq e^{c^2V_{1,t}/2}$ and $\mathbb{E}\left[e^{c Z_{2,t}} | \mathcal{F}_{t-1}\right] \preceq e^{c^2V_{2,t}/2}$, then for all $\varepsilon,\gamma > 0$,
$$ \mathbb{P}\left( \left\Vert \sum_{t=1}^{T}W_t \right\Vert_{\op} \geq \varepsilon\ \mathrm{and}\ \max\left( \left\Vert \sum_{t=1}^{T}V_{1,t}\right\Vert_{\op}, \left\Vert \sum_{t=1}^{T}V_{2,t} \right\Vert_{\op}\right) \leq \gamma^2\right) \leq 2(d+d')\exp\left(-\frac{\varepsilon^2}{2\gamma^2}\right).$$
\end{corollary}

\begin{proof}[Proof of Corollary \ref{corr:freedman-3}] We choose $g(c)=c^2/2$ in the statement of Proposition \ref{prop:freedman-2}. Note that $\displaystyle \inf_{c \in \mathbb{R}} \exp\left(-c \varepsilon + g(c)\gamma^2\right)=\exp\left(-\frac{\varepsilon^2}{2\gamma^2}\right)$, and that $V_{1,t}$ and $V_{2,t}$ are positive semi-definite matrices, which implies that $\lambda_{\max}\left(\sum_{t=1}^{T}V_{1,t}\right)=\Vert \sum_{t=1}^{T}V_{1,t}\Vert_{\op}$ and $\lambda_{\max}\left(\sum_{t=1}^{T}V_{2,t}\right)=\Vert \sum_{t=1}^{T}V_{2,t}\Vert_{\op}$. 
The same bound applies to the conditionally symmetric martingale difference sequence $\left(-W_t\right)$, and a union bound concludes the proof. 
\end{proof}

\newpage
\section{Proofs of results from Section \ref{subsec:dsm-pe} and \ref{subsec:rs-pe} : Policy Evaluation Upper Bounds}\label{app:pe}

\subsection{Proof of Theorem \ref{thm:dsm-pe-error}}\label{subsec:proof-dsm-pe-error}

We first note that the max-norm guarantee\footnote{The max-norm is denoted by $\Vert .\Vert_{\infty}$ in \cite{stojanovic2024spectral}.} of Theorem 1 of \cite{stojanovic2024spectral} can be generalized to the case where the samples are not necessarily selected uniformly at random by leveraging the two-to-infinity norm guarantees of Theorem \ref{thm:recovery-two-to-infinity-norm}. We state the general result below, where we recall that $L=\Vert M \Vert_{\max} \vee \sigma$ (note that the dependency in $\mu$ is also improved compared to Theorem 1 of \cite{stojanovic2024spectral}).

\begin{proposition}\label{prop:dsm-max-error} Let $\delta \in \left(0,1\right)$. If $T \geq \frac{cL^2\kappa^2(m+n)}{\sigma_r^2\omega_{\min}}\log^3\left(\frac{(m+n)T}{\delta}\right)$ for some universal constant $c>0$, then with probability larger than $1-\delta$, $$\Vert M-\widehat{M}\Vert_{\max} \lesssim \sqrt{\frac{L^2\mu^6\kappa^4 r^3 (m+n) }{T \omega_{\min}(m\wedge n)^2} \log^3\left( \frac{(m+n)T}{\delta}\right)}. $$
\end{proposition}
\begin{proof}[Proof of Proposition \ref{prop:dsm-max-error}]

The proof is divided in two steps:  first, we leverage the bound on $\epsilon_{\textup{Sub-Rec}}$ of Theorem \ref{thm:recovery-two-to-infinity-norm} to bound $\Vert M - \widehat{M} \Vert_{2 \to \infty},$ and then we bound $\Vert M - \widehat{M} \Vert_{\max}$ by relating it to $\epsilon_{\textup{Sub-Rec}}$ and $\Vert M - \widehat{M} \Vert_{2 \to \infty}.$

\paragraph{Bounding $\Vert M - \widehat{M} \Vert_{2 \to \infty}.$}

Let $E:=M-\widetilde{M}$. From (45) in \cite{stojanovic2024spectral},
\begin{equation*}
\Vert M - \widehat{M} \Vert_{2 \to \infty} \leq \epsilon_{\textup{Sub-Rec}}\left(\Vert E \Vert_{\op} + \Vert M \Vert_{\op} \right)+\Vert U \Vert_{2 \to \infty}\Vert E \Vert_{\op}.
\end{equation*}
Hence\footnote{Union bounds are performed throughout the proof. Since the relevant inequalities are stated up to a constant, they are not impacted.}, with probability larger than $1-\delta$, from Proposition \ref{prop:reward-concentration-1}, and Theorem \ref{thm:recovery-two-to-infinity-norm},
\begin{align*}\Vert M - \widehat{M} \Vert_{2 \to \infty} &\lesssim  \sigma_1\epsilon_{\textup{Sub-Rec}}+\mu\sqrt{\frac{r}{m}}\Vert E \Vert_{\op} \\ 
&\lesssim \kappa \sqrt{\frac{L^2\mu^2 \kappa^2 r (m+n) }{T \omega_{\min}(m\wedge n)} \log^3\left( \frac{(m+n)T}{\delta}\right)} \\ &+ \mu\sqrt{\frac{r}{m}}L\sqrt{\frac{m+n}{T\omega_{\min}}\log^3\left( \frac{(m+n)T}{\delta}\right)}
\\
&\lesssim \sqrt{\frac{L^2\mu^2\kappa^4 r (m+n) }{T \omega_{\min}(m\wedge n)} \log^3\left( \frac{(m+n)T}{\delta}\right)}
\end{align*}
under $T \gtrsim \frac{L^2(m+n)}{\sigma_r^2\omega_{\min}}\log^3\left(\frac{(m+n)T}{\delta}\right)$ (note that by Lemma \ref{lem:spikiness}, this condition is stronger than the condition of Proposition \ref{prop:reward-concentration-1}).

\paragraph{Bounding $\Vert M - \widehat{M} \Vert_{\max}.$}

We similarly write (see \cite{stojanovic2024spectral}),
\begin{equation*}
\Vert M - \widehat{M} \Vert_{\max} \leq \epsilon_{\textup{Sub-Rec}}(\Vert M-\widehat{M} \Vert_{2 \to \infty} + \Vert M \Vert_{2 \to \infty})+\Vert V \Vert_{2 \to \infty} \Vert M-\widehat{M} \Vert_{2 \to \infty}.
\end{equation*}

Note that by Lemma \ref{lem:spikiness}, $\Vert M \Vert_{2 \to \infty} \leq \sqrt{m}\Vert M \Vert_{\max} \leq \sigma_1\frac{\mu^2 r}{\sqrt{n}}$. Furthermore, under 
\begin{equation}\label{samples_cond}T \gtrsim \frac{L^2\kappa^2(m+n)}{\sigma_r^2\omega_{\min}}\log^3\left(\frac{(m+n)T}{\delta}\right),\end{equation}
we know that by Theorem \ref{thm:recovery-two-to-infinity-norm},  $\epsilon_{\textup{Sub-Rec}} \lesssim \frac{\mu  \sqrt{r}}{\sqrt{m \wedge n}}$ with probability larger than $1-\delta$. Thus, with probability larger than $1-\delta$, under \eqref{samples_cond},

\begin{align*}\Vert \widehat{M}-M \Vert_{\max}
&\lesssim \left(\epsilon_{\textup{Sub-Rec}}+\mu\sqrt{\frac{r}{n}}\right)\Vert M-\widehat{M} \Vert_{2 \to \infty} + \sigma_1\epsilon_{\textup{Sub-Rec}}\frac{\mu^2r}{\sqrt{n}}\\
&\lesssim \mu\sqrt{\frac{r}{m \wedge n}} \sqrt{\frac{L^2\mu^2\kappa^4 r (m+n) }{T \omega_{\min}(m\wedge n)} \log^3\left( \frac{(m+n)T}{\delta}\right)} \\
&+ \frac{\sigma_1}{\sigma_r}\frac{\mu^2r}{\sqrt{n}} \sqrt{\frac{L^2\mu^2\kappa^2 r (m+n) }{T \omega_{\min}(m\wedge n)} \log^3\left( \frac{(m+n)T}{\delta}\right)} \\
&\lesssim \sqrt{\frac{L^2\mu^6\kappa^4 r^3 (m+n) }{T \omega_{\min}(m\wedge n)^2} \log^3\left( \frac{(m+n)T}{\delta}\right)}.\end{align*}

\end{proof}

\begin{proof}[Proof of Theorem \ref{thm:dsm-pe-error}]

 Since $\sum_{i,j} \omega^{\pi}_{i,j}=1,$ we have      
   $\vert v^{\pi}-\hat{v}_{\dsmpe} \vert=\displaystyle \Big\vert \sum_{i,j}\omega^{\pi}_{i,j}\left(M_{i,j}-\widehat{M}_{i,j}\right)\Big\vert \leq \Vert M-\widehat{M} \Vert_{\max}.$ Proposition \ref{prop:dsm-max-error} entails that 
   $\vert v^{\pi}-\hat{v}_{\dsmpe}\vert \leq \varepsilon$  with probability larger than $1-\delta$ as soon as $$T \gtrsim  \frac{L^2\mu^{6} \kappa^4 r^3 (m+n)}{\omega_{\min} (m\wedge n)^2 \varepsilon^2}\log^3\!\left(  \frac{(m+n)T}{\delta}\right)\!$$ and $$T \gtrsim \frac{L^2\kappa^2(m+n)}{\sigma_r^2\omega_{\min}}\log^3\left(\frac{(m+n)T}{\delta}\right).$$
   By Lemma \ref{lem:spikiness}, $\frac{\kappa^2}{\sigma_r^2} \leq \frac{\mu^4\kappa^2r^2}{mn\Vert M \Vert_{\max}^2} \leq \frac{\mu^6\kappa^4r^3}{(m \wedge n)^2\Vert M \Vert_{\max}^2}$. Since  $\varepsilon \leq \Vert M \Vert_{\max}$, the second condition is weaker and can be dropped.
\end{proof}

\subsection{Proof of Theorem \ref{thm:rs-pe-error}}\label{subsec:proof-rs-pe-error}
Recall that $M_{i,j}=\phi_{i,j}^\top \theta + \epsilon_{i,j}.$ In this appendix, the following additional notations will be used: 
$$\Lambda= \sum_{i,j}\omega_{i,j}\psi_{i,j}\psi_{i,j}^\top, \Lambda_{\phi}=\sum_{i,j}\omega_{i,j}\phi_{i,j}\phi_{i,j}^\top, 
 \widehat{\Lambda}_{\tau}=\frac{1}{T-T_1}\left(\sum_{t=T_1+1}^{T}\phi_{i_t,j_t}\phi_{i_t,j_t}^\top + \tau I_d \right) \text{ and } \phi_{\pi}=\sum_{i,j} \omega^{\pi}_{i,j}\phi_{i,j}.$$ We also recall that $$\epsilon_{\max}=\max_{i \in [m],j \in [n]}\vert \epsilon_{i,j} \vert $$  is the misspecification of the misspecified linear bandit, and that the latter is of dimension $d=r(m+n)-r^2$. Finally, to ease the notation, we do not write the dependency in $\pi$ of the estimator $\hat{v}^{\pi}_{\rspe}$.
 We summarize \rspe\ in Algorithm \ref{algo:RSPE}.
 
\begin{algorithm}[th!]
\caption{\textsc{\textbf{R}ecover \textbf{S}ubspace for \textbf{P}olicy \textbf{E}valuation} $(\rspe)$ }  \label{algo:RSPE}
\begin{algorithmic}
\STATE \textbf{Input}: Samples $\left(i_t,j_t,r_t\right)_{t \in [T]}$, context distribution $\rho$, behavior policy $\pi^b$, target policy $\pi$,  sample size $T_1$, regularization parameter $\tau$
\STATE \emph{\color{purple}\underline{Phase 1:} Subspace recovery}   \STATE Use the first $T_1$ samples to construct 
$\widehat{U},\widehat{V}$ as in \eqref{eq:SVD_hat_M}
\STATE \emph{\color{purple}\underline{Phase 2:} Solving a misspecified linear bandit}
\STATE Use $\widehat{U},\widehat{V}$ to construct $\lbrace \phi_{i,j}: (i,j) \in [m]\times[n] \rbrace$ as in \eqref{eq:phi}
\STATE Use the remaining $T - T_1$ samples to construct the least squares estimator $\hat{\theta}$ as in \eqref{eq:rspe-lse}
\STATE \textbf{Output:} $\hat{v}^\pi_{\rspe}=\sum_{i,j}\omega^{\pi}_{i,j}\phi_{i,j}^T\hat{\theta}$.
\end{algorithmic}
\end{algorithm}

The general idea of the proof of Theorem \ref{thm:rs-pe-error} is to decompose the PE error of \rspe\ as a sum of a term that depends on the misspecification error and a term incurred by the error between $\hat{\theta}$ and $\theta$. The analysis of the latter term, which is in fact the largest contributor to the overall error, is heavily inspired by the analysis in Appendix \ref{subsec:linear-error-bound}. We first derive a high probability guarantee with the lower-order term that one would expect from the contextual linear bandit analysis, i.e., roughly $\Vert \phi_{\pi}\Vert_{\Lambda_{\phi}^{-1}}\sqrt{\frac{\log\left(1/\delta\right)}{T}}$. This bound is however not satisfactory in itself since it is a function of the randomness of the first phase of the algorithm. This issue can be resolved by leveraging the subspace recovery guarantees of Theorem \ref{thm:recovery-two-to-infinity-norm} to show that $\Vert \phi_{\pi}\Vert_2$ converges towards $\Vert \psi_{\pi}\Vert_2$. 

The majority of this section is thus devoted to proving the following intermediary result.

\begin{proposition}\label{prop:rs-pe-random-error} Assume that $T_1 = \lfloor \alpha T \rfloor$ for some $\alpha \in \left(0,1\right)$ and that the regularization parameter of $\widehat{\Lambda}_{\tau}$ verifies $\tau \leq \frac{r}{\left(m \wedge n\right)}\frac{\omega_{\min}}{\omega_{\max}}\log\left(16d/\delta\right)$. There exists a universal constant $C>0$ such that as long as $T \geq \frac{CL^2\kappa^2(m+n)}{\sigma_r^2\omega_{\min}}\log^3\left(\frac{(m+n)T}{\delta}\right)$ and $T \geq \frac{C}{\omega_{\min}}\log\left(16d/\delta\right)$, then with probability larger than $1-\delta$, $$ \begin{aligned} \left|v^{\pi}-\hat{v}_{\rspe}\right| \leq \Vert \phi_{\pi}\Vert_{\Lambda_{\phi}^{-1}}\sqrt{\frac{2\sigma^2\log(16/\delta)}{\left(1-\alpha\right)T}} +K_0\frac{\log^3\left(\left(m+n\right)T/\delta\right)}{T}
\end{aligned}$$ 
where 
\begin{align}\label{eq:def_of_K0} K_0 \lesssim \frac{L^2}{\Vert M \Vert_{\max}}\frac{\mu^{6}\kappa^{4}r^{3}\left(m+n\right)^{2}\sqrt{mn}}{\left(\omega_{\min}mn\right)^2\left(m \wedge n\right)}.\end{align}

\end{proposition} 

\begin{proof}[Proof of Proposition \ref{prop:rs-pe-random-error}]  Note that $ \displaystyle v^{\pi}=\sum_{i,j} \omega^{\pi}_{i,j} \left(\phi_{i,j}^\top \theta+\epsilon_{i,j}\right)=\phi_{\pi}^\top \theta + \sum_{i,j} \omega^{\pi}_{i,j}\epsilon_{i,j}.$
One can further remark that $$\hat{\theta}-\theta=\frac{1}{T_2}\widehat{\Lambda}_{\tau}^{-1}\left(\sum_{t=T_1+1}^{T}\left(\xi_t+\epsilon_{i_t,j_t}\right)\phi_{i_t,j_t}-\tau \theta\right)$$ for $T_2:=T-T_1.$
Consequently, \begin{align} \left|v^{\pi}-\hat{v}_{\rspe}\right| &\leq \left|\phi_{\pi}^\top \left(\hat{\theta}-\theta\right)\right|+ \sum_{i,j} \left|\omega^{\pi}_{i,j}\epsilon_{i,j}\right| \nonumber \\
&\leq \left|\frac{1}{T_2}\sum_{t=T_1+1}^{T}\phi_{\pi}^\top \left(\xi_t+\epsilon_{i_t,j_t}\right)\widehat{\Lambda}_{\tau}^{-1}\phi_{i_t,j_t}\right|+ \frac{\tau}{T_2}\phi_{\pi}^\top \widehat{\Lambda}_{\tau}^{-1} \theta + \epsilon_{\max} \nonumber \\
&\leq   \left| \frac{1}{T_2}\sum_{t=T_1+1}^{T} \xi_t\phi_{\pi}^\top \widehat{\Lambda}_{\tau}^{-1}\phi_{i_t,j_t}\right| + \frac{\tau}{T_2}\phi_{\pi}^\top \widehat{\Lambda}_{\tau}^{-1} \theta \\ &+ \epsilon_{\max}\left(1+\left| \frac{1}{T_2}\sum_{t=T_1+1}^{T} \phi_{\pi}^\top \widehat{\Lambda}_{\tau}^{-1}\phi_{i_t,j_t}\right|\right). \nonumber \\
 \end{align}

\textbf{Error decomposition.} Let $\Delta:=\Lambda_{\phi}^{-1}-\widehat{\Lambda}_{\tau}^{-1}.$  We consider the following decomposition of the error: 
 $$\left|v^{\pi}-\hat{v}_{\rspe}\right| \leq L_1 + L_2 + L_3$$ where $$L_1=\frac{1}{T_2}\sum_{t=T_1+1}^{T_2} \xi_t\phi_{\pi}^\top \Lambda_{\phi}^{-1}\phi_{i_t,j_t},$$ $$L_2=\frac{1}{T_2}\sum_{t=T_1+1}^{T_2} \xi_t\phi_{\pi}^\top \Delta\phi_{i_t,j_t}+\frac{\tau}{T_2}\phi_{\pi}^\top \widehat{\Lambda}_{\tau}^{-1} \theta $$
 and $$L_3=\epsilon_{\max}\left(1+\left| \frac{1}{T_2}\sum_{t=T_1+1}^{T} \phi_{\pi}^\top \widehat{\Lambda}_{\tau}^{-1}\phi_{i_t,j_t}\right|\right).$$

As discussed before, we first perform the analysis conditionally on the first phase of the algorithm. We will then control the misspecification error induced by the first phase with Corollary \ref{corr:misspecification}.
Let $L_{\phi}:=\max_{i,j} \phi_{i,j}^\top \Lambda_{\phi}^{-1}\phi_{i,j}$, $L_{\theta}:=\max_{i,j} \vert \phi_{i,j}^\top \theta \vert$. Adapting the proof of Theorem \ref{thm:linear-error-bound} yields the following lemma.
\begin{lemma}\label{lemma:rs-pe-random-error}
    As long as $T_2 \geq \frac{18}{\omega_{\min}}\log\left(16d/\delta\right)$ and $\tau \leq \frac{r}{\left(m \wedge n\right)}\frac{\omega_{\min}}{\omega_{\max}}\log\left(16d/\delta\right)$, then with probability at least $1-\delta/2$, $$|L_1|+|L_2| \leq \frac{E_1}{\sqrt{T_2}}+\frac{E_2}{T_2}+\frac{E_3}{T_2^{3/2}}+\frac{E_4}{T_2^2}$$ and $$|L_3| \leq \epsilon_{\max}\left(1+\sqrt{L_{\phi}}\Vert \phi_{\pi}\Vert_{\Lambda_{\phi}^{-1}}\left(1+4\sqrt{\frac{\log\left(16d/\delta\right)}{T_2}}\right)\right)$$ with $E_1:=\Vert \phi_{\pi}\Vert_{\Lambda_{\phi}^{-1}}\sqrt{2\sigma^2\log(16/\delta)}, E_2 :=\Vert \phi_{\pi}\Vert_{\Lambda_{\phi}^{-1}}\left(L_{\phi} L_{\theta}+\sigma(1+4\sqrt{2d})\sqrt{L_{\phi}}\right)\log\left(32(d+1)/\delta\right), E_3:=\Vert \phi_{\pi}\Vert_{\Lambda_{\phi}^{-1}}\left(\frac{\sigma\sqrt{2}}{3}+4\sigma\sqrt{d}+4\sqrt{L_{\phi}}L_{\theta}\right)L_{\phi}\log^{3/2}\left(32(d+1)/\delta\right), E_4:=\Vert \phi_{\pi}\Vert_{\Lambda_{\phi}^{-1}}\frac{4\sigma\sqrt{2}}{3}\sqrt{d}L_{\phi}^{3/2}\log^{2}\left(32(d+1)/\delta\right).$
\end{lemma}

\begin{remark} 
To see why we must split the data for the two phases in the analysis, let $\mathcal{F}_t=\sigma\left(i_1,j_1,r_1,\dots,i_{T_1},j_{T_1},r_{T_1},\dots,i_{t},j_{t},r_{t},i_{t+1},j_{t+1}\right)$ for $t=T_1,\dots, T-1$. Because $\Lambda_{\phi}^{-1}$ is $\mathcal{F}_{T_1}$-measurable, $e_t:=\phi_{\pi}^\top \xi_t\Lambda_{\phi}^{-1}\phi_{i_{t},j_{t}}$ defines a martingale difference sequence with respect to $\left(\mathcal{F}_t\right)_{t=T_1,\dots,T-1}$ and the contextual linear bandit analysis applies. If all the samples were used to build $\phi$, $e_t$ would not define a martingale difference sequence, which means that we could not use our concentration results.
\end{remark}

\begin{proof}[Proof of Lemma \ref{lemma:rs-pe-random-error}]
 The control of these terms closely follows the proof of theorem \ref{thm:linear-error-bound}; in particular, we have $$ \begin{aligned} \left|L_2\right| 
&\leq \Vert \phi_{\pi}\Vert_{\Lambda_{\phi}^{-1}}\left(\Vert \Lambda_{\phi}^{1/2}\Delta \Lambda_{\phi}^{1/2}\Vert_{\op}\Vert \Lambda_{\phi}^{-1/2}W\Vert_2+\frac{\tau L_{\theta}}{\lambda_{\min}\left(\Lambda_{\phi}\right)T_2}\left(1+\Vert \Lambda^{1/2}\Delta \Lambda_{\phi}^{1/2}\Vert_{\op}\right)\right).\end{aligned} $$ for $W=\frac{1}{T_2}\sum_{t=T_1+1}^{T_2}\xi_t\phi_{i_t,j_t}.$ 
 The only difference is the additional term $L_3$, which can be controlled similarly to $L_2$. Indeed, by Cauchy-Schwartz and the definition of the operator norm, $$ \begin{aligned} \frac{1}{T_2}\sum_{t=T_1+1}^{T} \phi_{\pi}^\top \widehat{\Lambda}_{\tau}^{-1}\phi_{i_t,j_t}&=\frac{1}{T_2}\sum_{t=T_1+1}^{T} \phi_{\pi}^\top \Lambda_{\phi}^{-1}\phi_{i_t,j_t}+\frac{1}{T_2}\sum_{t=T_1+1}^{T} \phi_{\pi}^\top \left(\widehat{\Lambda}_{\tau}^{-1}-\Lambda_{\phi}^{-1}\right)\phi_{i_t,j_t} \\
&\leq \frac{1}{T_2}\sum_{t=T_1+1}^{T} \Vert \phi_{\pi}\Vert_{\Lambda_{\phi}^{-1}}\sqrt{\phi_{i_t,j_t}^\top \Lambda_{\phi}^{-1}\phi_{i_t,j_t}} \\ &+ \frac{1}{T_2}\sum_{t=T_1+1}^{T} \Vert \Lambda_{\phi}^{-1/2}\phi_{\pi}\Vert_{2}\Vert \Lambda_{\phi}^{1/2}\Delta \Lambda_{\phi}^{1/2}\Vert_{\op}\Vert \Lambda_{\phi}^{-1/2}\phi_{i_t,j_t}\Vert_{2} \\
&\leq \sqrt{L_{\phi}}\Vert \phi_{\pi}\Vert_{\Lambda_{\phi}^{-1}}\left(1+\Vert \Lambda_{\phi}^{1/2}\Delta \Lambda_{\phi}^{1/2}\Vert_{\op}\right).\end{aligned}$$
Thus, $L_3 \leq \epsilon_{\max}\left(1+\sqrt{L_{\phi}}\Vert \phi_{\pi}\Vert_{\Lambda_{\phi}^{-1}}\left(1+\Vert \Lambda_{\phi}^{1/2}\Delta \Lambda_{\phi}^{1/2}\Vert_{\op}\right)\right).$

By the same arguments as in the proof of Lemma B.6. from \cite{duan2020minimax}, if $T_2 \geq 18L_{\phi}\log\left(16d/\delta\right)$ and $\tau \leq L_{\phi}\log(16d/\delta)\lambda_{\min}\left(\Lambda_{\phi}\right)$, then $\Vert \Lambda_{\phi}^{1/2}\Delta \Lambda_{\phi}^{1/2}\Vert_{\op} \leq 4\sqrt{\frac{\log(16d/\delta)L_{\phi}}{T_2}}$ with probability larger than $1-\delta/8$. 

We can now derive high probability bounds on $L_1$, $L_2$ and $L_3$ by adapting the contextual linear bandit analysis. The main difference is that here, the analysis is performed conditionally on the randomness of the first phase of the algorithm.  One subtlety to note is that we need to use a version of Proposition \ref{prop:freedman-1} that is conditional on the first $T_1$ samples because the choices of $\gamma$ and $\varepsilon$ in the analysis depend on $\phi$, which depends on the estimator $\widehat{M}$ built with the first $T_1$ samples. It is not an issue as it directly implies an unconditional high probability bound by taking the expectation. \\ \\
Following the proof of Theorem \ref{thm:linear-error-bound} thus yields the high probability bound in the statement as soon as $T_2 \geq 18L_{\phi}\log\left(16d/\delta\right)$ and $\tau \leq L_{\phi}\log(16d/\delta)\lambda_{\min}\left(\Lambda_{\phi}\right)$. All that remains to be checked is that these inequalities are true as soon as $T_2 \geq \frac{18}{\omega_{\min}}\log\left(16d/\delta\right)$ and $\tau \leq \frac{r}{\left(m \wedge n\right)}\frac{\omega_{\min}}{\omega_{\max}}\log\left(16d/\delta\right)$. By Lemma \ref{lemma:cov-invertibility} and the fact that $\displaystyle \max_{i,j} \Vert \phi_{i,j}\Vert^2_2 \geq \max\left(\Vert U\Vert^2_{2 \to \infty},\Vert V\Vert^2_{2 \to \infty}\right)$, $$\begin{aligned}
  \frac{r}{\left(m \wedge n \right)\omega_{\max}} \leq \max_{i,j} \frac{\Vert \phi_{i,j}\Vert^2_2}{\lambda_{\max}\left(\Lambda_{\phi}\right)} \leq  L_{\phi} \leq \max_{i,j} \frac{\Vert \phi_{i,j}\Vert^2_2}{\lambda_{\min}\left(\Lambda_{\phi}\right)} \leq \frac{1}{\omega_{\min}},
\end{aligned}$$ which entails the desired result. 
\end{proof}

Note that in addition to $L_{\phi} \leq \frac{1}{\omega_{\min}}$, we also have $L_{\theta} \leq \max_{i,j} \Vert  \phi_{i,j}\Vert_{2} \Vert \theta\Vert_{2} \leq \sqrt{mn}\Vert M\Vert_{\max}$. This shows that the terms that depend on $E_2$, $E_3$ and $E_4$ are of lower order compared to the term that depends on $E_1$, and that they have a polynomial dependency in the model parameters (note that $\Vert \phi_{\pi}\Vert_{\Lambda_{\phi}^{-1}} \leq \sqrt{L_{\phi}}$). We will now see that we can obtain a sharper control of the dependency in the model parameters by leveraging high probability bounds on $L_{\phi}$ and $L_{\theta}$ instead of deterministic ones. These high probability bounds come freely with the control of the misspecification error.

\begin{lemma}\label{lemma:control-of-pe-constants} Let $\delta>0$. There exists an  universal constant $c>0$ such that  $T_1 \geq \frac{cL^2\kappa^2(m+n)}{\sigma_r^2\omega_{\min}}\log^3\left(\frac{(m+n)T_1}{\delta}\right)$ ensures  that with probability larger than $1-\delta$, the following inequalities happen at the same time :
\begin{itemize}
    \item $\epsilon_{\max} \lesssim \frac{ L^2 }{\sigma_r} \frac{\mu^2\kappa^3 r (m+n) }{T_1 \omega_{\min}(m\wedge n)} \log^3\left( \frac{(m+n)T_1}{\delta}\right)$
    \item $L_{\theta} \lesssim \frac{\sigma_1\mu^2r}{m \wedge n}$
    \item $L_{\phi} \lesssim \frac{\mu^2r\left(m+n\right)}{\omega_{\min}mn}$.
\end{itemize}
\end{lemma}

\begin{proof}[Proof of Lemma \ref{lemma:control-of-pe-constants}]

Since $T_1 \gtrsim \frac{L^2(m+n)}{\sigma_r^2\omega_{\min}}\log^3\left(\frac{(m+n)T_1}{\delta}\right)$, by Theorem \ref{thm:recovery-two-to-infinity-norm}, there is a universal constant $C$ such that the event $$A_{\delta}=\left\{\epsilon_{\textup{Sub-Rec}} \leq \frac{CL}{\sigma_r} \sqrt{\frac{\mu^2\kappa^2 r (m+n) }{T_1 \omega_{\min}(m\wedge n)} \log^3\left( \frac{(m+n)T_1}{\delta}\right)} \right\}$$ holds with probability at least $1-\delta.$ Under $A_{\delta}$, from \eqref{eq:eps_bound_pert_bound}, $$\epsilon_{\max} \leq \sigma_1 \epsilon^2_{\textup{Sub-Rec}} \lesssim \frac{ L^2 }{\sigma_r} \frac{\mu^2\kappa^3 r (m+n) }{T_1 \omega_{\min}(m\wedge n)} \log^3\left( \frac{(m+n)T_1}{\delta}\right).$$
Let us now prove that both of the remaining inequalities are true under $A_{\delta}.$ Specifically, we will see that $L_{\theta}$ (resp. $L_{\phi}$) can be deterministically bounded by a function of $\epsilon_{\max}$ (resp. $\epsilon_{\textup{Sub-Rec}}$). 

\textbf{Control of $L_{\theta}$.} Recall that $M_{i,j}=\phi_{i,j}^\top \theta + \epsilon_{i,j}.$ Consequently, $L_{\theta}=\displaystyle \max_{i,j} \left|\phi_{i,j}^\top \theta\right| \leq \Vert M\Vert_{\max} + \epsilon_{\max}$. Since $T_1 \gtrsim \frac{L^2\kappa^2(m+n)}{\sigma_r^2\omega_{\min}}\log^3\left(\frac{(m+n)T_1}{\delta}\right)$, under $A_{\delta}$ we have $\epsilon_{\max} \lesssim \sigma_1\frac{\mu^2r}{m \wedge n}$, so that by Lemma \ref{lem:spikiness}, \begin{align*}L_{\theta} &\leq \Vert M \Vert_{\max}+\epsilon_{\max} \\ &\lesssim \sigma_1\frac{\mu^2r}{\sqrt{mn}}+\sigma_1\frac{\mu^2r}{m \wedge n} \\&\lesssim \frac{\sigma_1\mu^2r}{m \wedge n }.\end{align*} 

\textbf{Control of $L_{\phi}$.} Recall that $L_{\phi} \leq \displaystyle \max_{i,j} \frac{\Vert \phi_{i,j}\Vert^2_2}{\lambda_{\min}\left(\Lambda_{\phi}\right)} \leq 
 \max_{i,j} \frac{\Vert \phi_{i,j}\Vert^2_2}{\omega_{\min}}.$
Note that by Lemma \ref{lemma:feature-map-bound}, $\displaystyle \max_{i,j} \Vert \phi_{i,j}\Vert^2_2 \leq  \mu^2r\frac{m+n-\mu^2r}{mn}+3\epsilon_{\textup{Sub-Rec}}$. Consequently,  $L_{\phi} \leq \frac{1}{\omega_{\min}}\left(\mu^2r\frac{m+n-\mu^2r}{mn}+3\epsilon_{\textup{Sub-Rec}}\right).$ Since $T_1 \gtrsim \frac{L^2\kappa^2(m+n)}{\sigma_r^2\omega_{\min}}\log^3\left(\frac{(m+n)T_1}{\delta}\right)$, under $A_{\delta}$, we have $\epsilon_{\textup{Sub-Rec}} \lesssim \frac{\mu\sqrt{r}}{\sqrt{m \wedge n}}, $ so 
 $$L_{\phi} \lesssim \frac{1}{\omega_{\min}}\left(\mu^2r\frac{m+n-\mu^2r}{mn}+\frac{\mu\sqrt{r}}{\sqrt{m \wedge n}}\right) \lesssim \frac{\mu^2r\left(m+n\right)}{\omega_{\min}mn}.$$ 
    
\end{proof}

\textbf{Choice of $T_1$.} Note that the high probability bounds of Lemma \ref{lemma:rs-pe-random-error} and \ref{lemma:control-of-pe-constants} together suggest that as long as $T_1=O\left(T^{\beta}\right)$ for $1/2<\beta\leq 1$, the OPE error that originates from the misspecification will be of higher order than the error induced by the least squares estimation of $\theta$ with high probability. Furthermore, to ensure that the least squares error term and the misspecification error term both have the best achievable scaling in $T$, we should have $T_1=\Theta(T)$ and $T_2=\Theta(T).$  Specifically, for some $\alpha \in \left(0,1\right),$ we can choose $T_1=\floor{\alpha T}$ and $T_2=\ceil{\left(1-\alpha \right) T}$. 

We can now finish the proof of  Proposition \ref{prop:rs-pe-random-error} by combining the error bounds of Lemma \ref{lemma:rs-pe-random-error} and \ref{lemma:control-of-pe-constants} with a union bound.

Note that the high probability bounds of Lemma \ref{lemma:feature-map-bound} and \ref{lemma:control-of-pe-constants} hold as soon as
\begin{align*}
\alpha T &\geq 1+\frac{cL^2\kappa^2(m+n)}{\sigma_r^2\omega_{\min}}\log^3\left(\frac{(m+n)\alpha T}{\delta}\right),\\
\left(1-\alpha\right)T &\geq \frac{18}{\omega_{\min}}\log\left(16d/\delta\right)\\
\text{ and }\tau &\leq \frac{r}{m \wedge n}\frac{\omega_{\min}}{\omega_{\max}}\log\left(16d/\delta\right)
\end{align*}
for some universal constant $c>0$. Recall that we denoted by $A_{\delta}$ the event under which the bounds of Lemma \ref{lemma:control-of-pe-constants} holds. By Lemma \ref{lemma:rs-pe-random-error}, under $A_{\delta}$, with probability larger than $1-\delta/2$, we have $$\left|L_1\right|+\left|L_2\right|\leq \Vert \phi_{\pi}\Vert_{\Lambda_{\phi}^{-1}}\sqrt{\frac{2\sigma^2\log(16/\delta)}{\left(1-\alpha\right)T}}+K_0\frac{\log^2\left(e/\delta\right)}{T},$$ $$|L_3| \leq K_0\frac{\log^3\left(\left(m+n\right)T/\delta\right)}{T}$$ for $K_0=\operatorname{poly}\left(\Vert M\Vert_{\max},\sigma,\mu,\kappa,r,m,n,(\omega_{\min}mn)^{-1}\right)$ (we have used that  $\frac{1}{\sqrt{T_2}} \leq \frac{1}{\sqrt{\left(1-\alpha\right)T}}$). A union bound concludes the proof of the first part of Proposition \ref{prop:rs-pe-random-error}. \\ \\
Let us now upper bound the polynomial term $K_0$. Direct computations ensure that when the high probability bounds of Lemma \ref{lemma:rs-pe-random-error} and \ref{lemma:control-of-pe-constants} hold, 
\begin{align*}\left|L_3\right| &\lesssim \epsilon_{\max}L_{\phi} \lesssim \frac{L^2}{\sigma_r}\frac{\mu^4\kappa^3r^2\left(m+n\right)^2\log^3\left(\left(m+n\right)T/\delta\right)}{T\omega_{\min}^2mn\left(m \wedge n \right)} \\
&\lesssim \frac{L^2}{\Vert M \Vert_{\max}}\frac{\mu^6\kappa^4r^3\left(m+n\right)^2\sqrt{mn}\log^3\left(\left(m+n\right)T/\delta\right)}{T\left(\omega_{\min}mn\right)^2\left(m \wedge n \right)} ,\end{align*}  $$\frac{E_2}{T} \lesssim \frac{\sqrt{L_{\phi}}\left(L_{\phi}L_{\theta}+\sigma\sqrt{dL_{\phi}}\right)\log(d/\delta)}{T} \lesssim \frac{\left(\Vert M \Vert_{\max}+\sigma\right)\mu^5r^{5/2}\left(m+n\right)^{3/2}\sqrt{mn}\log\left(d/\delta\right)}{T\left(\omega_{\min}mn\right)^{3/2}\left(m \wedge n\right)},$$ $$\begin{aligned} \frac{E_3}{T^{3/2}} &\lesssim \frac{L_{\phi}\left(\sigma\sqrt{d}+\sqrt{L_{\phi}}L_{\theta}\right)\log(d/\delta)}{T} \lesssim \frac{\left(\Vert M \Vert_{\max}+\sigma\right)\mu^5r^{5/2}\left(m+n\right)^{3/2}\sqrt{mn}\log\left(d/\delta\right)}{T\left(\omega_{\min}mn\right)^{3/2}\left(m \wedge n\right)}, \end{aligned} $$ $$\begin{aligned} \text{ and }\frac{E_4}{T^2} &\lesssim \frac{\sigma\sqrt{d}L_{\phi}\log\left(d/\delta\right)}{T} \lesssim \frac{\sigma\mu^2r^{3/2}\left(m+n\right)^{3/2}\log\left(d/\delta\right)}{T\omega_{\min}mn} \end{aligned} $$

where we have used Lemma \ref{lem:spikiness} to bound $\sigma_1$ and $\sigma_r$, and $T \gtrsim L_{\phi}\log(d/\delta)$ (see the end of the proof of Lemma \ref{lemma:rs-pe-random-error}) in the last inequalities.
Since $\frac{4L^2}{\Vert M \Vert_{\max}} \geq \frac{\left(\Vert M \Vert_{\max}+ \sigma\right)^2}{\Vert M \Vert_{\max}} \geq \Vert M \Vert_{\max}+\sigma \geq \sigma$, the first error term dominates all the others, and we can take  $$K_0:=C'\frac{L^2}{\Vert M \Vert_{\max}}\frac{\mu^6\kappa^{4}r^{3}\left(m+n\right)^{2}\sqrt{mn}}{\left(\omega_{\min}mn\right)^2\left(m \wedge n\right)}$$ for some universal constant $C'>0$.
\end{proof}

\begin{remark}\label{rem:instance-optimal-bound} To derive an instance-dependent upper bound from Proposition \ref{prop:rs-pe-random-error}, one may be tempted to prove that with high probability, 
$$\Vert \phi_{\pi}\Vert_{\Lambda_{\phi}^{-1}}=\Vert \psi_{\pi}\Vert_{\Lambda^{-1}}+o\left(1\right).$$ This would yield an error bound whose dominant term scales with $\Vert \psi_{\pi} \Vert_{\Lambda}^{-1}\sqrt{\frac{\log\left(1/\delta\right)}{T}}$ and would echo the linear bandit upper bound of Theorem \ref{thm:linear-error-bound}, that we have demonstrated is instance-optimal. Although numerical experiments suggest this result holds, we were unsuccessful in our proof attempts. Notably, one could write $ \Vert \phi_{\pi}\Vert_{\Lambda_{\phi}^{-1}}^2 =  \Vert \psi_{\pi}\Vert_{\Lambda^{-1}}^2  + (\Vert \phi_{\pi}\Vert_{\Lambda_{\phi}^{-1}}^2-\Vert \phi_{\pi}\Vert_{\Lambda^{-1}}^2)+( \Vert \phi_{\pi}\Vert_{\Lambda^{-1}}^2- \Vert \psi_{\pi}\Vert_{\Lambda^{-1}}^2) $ and attempt to prove that the two rightmost terms converge to $0$ with high probability. However, numerical experiments suggest that these terms do not both vanish in general, but rather cancel out, and thus a more careful analysis is required.
\end{remark}

By leveraging Lemma \ref{lemma:feature-map-bound}, we can however obtain the following instance-dependent bound, with a dependency on $\omega_{\min}$ instead of the entire behavior policy. Theorem \ref{thm:rs-pe-error} will follow from this result.

\begin{theorem}\label{thm:rs-pe-error-full}
Under the assumptions of Proposition \ref{prop:rs-pe-random-error}, with probability larger than $1-\delta,$ 
    $$\begin{aligned} \left|v^{\pi}-\hat{v}_{\rspe}\right| &\leq \frac{\Vert \psi_{\pi}\Vert_{2}}{\sqrt{\omega_{\min}}}\sqrt{\frac{2\sigma^2\log(16/\delta)}{\left(1-\alpha\right)T}} +K_0\left(\frac{\log^3\left(\left(m+n\right)T/\delta\right)}{T}+\frac{\log^{5/4}\left(\left(m+n\right)T/\delta\right)}{T^{3/4}}\right), \end{aligned}$$
    where \begin{align*} K_0 \lesssim \frac{L^2}{\Vert M \Vert_{\max}}\frac{\mu^{6}\kappa^{4}r^{3}\left(m+n\right)^{2}\sqrt{mn}}{\left(\omega_{\min}mn\right)^2\left(m \wedge n\right)}.\end{align*}
\end{theorem} 

\begin{proof}[Proof of Theorem \ref{thm:rs-pe-error-full}]
Note that by Lemmas \ref{lemma:cov-invertibility} and \ref{lemma:feature-map-bound}, $$\Vert \phi_{\pi}\Vert^2_{\Lambda_{\phi}^{-1}} \leq \frac{\Vert \phi_{\pi}\Vert^2_2}{\omega_{\min}} \leq \frac{1}{\omega_{\min}}\left(\Vert \psi_{\pi}\Vert^2_2+3\epsilon_{\textup{Sub-Rec}}\right).$$ 
Under the event $A_{\delta}$ defined in the proof of Lemma \ref{lemma:control-of-pe-constants}, we consequently also have (using $\frac{1}{\sigma_r} \leq \frac{\mu^2\kappa r}{\sqrt{mn}\Vert M \Vert_{\max}}$) \begin{align}\label{eq:bound_phi_to_psi}\Vert \psi_{\pi}\Vert^2_{\Lambda_{\phi}^{-1}} \leq \frac{1}{\omega_{\min}}\left(\Vert \psi_{\pi}\Vert^2_2+\frac{C}{\sqrt{T}}\frac{L}{\Vert M \Vert_{\max}}\mu^3\kappa^2r^{3/2}\sqrt{\frac{(m+n)}{\omega_{\min}mn\left(m \wedge n\right)}}\log^{3/2}\left(\frac{\left(m+n\right)T}{\delta}\right)\right)\end{align} for some universal constant $C>0.$  Since $\sqrt{x+y} \leq \sqrt{x}+\sqrt{y}$ for every $x,y \geq 0$,  this implies \begin{align}\label{eq:bound-weighted-norm} \Vert \phi_{\pi}\Vert_{\Lambda_{\phi}^{-1}} \leq \frac{\Vert \psi_{\pi}\Vert_2}{\sqrt{\omega_{\min}}}+C^{1/2}\sqrt{\frac{L}{\Vert M \Vert_{\max}}}\mu^{3/2}\kappa r^{3/4}\omega_{\min}^{-3/4}\left(\frac{m+n}{mn\left(m \wedge n\right)}\right)^{1/4}\frac{\log^{3/4}\left(\left(m+n\right)T/\delta\right)}{T^{1/4}}.\end{align} 
Combining this result with Proposition \ref{prop:rs-pe-random-error}\footnote{Note that there is no need for a union bound as the high probability bound \eqref{eq:bound-weighted-norm} is true under $A_{\delta}$.} and noting that $\sigma\sqrt{\frac{L}{\Vert M \Vert_{\max}}} \leq \frac{\sigma L}{\Vert M \Vert_{\max}} \leq \frac{L^2}{\Vert M \Vert_{\max}}:=A$, we get that with probability larger than $1-\delta,$
$$\begin{aligned} \left|v^{\pi}-\hat{v}_{\rspe}\right| &\leq \frac{\Vert \psi_{\pi}\Vert_{2}}{\sqrt{\omega_{\min}}}\sqrt{\frac{2\sigma^2\log(16/\delta)}{\left(1-\alpha\right)T}} +C_1 A\sqrt{\log(e/\delta)}\mu^{3/2}\kappa r^{3/4}\omega_{\min}^{-3/4}\left(\frac{m+n}{mn\left(m \wedge n\right)}\right)^{1/4}\frac{\log^{3/4}\left(\left(m+n\right)T/\delta\right)}{T^{3/4}}\\ &+C_2A\mu^{6}\kappa^{4}r^{3}\frac{\left(m+n\right)^{2}\sqrt{mn}}{\left(\omega_{\min}mn\right)^2\left(m \wedge n\right)}\frac{\log^3\left(\left(m+n\right)T/\delta\right)}{T}\end{aligned}$$ for some universal constants $C_1,C_2>0$.
 Since $\sqrt{\log(e/\delta)}\mu^{3/2}\kappa r^{3/4}\omega_{\min}^{-3/4}\left(\frac{m+n}{mn\left(m \wedge n\right)}\right)^{1/4} \log^{3/4}\left(\left(m+n\right)T/\delta\right)\lesssim \mu^{6}\kappa^{4}r^{3}\frac{\left(m+n\right)^{2}\sqrt{mn}}{\left(\omega_{\min}mn\right)^2\left(m \wedge n\right)}\log^{5/4}\left(\left(m+n\right)T/\delta\right)$, this concludes the proof. 

 \end{proof}
 \begin{remark}\label{remark:homogeneous-scaling}
When $M$ and $\omega$ are homogeneous, $K_0 \lesssim \frac{(\Vert M \Vert_{\max} \vee \sigma)^2}{\Vert M \Vert_{\max}}mn$, and if we additionally have $\Vert M \Vert_{\max}=\Theta(\sigma),$ $K_0 \lesssim \sigma mn.$
\end{remark} 

When both the context distribution and the behavior policy are uniform, the upper bound of Theorem \ref{thm:rs-pe-error-full} corresponds to the conjectured instance-dependent upper bound mentioned in Remark \ref{rem:instance-optimal-bound}. Indeed, Lemma \ref{lemma:cov-invertibility} ensures that in this case, $\Lambda=\frac{1}{mn}I_d,$ and thus
$\Vert \psi_{\pi}\Vert_{\Lambda^{-1}}= \sqrt{mn}\Vert \psi_{\pi}\Vert_{2}=\frac{\Vert \psi_{\pi}\Vert_{2}}{\sqrt{\omega_{\min}}}.$ 

We now explain how to reformulate the statement of Theorem \ref{thm:rs-pe-error-full} as a sample complexity result and retrieve the statement of Theorem \ref{thm:rs-pe-error}.

\begin{proof}[Proof of Theorem \ref{thm:rs-pe-error}]\label{proof:error-to-sample-complexity}

Let us assume that we are under the assumptions of Proposition \ref{prop:rs-pe-random-error}. By Theorem \ref{thm:rs-pe-error-full}, $$\left|v^{\pi}-\hat{v}_{\rspe}\right| \leq \frac{\Vert \psi_{\pi}\Vert_{2}}{\sqrt{\omega_{\min}}}\sqrt{\frac{2\sigma^2\log(16/\delta)}{\left(1-\alpha\right)T}} +K_0\frac{\log^3\left(\left(m+n\right)T/\delta\right)}{T}+K_0\frac{\log^{5/4}\left(\left(m+n\right)T/\delta\right)}{T^{3/4}}$$ with probability larger than $1-\delta.$ In particular, $\left|v^{\pi}-\hat{v}_{\rspe}\right| \leq \varepsilon$ with probability larger than $1-\delta$ as soon as
\begin{align*}
    \frac{\varepsilon}{3} \ge \frac{\Vert \psi_{\pi}\Vert_{2}}{\sqrt{\omega_{\min}}}\sqrt{\frac{2\sigma^2\log(16/\delta)}{\left(1-\alpha\right)T}} \text{,  }\quad  \frac{\varepsilon}{3} \ge 
K_0\frac{\log^3\left(\left(m+n\right)T/\delta\right)}{T},\quad \text{and}\quad
    \frac{\varepsilon}{3} \ge K_0\frac{\log^{5/4}\left(\left(m+n\right)T/\delta\right)}{T^{3/4}}. 
\end{align*}

It is straightforward to check that \begin{align}\label{instance-specific-complexity}
    T \gtrsim \frac{\sigma^2\Vert \psi_\pi \Vert^2_2}{\omega_{\min}\,\varepsilon^2}\log\left(\frac{e}{\delta}\right) + \frac{K_0^{4/3}}{\varepsilon^{4/3}}\log^{3}\left( \frac{(m+n)T}{\delta}\right)\!
\end{align} is enough to satisfy all three inequalities, where we recall that $K_0$ has been defined in \eqref{eq:def_of_K0}. 

    Furthermore, the conditions on the number of samples of Theorem \ref{thm:rs-pe-error-full}, namely $T \gtrsim \frac{L^2\kappa^2(m+n)}{\sigma_r^2\omega_{\min}}\log^3\left(\frac{(m+n)T}{\delta}\right)$ and $T \gtrsim \frac{1}{\omega_{\min}}\log\left(16d/\delta\right)$, are both implied by \eqref{instance-specific-complexity} when $\varepsilon \leq \Vert M \Vert_{\max}$ because this choice implies $T \gtrsim \frac{K_0}{\Vert M \Vert_{\max}}\log^{3}\left( \frac{(m+n)T}{\delta}\right) $ (to see that the second condition is verified, note that  $\frac{K_0}{\Vert M \Vert_{\max}} \gtrsim \frac{(m+n)^2}{\omega_{\min}\sqrt{mn}\left(m \wedge n\right)} \gtrsim \frac{1}{\omega_{\min}}$).

The statement of the theorem thus holds with the multiplicative factor
\begin{align}\label{eq:def_of_K1}K_1:=K_0^{4/3} \lesssim \left(\frac{L^2}{\Vert M \Vert_{\max}}\frac{\mu^{6}\kappa^{4}r^{3}\left(m+n\right)^{2}\sqrt{mn}}{\left(\omega_{\min}mn\right)^2\left(m \wedge n\right)}\right)^{4/3}.\end{align}
We finally note that by Remark \ref{remark:homogeneous-scaling}, when $M$ and $\omega$ are homogeneous and $\Vert M \Vert_{\max} =\Theta(\sigma)$, $K_1 \lesssim \left(\sigma mn\right)^{4/3}.$
\end{proof}

\textbf{Minimax sample complexity guarantee.}
Let us note that by \eqref{eq:phi-pi-norm-bound},  $\Vert \psi_{\pi} \Vert^2_{2} \leq \mu^2r\frac{m+n-\mu^2r}{mn}$, which implies that $\vert v^{\pi}-\hat{v}^\pi_{\rspe} \vert \leq \varepsilon$ with probability larger than $1-\delta$ as soon as $$T \gtrsim \frac{\sigma^2\mu^2r\left(m+n-\mu^2r\right)}{\omega_{\min}mn\,\varepsilon^2}\log\left(\frac{e}{\delta}\right) + \frac{K_1}{\varepsilon^{4/3}}\log^{3}\left( \frac{(m+n)T}{\delta}\right).$$
By adapting the previous analysis with a more refined argument, the dependency in $\varepsilon$ and the model parameters of the second term can be improved. This is what we demonstrate now.

\begin{theorem}\label{thm:rs-pe-error-minimax}
Let $\varepsilon, \delta \in (0,1)$. With the choices $T_1 = \lfloor \alpha T \rfloor$ and $\tau \le r(m\wedge n)^{-1} (\omega_{\min}/\omega_{\max}) \log(16d/\delta)$,
$\vert v^{\pi}-\hat{v}^\pi_{\rspe} \vert \leq \varepsilon$ with probability larger than $1-\delta$ as soon as $$T \gtrsim \frac{\sigma^2\mu^2r\left(m+n\right)}{\omega_{\min}mn\,\varepsilon^2}\log\left(\frac{e}{\delta}\right) + \frac{K_0}{\varepsilon}\log^{3}\left( \frac{(m+n)T}{\delta}\right)$$

where \begin{align} K_0 \lesssim \frac{L^2}{\Vert M \Vert_{\max}}\frac{\mu^{6}\kappa^{4}r^{3}\left(m+n\right)^{2}\sqrt{mn}}{\left(\omega_{\min}mn\right)^2\left(m \wedge n\right)}.\end{align}
\end{theorem}

\begin{remark}
When $M$ and $\omega$ are homogeneous and $\Vert M \Vert_{\max}=\Theta(\sigma)$, $K_0 \lesssim \sigma mn.$
\end{remark}
\begin{proof}[Proof of Theorem \ref{thm:rs-pe-error-minimax}] 
Combining \eqref{eq:bound_phi_to_psi} with the bound on $\Vert \psi_{\pi}\Vert^2_2$ entails
\begin{align}\label{eq:intermediate-minimax}\Vert \psi_{\pi}\Vert^2_{\Lambda_{\phi}^{-1}} &\leq \frac{1}{\omega_{\min}}\left(\frac{\mu^2r\left(m+n\right)}{mn}+\frac{C}{\sqrt{T}}\frac{L}{\Vert M \Vert_{\max}}\mu^3\kappa^2r^{3/2}\sqrt{\frac{m+n}{mn\left(m \wedge n\right)\omega_{\min}}}\log^{3/2}\left(\frac{\left(m+n\right)T}{\delta}\right)\right) \nonumber \\
&= \frac{\mu^2r\left(m+n\right)}{\omega_{\min}mn}\left(1+\frac{C}{\sqrt{T}}\frac{L}{\Vert M \Vert_{\max}}\mu\kappa^2r^{1/2}\sqrt{\frac{mn}{\left(m+n\right)\left(m \wedge n\right)\omega_{\min}}}\log^{3/2}\left(\frac{\left(m+n\right)T}{\delta}\right)\right)\end{align} for some universal constant $C>0$.
Now, using the inequality $\sqrt{1+x} \leq 1+\frac{x}{2}$ (instead of $\sqrt{x+y} \leq \sqrt{x} + \sqrt{y}$ which was used previously), \eqref{eq:intermediate-minimax} implies 
\begin{align*} \Vert \psi_{\pi}\Vert_{\Lambda_{\phi}^{-1}} \leq \mu\sqrt{\frac{r\left(m+n\right)}{\omega_{\min}mn}}\left(1+\frac{C}{2\sqrt{T}}\frac{L}{\Vert M \Vert_{\max}}\mu\kappa^2r^{1/2}\sqrt{\frac{mn}{\left(m+n\right)\left(m \wedge n\right)\omega_{\min}}}\log^{3/2}\left(\frac{\left(m+n\right)T}{\delta}\right)\right)\end{align*}

Combining this result with Proposition \ref{prop:rs-pe-random-error} and noting that $\frac{\sigma L}{\Vert M \Vert_{\max}} \leq \frac{L^2}{\Vert M \Vert_{\max}}=A,$ we get that with probability larger than $1-\delta,$
$$\begin{aligned} \left|v^{\pi}-\hat{v}_{\rspe}\right| &\leq \mu\sqrt{\frac{r\left(m+n\right)}{\omega_{\min}mn}}\sqrt{\frac{2\sigma^2\log(16/\delta)}{\left(1-\alpha\right)T}} +C_1A\sqrt{\log(e/\delta)}\mu^{2}\kappa^2r\omega_{\min}^{-1}\left(m \wedge n\right)^{-1/2}\frac{\log^{3/2}\left(\left(m+n\right)T/\delta\right)}{T}\\ &+C_2A\mu^{6}\kappa^{4}r^{3}\frac{\left(m+n\right)^{2}\sqrt{mn}}{\left(\omega_{\min}mn\right)^2\left(m \wedge n\right)}\frac{\log^3\left(\left(m+n\right)T/\delta\right)}{T}\end{aligned}$$ for some universal constants $C_1,C_2>0$.
 Since $\sqrt{\log(e/\delta)}\mu^{2}\kappa^2r\omega_{\min}^{-1}\left(m \wedge n\right)^{-1/2}\log^{3/2}\left(\left(m+n\right)T/\delta\right) \lesssim \mu^{6}\kappa^{4}r^{3}\frac{\left(m+n\right)^{2}\sqrt{mn}}{\left(\omega_{\min}mn\right)^2\left(m \wedge n\right)}\log^{3}\left(\left(m+n\right)T/\delta\right)$, we obtain that under the assumptions of Proposition \ref{prop:rs-pe-random-error},
 \begin{align}\label{eq:sharp-minimax-bound}
 \left|v^{\pi}-\hat{v}_{\rspe}\right| \leq \mu\sqrt{\frac{r\left(m+n\right)}{\omega_{\min}mn}}\sqrt{\frac{2\sigma^2\log(16/\delta)}{\left(1-\alpha\right)T}} +K_0\frac{\log^3\left(\left(m+n\right)T/\delta\right)}{T}\end{align} with probability larger than $1-\delta,$ and transforming this result into a sample complexity guarantee as in the proof of Theorem \ref{thm:rs-pe-error} yields that
$\vert v^{\pi}-\hat{v}^\pi_{\rspe} \vert \leq \varepsilon$ with probability larger than $1-\delta$ as soon as $$T \gtrsim \frac{\sigma^2\mu^2r\left(m+n\right)}{\omega_{\min}mn\,\varepsilon^2}\log\left(\frac{e}{\delta}\right) + \frac{K_0}{\varepsilon}\log^{3}\left( \frac{(m+n)T}{\delta}\right).$$

\end{proof}
The above sample complexity guarantee can be reformulated as \begin{align*}T 
&\gtrsim \frac{\sigma^2\mu^2r\left(m+n\right)}{\omega_{\min}mn\varepsilon^2}\log\left(\frac{e}{\delta}\right) \end{align*} when $\frac{\sigma^2\mu^2r\left(m+n\right)}{\omega_{\min}mn\,\varepsilon^2}\log\left(\frac{e}{\delta}\right) \gtrsim \frac{K_0}{\varepsilon}\log^{3}\left( \frac{(m+n)T}{\delta}\right)$,
i.e., \begin{align}\label{cond-on-epsilon} \varepsilon \lesssim  \frac{\sigma^2 \Vert M \Vert_{\max}}{L^2}\frac{\omega_{\min}mn\left(m \wedge n\right)}{\mu^4\kappa^2r^2\left(m+n\right)\sqrt{mn}}\frac{\log\left(e/\delta\right)}{\log^3\left(\left(m+n\right)T/\delta\right)}.\end{align}

When $M$ and $\omega$ are homogeneous and $\sigma \gtrsim \Vert M \Vert_{\max}$, this simplifies to \begin{align*}T 
&\gtrsim \frac{\sigma^2\left(m+n\right)}{\varepsilon^2}\log\left(\frac{e}{\delta}\right) \end{align*} when \begin{align} \varepsilon \lesssim  \frac{\Vert M \Vert_{\max}}{m+n}\frac{\log\left(e/\delta\right)}{\log^3\left(\left(m+n\right)T/\delta\right)}.\end{align}

\subsection{Additional Lemmas}\label{subsec:useful-lemmas}

The following lemma entails that $\Lambda$ and $\Lambda_{\phi}$ are always invertible under our assumption $\omega_{\min}>0$.
\begin{lemma}\label{lemma:cov-invertibility}
 $\sum_{i,j} \psi_{i,j}\psi_{i,j}^\top =\sum_{i,j} \phi_{i,j}\phi_{i,j}^\top =I_{d}$. Consequently, $$\omega_{\max} \geq \lambda_{\max}\left(\Lambda\right) \geq \lambda_{\min}\left(\Lambda\right)\geq \omega_{\min}$$ and $$\omega_{\max} \geq \lambda_{\max}\left(\Lambda_{\phi}\right) \geq 
\lambda_{\min}\left(\Lambda_{\phi}\right) \geq \omega_{\min}.$$ 

\end{lemma}

\begin{proof}[Proof of Lemma \ref{lemma:cov-invertibility}]
    A direct computation shows that the diagonal entries of $\sum_{i,j} \psi_{i,j}\psi_{i,j}^\top$ can all be written as $\sum_{i,j}U^2_{i,k}V^2_{j,l}$, $\sum_{i,j}U^2_{\perp, i,k}V_{j,l}^2$ or $\sum_{i,j}U^2_{i,k}V_{\perp,j,l}^2$ for some $k$ and $l$. All of these terms are equal to $1$ by the orthogonality of  $ \begin{bmatrix}
U & U_{\perp} \\
\end{bmatrix}$ and $ \begin{bmatrix}
V & V_{\perp} \\
\end{bmatrix}$. The entries outside of the diagonal can all be factorized by a scalar product between two columns of $ \begin{bmatrix}
U & U_{\perp} \\
\end{bmatrix}$ or two columns of $ \begin{bmatrix}
V & V_{\perp} \\
\end{bmatrix}$, which is equal to 0 by the same orthogonality argument. The same reasoning holds for $\sum_{i,j} \phi_{i,j}\phi_{i,j}^\top$, as $ \begin{bmatrix}
\widehat{U} & \widehat{U}_{\perp} \\
\end{bmatrix}$ and $ \begin{bmatrix}
\widehat{V} & \widehat{V}_{\perp} \\
\end{bmatrix}$ are also orthogonal. The other part of the statement directly follows from the identities  $ \omega_{\max}\sum_{i,j} \psi_{i,j}\psi_{i,j}^\top  \succeq \Lambda \succeq \omega_{\min}\sum_{i,j} \psi_{i,j}\psi_{i,j}^\top $  and $ \omega_{\max}\sum_{i,j} \phi_{i,j}\phi_{i,j}^\top  
 \succeq \Lambda_{\phi} \succeq \omega_{\min}\sum_{i,j} \phi_{i,j}\phi_{i,j}^\top $.
\end{proof} 

\begin{lemma}\label{lemma:feature-map-bound}
    Let $\nu$ be a distribution on $[m] \times [n]$ so that $\sum_{i,j} \nu_{i,j}=1$, and let $\psi_{\nu}=\sum_{i,j} \nu_{i,j} \psi_{i,j}$, $\phi_{\nu}=\sum_{i,j} \nu_{i,j} \phi_{i,j}.$ Then
    \begin{itemize}
        \item $\left|\Vert \phi_{\nu}\Vert^2_2-\Vert \psi_{\nu}\Vert^2_2\right| \leq 3\epsilon_{\textup{Sub-Rec}}$
        \item $\displaystyle \max_{i,j} \Vert \psi_{i,j}\Vert^2_2 \leq \mu^2r\frac{m+n-\mu^2r}{mn}$.
    \end{itemize}
    \end{lemma}

\begin{proof}[Proof of Lemma \ref{lemma:feature-map-bound}]

We have $\displaystyle \Vert \phi_{\nu}\Vert^2_2=\sum_{(i,j),(i',j')} \nu_{i,j}\nu_{i',j'}\phi_{i,j}^\top\phi_{i',j'}$, so it is enough to show that $\phi_{i,j}^\top\phi_{i',j'}$ approaches $\psi_{i,j}^\top\psi_{i',j'}$ for every $(i,j),(i',j')$. A direct computation yields $$\phi_{i,j}^\top\phi_{i',j'}=\left(\widehat{U}_i^\top\widehat{U}_{i'}\right)\left(\widehat{V}_j^\top\widehat{V}_{j'}\right)+\left(\widehat{U}_{\perp,i}^\top\widehat{U}_{\perp,i'}\right)\left(\widehat{V}_j^\top\widehat{V}_{j'}\right)+\left(\widehat{U}_i^\top\widehat{U}_{i'}\right)\left(\widehat{V}_{\perp,j}^\top\widehat{V}_{\perp,j'}\right) $$ and $$\psi_{i,j}^\top\psi_{i',j'}=\left(U_i^\top U_{i'}\right)\left(V_j^\top V_{j'}\right)+\left(U_{\perp,i}^\top U_{\perp,i'}\right)\left(V_j^\top V_{j'}\right)+\left(U_i^\top U_{i'}\right)\left(V_{\perp,j}^\top V_{\perp,j'}\right).$$
By orthogonality arguments, $\widehat{U}_i^\top \widehat{U}_{i'}+\widehat{U}_{\perp,i}^\top \widehat{U}_{\perp,i'}=U_i^\top U_{i'}+U_{\perp,i}^\top U_{\perp,i'}=0$ when $i \neq i'$, and $\widehat{V}_j^\top \widehat{V}_{j'}+\widehat{V}_{\perp,j}^\top \widehat{V}_{\perp,j'}=V_j^\top V_{j'}+V_{\perp,j}^\top V_{\perp,j'}=0$ when $j \neq j'$. Thus, two out of the three terms in the equalities above cancel out when $(i,j) \neq (i',j')$. When $(i,j)=(i',j')$, we can combine two terms by using the fact that $\Vert U_i\Vert^2_2+\Vert U_{\perp,i}\Vert^2_2=1$ to write $\psi_{i,j}^\top \psi_{i,j} = \Vert V_j\Vert^2_2+\Vert U_i\Vert^2_2\Vert V_{\perp,j}\Vert^2_2.$ In any case, each term in the difference can be controlled by the subspace recovery error. For example,
$$ \begin{aligned} \left(\widehat{U}_i^\top \widehat{U}_{i'}\right)\left(\widehat{V}_{\perp,j}^\top \widehat{V}_{\perp,j'}\right)-\left(U_i^\top U_i'\right)\left(V_{\perp,j}^\top V_{\perp,j'}\right)&=\left(\widehat{V}_{\perp,j}^\top \widehat{V}_{\perp,j'}\right)\left(\widehat{U}_i^\top \widehat{U}_{i'}-U_i^\top U_i'\right)+\left(U_i^\top U_i'\right)\left(\widehat{V}_{\perp,j}^\top \widehat{V}_{\perp,j'}-V_{\perp,j}^\top V_{\perp,j'}\right) \\
&\leq \Vert \widehat{V}_{\perp,j}\Vert_2\Vert \widehat{V}_{\perp,j'}\Vert_2\Vert \widehat{U}\widehat{U}^\top -UU^\top \Vert_{2 \to \infty} \\ &+ \Vert U_i\Vert_2\Vert U_{i'}\Vert_2\Vert \widehat{V}_{\perp}\widehat{V}_{\perp}^\top -V_{\perp}V_{\perp}^\top \Vert_{2 \to \infty} \\
&\leq 2\epsilon_{\textup{Sub-Rec}}.
\end{aligned} $$
Consequently, $$\phi_{i,j}^\top \phi_{i',j'}-\psi_{i,j}^\top \psi_{i',j'} \leq \left(2+\indicator_{(i,j)=(i',j')}\right)\epsilon_{\textup{Sub-Rec}}. $$
In the end, $$\begin{aligned} \left|\Vert \phi_{\nu}\Vert^2_2-\Vert \psi_{\nu}\Vert^2_2\right|&=\left|\sum_{(i,j),(i',j')} \nu_{i,j}\nu_{i',j'}\left(\phi_{i,j}^\top \phi_{i',j'}-\psi_{i,j}^\top \psi_{i',j'}\right)\right| \\
&\leq \left(\sum_{(i,j),(i',j')} \nu_{i,j}\nu_{i',j'}\left(2+\indicator_{(i,j)=(i',j')}\right)\right)\epsilon_{\textup{Sub-Rec}}  \\
&= \left(2+\sum_{i,j}\nu_{i,j}^2\right)\epsilon_{\textup{Sub-Rec}},\end{aligned} $$
which is slightly sharper than the desired result. \\ \\
For the second bullet point, note that for any (context, arm) pair $(i,j)$, 
$$\begin{aligned} \Vert \psi_{i,j}\Vert^2_2 &=\Vert U_i\Vert^2_2\Vert V_j\Vert^2_2+\Vert U_{\perp,i}\Vert^2_2\Vert V_j\Vert^2_2+\Vert U_i\Vert^2_2\Vert V_{\perp,j}\Vert^2_2 \\ &= \Vert V_j\Vert^2_2+\Vert U_i\Vert^2_2\left(1-\Vert V_j\Vert^2_2\right) \\
&\leq \Vert V_j\Vert^2_2+\Vert U\Vert^2_{2 \to \infty}\left(1-\Vert V_j\Vert^2_2\right) \\
&= \Vert U\Vert^2_{2 \to \infty}+\Vert V_j\Vert^2_2\left(1-\Vert U\Vert^2_{2 \to \infty}\right) \\
&\leq \Vert U\Vert^2_{2 \to \infty}+\Vert V\Vert^2_{2 \to \infty}\left(1-\Vert U\Vert^2_{2 \to \infty}\right) \\
&\leq \mu^2\frac{r}{n}+\mu^2\frac{r}{m}\left(1-\mu^2\frac{r}{n}\right) \\
&= \mu^2r\frac{m+n-\mu^2r}{mn},
\end{aligned} $$
where the last inequality is obtained by distinguishing the cases $\mu^2=\frac{m}{r}\Vert U\Vert^2_{2 \to \infty}$ and $\mu^2=\frac{n}{r}\Vert V\Vert^2_{2 \to \infty}.$
\end{proof}

There are two interesting things to mention about the second inequality of Lemma \ref{lemma:feature-map-bound}.
First, it is tight in the sense that it is attained for $r=m \wedge n$. Indeed, assume then that $m \wedge n =n$ without loss of generality. On the one hand,
$\Vert \psi_{i,j}\Vert^2_2=\Vert V_j\Vert^2_2+\Vert U_i\Vert^2_2\left(1-\Vert V_j\Vert^2_2\right)=1$ since $V_j$ is a row of the orthogonal matrix $V$.
On the other hand, by orthogonality of $U$, $\mu=\Vert U\Vert_{2 \to \infty}=1$ so that $\frac{\mu^2r\left(m+n-\mu^2r\right)}{mn}=1$.
Second, it directly entails an upper bound on the norm of $\psi_{\pi}.$ Indeed, \begin{align*}
    \Vert \psi_{\pi} \Vert_{2} =\Vert \sum_{i,j}\omega^{\pi}_{i,j}\psi_{i,j}\Vert_{2} 
    \leq \sum_{i,j} \omega^{\pi}_{i,j}  \Vert \psi_{i,j}\Vert_{2} 
    \leq \max_{i,j} \Vert \psi_{i,j}\Vert_{2}, 
\end{align*}
where we have used $\sum_{i,j} \omega^{\pi}_{i,j}=1$ in the last inequality. This ensures that \begin{align}\label{eq:phi-pi-norm-bound} \Vert \psi_{\pi} \Vert^2_{2} \leq \mu^2r\frac{m+n-\mu^2r}{mn}. \end{align}
\newpage
\section{Proofs of Section \ref{subsec:lowerbounds-pe} : Policy Evaluation Lower Bounds}

\subsection{Proof of Theorem \ref{thm:instance-lower-bound-pe}}

In this appendix, we assume that $\xi_t \sim \mathcal{N}(0,\sigma^2)$ and we choose the rank factorization $M=U\Sigma V^\top =P Q^\top $ with $P=U\sqrt{\Sigma}$ and $Q=V\sqrt{\Sigma}$. Furthermore, we denote by $N_i$ the $i$-th row vector of a matrix $N$. As noted in \cite{jun2019bilinear}, the bilinear nature of the reward in low-rank bandit problems makes the derivation of sampling lower bounds more challenging than for linear bandits. We overcome this challenge by considering a subset of perturbations of the reward matrix that behave linearly, yielding a relaxed sampling lower bound.

The quantities associated with a perturbation $M'$ of the reward matrix will be denoted by $P',Q'$ and so on. We will use the following straightforward characterization of matrices of rank at most $r$.
\begin{lemma}\label{lemma:rank-factorization}
$M \in \mathbb{R}^{m \times n}$ is of rank at most $r$ if and only if there exists $P \in \mathbb{R}^{m \times r}$ and $Q \in \mathbb{R}^{n \times r}$ such that $M=PQ^\top .$
\end{lemma}

Note that $M=PQ^\top $ yields $M_{i,j}=P_i^\top Q_j$ for every $(i,j)$.
By using the same arguments as in the proof of the sampling lower bound for linear bandits (see Proposition \ref{prop:sampling-lower-bound-1}) and the characterization of Lemma \ref{lemma:rank-factorization}, we obtain the following result. 

\begin{lemma}\label{lemma:sample-complexity-pe}
    The sample complexity of any $(\varepsilon,\delta)$-PAC estimator of $v^{\pi}$ must satisfy $$T \geq \frac{\operatorname{kl}\left(\delta,1-\delta\right)}{I}$$ for $$I=\inf_{\substack{P' \in \mathbb{R}^{m \times r} ,Q' \in \mathbb{R}^{n \times r},\\
    |v^{\pi}-v^{' \pi}| \geq 2\varepsilon}} \frac{1}{2\sigma^2} \sum_{i,j} \omega_{i,j}(P_i^\top Q_j-P^{'\top}_i Q'_j)^2.$$
\end{lemma}

\begin{proof}[Proof of Lemma \ref{lemma:sample-complexity-pe}]
    Adapting the proof of Proposition \ref{prop:sampling-lower-bound-1} yields  $T \geq \frac{\operatorname{kl}\left(\delta,1-\delta\right)}{J}$ for $$J:=\inf_{\substack{M', \operatorname{rank}(M')=r, \\  |v^{\pi}-v^{' \pi}| \geq 2\varepsilon}} \frac{1}{2\sigma^2} \sum_{i,j} \omega_{i,j}(M_{i,j}-M'_{i,j})^2.$$
    Since the set of rank $r$ matrices is dense in the set of matrices of rank smaller than $r$, a continuity argument ensures that $$J=\inf_{\substack{M', \operatorname{rank}(M') \leq r, \\ |v^{\pi}-v^{' \pi}| \geq 2\varepsilon}} \frac{1}{2\sigma^2} \sum_{i,j} \omega_{i,j}(M_{i,j}-M'_{i,j})^2.$$ The characterization of Lemma \ref{lemma:rank-factorization} concludes the proof.
    \end{proof}

Optimizing in $P'$ and $Q'$ at the same time is delicate, so it is difficult to compute $I$ directly. To alleviate this problem, we provide an upper bound on $I$ by optimizing in $P'$ and $Q'$ separately. We denote by $J_1$ (resp. $J_2$) the infimum obtained when adding the constraint $Q'=Q$ (resp. $P'=P$). We naturally have $I \leq \min(J_1,J_2)$, and we will see that $J_1$ and $J_2$ can both be computed explicitly. 
We recall the definitions $\Lambda^{i}_Q=\sum_{j=1}^{n}\omega_{i,j}Q_jQ_j^\top ,
Q^{i}_{\pi}=\sum_{j=1}^{n}\omega^{\pi}_{i,j}Q_j, \Lambda^{j}_P=\sum_{i=1}^{m}\omega_{i,j}P_iP_i^\top, P^{j}_{\pi}=\sum_{i=1}^{m}\omega^{\pi}_{i,j}P_i.$ These quantities will allow us to simplify the objective and the constraints of the optimization problems.

\begin{lemma}\label{lemma:expression-of-inf}
    For every $(i,j)$, $\Lambda_{Q}^{i}$ and $\Lambda_{P}^{j}$ are invertible. In addition,  $$J_1=\frac{2\varepsilon^2}{\sigma^2\sum_i \Vert Q^{i}_{\pi}\Vert^2_{\left(\Lambda^{i}_{Q}\right)^{-1}}} \text{ and } J_2=\frac{2\varepsilon^2}{\sigma^2\sum_j \Vert P^{j}_{\pi}\Vert^2_{\left(\Lambda^{j}_{P}\right)^{-1}}}.$$
\end{lemma}

\begin{proof}[Proof of Lemma \ref{lemma:expression-of-inf}]
Since $P=U\sqrt{\Sigma}$ and $Q=V\sqrt{\Sigma}$, the $(k,l)$-th entry of $\sum_{i} P_iP_i^\top $ (resp. $\sum_{j} Q_jQ_j^\top $) is $\sqrt{\sigma_k\sigma_l}\sum_{i}U_{i,k}U_{i,l}$ (resp. $\sqrt{\sigma_k\sigma_l}\sum_{j}V_{j,k}V_{j,l}$). By the semi-orthogonality of $U$ and $V$, these entries are equal to $\sigma_{k}$ when $k=l$ and $0$ otherwise: $$\sum_{i} P_iP_i^\top =\sum_{j} Q_jQ_j^\top =\Sigma.$$
Consequently,  $\omega_{\min}\Sigma \preceq \Lambda^{i}_{Q}\preceq \omega_{\max}\Sigma$ and $\omega_{\min}\Sigma \preceq \Lambda^{j}_{P}\preceq \omega_{\max}\Sigma$. In particular, since $\omega_{\min}>0,$ $\Lambda_{Q}^{i}$ and $\Lambda_{P}^{j}$ are invertible. \\ \\
Now, notice that when $Q'=Q$, we have $v^{\pi}-v^{' \pi}=\sum_{i}D_i^\top Q^{i}_{\pi}$ and $\sum_{i,j} \omega_{i,j}(P_i^\top Q_j-P^{'\top}_iQ'_j)^2=\sum_{i} D_i^\top \Lambda_{Q}^{i}D_i$ for $D_i=P_i-P'_i.$
    By the same arguments as in the proof of Theorem \ref{thm:sampling-lower-bound-2}, this observation allows us to write $$J_1=\inf_{\substack{D \in \mathbb{R}^{m \times r}, \\ \sum_i D_i^\top Q^{i}_{\pi}=2\varepsilon}}\frac{1}{2\sigma^2}\sum_{i} \Vert D_i\Vert^2_{\Lambda^{i}_{Q}},$$ and similarly, $$J_2=\inf_{\substack{D \in \mathbb{R}^{n \times r}, \\ \sum_j D_j^\top P^{j}_{\pi}=2\varepsilon}}\frac{1}{2\sigma^2}\sum_{j} \Vert D_j\Vert^2_{\Lambda^{j}_{P}}.$$

An application of Lemma \ref{lemma:psd-infimum} concludes the proof.
\end{proof}

\begin{proof}[Proof of Theorem \ref{thm:instance-lower-bound-pe}] By combining Lemmas \ref{lemma:sample-complexity-pe} and \ref{lemma:expression-of-inf}, we directly get the main statement of Theorem \ref{thm:instance-lower-bound-pe}. Finally, let us show that the lower bound is independent of the choice of rank factorization $M=PQ^\top $. Let $M=P' Q'^\top $ be another rank factorization of $M$. Then, there exists an invertible matrix $R \in \mathbb{R}^{r \times r}$ such that $Q'=Q R$ and $P' = P R^{-\top}$ \cite{piziak1999full}.
This implies that for every $(i,j)$, $Q'_j=\left(Q_j^\top R\right)^\top =R^\top Q_j,$ and $P'_i=\left(P_i^\top R^{-\top}\right)^\top =R^{-1}P_i$, which in turn ensures that $\Lambda^{i}_{Q'}=R^\top \Lambda^{i}_{Q}R$ and  $\Lambda^{j}_{P'}=R^\top \Lambda^{j}_{P}R$ are invertible. Furthermore, $$ \begin{aligned}   
\Vert Q^{' i}_{\pi}\Vert^2_{(\Lambda^{i}_{Q'})^{-1}} &=\left(R^\top Q^{i}_{\pi}\right)^\top \left(R^\top \Lambda^{i}_{Q}R\right)^{-1}\left(R^\top Q^{i}_{\pi}\right) \\
&=\left(Q^{i}_{\pi}\right)^\top RR^{-1}\left(\Lambda^{i}_{Q}\right)^{-1}R^{-\top}R^{\top}Q^{i}_{\pi} \\
&= \Vert Q^{ i}_{\pi}\Vert^2_{(\Lambda^{i}_{Q})^{-1}}, 
\end{aligned}
$$
and similarly, $$\Vert P^{' j}_{\pi}\Vert^2_{(\Lambda^{j}_{P'})^{-1}}=\Vert P^{ j}_{\pi}\Vert^2_{(\Lambda^{j}_{P})^{-1}}.$$
\end{proof}

\subsection{Proof of Proposition \ref{corr:worst-case-lower-bound-pe}}
We first provide a relaxed sampling lower bound as a corollary of Theorem \ref{thm:instance-lower-bound-pe}.

\begin{corollary}\label{corollary:instance-lower-bound-pe}
The policy and model-dependent quantity $L_{M,\pi}$ in the sampling lower bound of Theorem \ref{thm:instance-lower-bound-pe} can be controlled as $$ \begin{aligned} \frac{1}{\sigma_r\omega_{\min}}\max\left(\sum_i \Vert Q^{i}_{\pi}\Vert^2_2,\sum_j \Vert P^{j}_{\pi}\Vert^2_2\right) &\geq \max\left(\sum_i \Vert Q^{i}_{\pi}\Vert^2_{(\Lambda^{i}_{Q})^{-1}},\sum_j \Vert P^{j}_{\pi}\Vert^2_{\left(\Lambda^{j}_{P}\right)^{-1}}\right) \\
&\geq \frac{1}{\sigma_1\omega_{\max}}\max\left(\sum_i \Vert Q^{i}_{\pi}\Vert^2_2,\sum_j \Vert P^{j}_{\pi}\Vert^2_2\right). \end{aligned}
$$
In particular, the sample complexity of any $(\varepsilon,\delta)$-PAC estimator of $v^{\pi}$ must satisfy $$T \geq \frac{\sigma^2\operatorname{kl}\left(\delta,1-\delta\right)}{2\sigma_1\omega_{\max}\varepsilon^2}\max\left(\sum_{i} \Vert Q^{i}_{\pi} \Vert_2^2, \sum_{j} \Vert P^{j}_{\pi} \Vert_2^2\right).$$
\end{corollary}

\begin{proof}[Proof of Corollary \ref{corollary:instance-lower-bound-pe}]
By the proof of Lemma \ref{lemma:expression-of-inf}, $\omega_{\min}\Sigma \preceq \Lambda^{i}_{Q}\preceq \omega_{\max}\Sigma$ and $\omega_{\min}\Sigma \preceq \Lambda^{j}_{P}\preceq \omega_{\max}\Sigma$. In particular,  $\sigma_r\omega_{\min}\leq \lambda_{\min}\left(\Lambda^{i}_{Q}\right) \leq \lambda_{\max}\left(\Lambda^{i}_{Q}\right) \leq \sigma_1\omega_{\max}$, and the same results hold for $\Lambda^{j}_{P}$. Now, note that for any $x \in \mathbb{R}^r$ and any invertible matrix $B \in \mathbb{R}^{r \times r}$,
$\lambda_{\min}\left( B\right)^{-1}\Vert  x \Vert^2_2 \geq  \Vert x\Vert^2_{B^{-1}} \geq \lambda_{\max}\left(B\right)^{-1}\Vert  x \Vert^2_2$. Consequently, for any $\left(i,j\right)$, $$\frac{\Vert Q^{i}_{\pi}\Vert^2_2}{\sigma_r \omega_{\min}} \geq \Vert Q^{i}_{\pi}\Vert^2_{\left(\Lambda^{i}_{Q}\right)^{-1}} \geq \frac{\Vert Q^{i}_{\pi}\Vert^2_2}{\sigma_1 \omega_{\max}} \text{ and } \frac{\Vert P^{j}_{\pi}\Vert^2_2}{\sigma_r \omega_{\min}} \geq \Vert P^{j}_{\pi}\Vert^2_{\left(\Lambda^{j}_{P}\right)^{-1}} \geq \frac{\Vert P^{j}_{\pi}\Vert^2_2}{\sigma_1 \omega_{\max}},$$

and the first result follows.
The sampling lower bound can then be directly deduced from Theorem \ref{thm:instance-lower-bound-pe}.
\end{proof}

The control of the lower bound identified above is tight when $\omega$ is homogeneous and $\kappa=\Theta\left(1\right)$, since in this case, the two sides of the inequality only differ by a constant factor. In particular, the relaxed lower bound then scales as the lower bound of Theorem \ref{thm:instance-lower-bound-pe}.

\begin{remark} We could also upper bound the maximal eigenvalues of the covariance matrices by noticing that their positive semi-definiteness entails $\lambda_{\max}\left(\Lambda^{i}_{Q}\right) \leq \operatorname{Tr}\left(\Lambda^{i}_{Q}\right) = \sum_{j} \omega_{i,j}\Vert Q_j\Vert^2$ and $\lambda_{\max}\left(\Lambda^{j}_{P}\right) \leq \operatorname{Tr}\left(\Lambda^{j}_{P}\right) = \sum_{i} \omega_{i,j}\Vert P_i\Vert^2$. When $r=1$, this would yield a bound that coincides with the bound of Theorem \ref{thm:instance-lower-bound-pe}. However, when $r>1$, these eigenvalue upper bounds would scale with $\sum_{i=1}^{r} \sigma_i$ instead of $\sigma_1$.
\end{remark}
 We are now ready to derive the worst-case lower bound of Proposition \ref{corr:worst-case-lower-bound-pe}. To that end, we will show that $$K_{M,\pi}:=\frac{1}{\sigma_1\omega_{\max}}\max\left(\sum_{i} \Vert Q^{i}_{\pi} \Vert_2^2, \sum_{j} \Vert P^{j}_{\pi} \Vert_2^2\right)$$ is of order $\frac{m+n}{\omega_{\max}mn}$ in the worst case (i.e., for a choice of homogeneous reward matrix and target policy that maximizes the lower bound). We use the notation $\pi_j:=\sum_{i=1}^{m} \rho_i\pi\left(j|i\right)$. We note that the third bullet point of the following proposition corresponds to the statement of Proposition \ref{corr:worst-case-lower-bound-pe}. 

\begin{proposition}\label{prop:minimax-lower-bound-pe-full}
Denote by $\mathcal{H}$ the set of homogeneous matrices of $\mathcal{M}\left(c,c'\right)$ for some fixed $c,c'$.
\begin{itemize}
\item For any $M \in \mathbb{R}^{m \times n}$ of rank $r$, $K_{M,\pi} \leq \frac{\mu^2r}{\left(m \wedge n \right)\omega_{\max}}\max\left(\sum_i \rho_i^2,\sum_j \pi_j^2\right)$
\item When $m=\Theta\left(n\right)$ and $r=\Theta\left(1\right),$ $\displaystyle \sup_{ M \in \mathcal{H}} K_{M,\pi} = \Theta\left(\frac{\max\left(\sum_i \rho_i^2,\sum_j \pi_j^2\right)}{\left(m \vee n\right)\omega_{\max}}\right)$
    \item  When $m=\Theta\left(n\right)$ and $r=\Theta\left(1\right),$
    there is a target policy $\pi$ and a homogeneous reward matrix $M \in \mathcal{H}$ such that any $\left(\varepsilon,\delta\right)$-PAC estimator of $v^{\pi}$ must satisfy $$T \gtrsim \frac{\sigma^2\left(m+n\right)}{\omega_{\max}mn\varepsilon^2}\operatorname{kl}\left(\delta,1-\delta\right).$$
\end{itemize} 
\end{proposition}

\begin{remark}
The upper bound of Theorem \ref{thm:rs-pe-error} matches the scaling in $\mu$ and $r$ suggested by the first bullet point of Proposition \ref{prop:minimax-lower-bound-pe-full}.
\end{remark}

\begin{proof}[Proof of Proposition \ref{prop:minimax-lower-bound-pe-full}]
To prove the first bullet point, notice that $$ \begin{aligned} K_{M,\pi} &= \frac{1}{\sigma_1\omega_{\max}}\max\left(\sum_{i} \Vert \sum_j \omega^{\pi}_{i,j}Q_j \Vert_2^2, \sum_{j} \Vert \sum_i \omega^{\pi}_{i,j}P_i \Vert_2^2\right)
\\
&\leq \frac{1}{\sigma_1\omega_{\max}}\max\left(\sum_{i} \left(\sum_j \omega^{\pi}_{i,j}\Vert Q_j \Vert_2\right)^2,\sum_{j} \left(\sum_i \omega^{\pi}_{i,j}\Vert P_i\Vert_2\right)^2\right) \\
    &\leq \frac{1}{\sigma_1\omega_{\max}}\max\left(\max_{j}  
 \Vert Q_j \Vert_2^2\left(\sum_i \rho_i^2\right), \max_{i} \Vert P_i\Vert_2^2\left(\sum_j \pi_j^2\right)\right). \\ 
 \end{aligned}
 $$
Since $P=U\sqrt{\Sigma}$ and $Q=V\sqrt{\Sigma}$, we have $\max_{i} \Vert P_i \Vert_2^2 =\max_{i} \sum_{j=1}^r \sigma_j U_{x,j}^2 \leq \sigma_{1} \Vert U \Vert_{2 \to \infty}^2$
and similarly, $\max_{j} \Vert Q_j\Vert_2^2 \leq \sigma_1\Vert V \Vert_{2 \to \infty}^2.$ This finally ensures that $$ \begin{aligned} K_{M,\pi} &\leq \frac{1}{\omega_{\max}}\max(\Vert U \Vert_{2 \to \infty}^2,\Vert V \Vert_{2 \to \infty}^2)\max\left(\sum_x \rho_i^2,\sum_a \pi_j^2\right) \\
&\leq \frac{\mu^2r \max(1/m,1/n)}{\omega_{\max}}\max\left(\sum_x \rho_i^2,\sum_a \pi_j^2\right) \\
&=  \frac{\mu^2r}{\left(m \wedge n \right)\omega_{\max}}\max\left(\sum_x \rho_i^2,\sum_a \pi_j^2\right). \end{aligned} $$ 

For the second bullet point, let $M \in \mathcal{H}$. The first bullet point and the homogeneity properties of $M$ entail $K_{M,\pi} \lesssim \frac{\max\left(\sum_i \rho_i^2,\sum_j \pi_j^2\right)}{\left(m \wedge n\right)\omega_{\max}} \lesssim \frac{\max\left(\sum_i \rho_i^2,\sum_j \pi_j^2\right)}{\left(m \vee n\right)\omega_{\max}}$, where we have used $m=\Theta\left(n\right)$ in the last inequality. To achieve this upper bound, we choose a homogeneous matrix of rank $r$ such that $U$ can be chosen with a first column equal to $\left(1/\sqrt{m},\dots,1/\sqrt{m}\right)^\top $, and $V$ can be chosen with a first column equal to $\left(1/\sqrt{n},\dots,1/\sqrt{n}\right)^\top $. An explicit example for $r=1$ is the matrix that has each entry equal to $1$. For such a choice, we have 
$$ \begin{aligned}
    \sum_{j} \Vert P^{j}_{\pi}\Vert_2^2 &= \sum_j\left\Vert \sum_i \rho_i\pi\left(j|i\right)P_i \right\Vert_2^2 \\
    &\geq \sum_{j}\left(\sum_{i}\rho_i\pi\left(j|i\right)\sqrt{\sigma_1}U_{x,1}\right)^2 \\
    &= \frac{\sigma_1}{m}\left(\sum_j \pi_j^2\right),
    \end{aligned}, $$

and $$ \begin{aligned}
    \sum_{i} \Vert Q^{i}_{\pi}\Vert_2^2 &= \sum_i \left\Vert  \sum_{j} \rho_i\pi\left(j|i\right)Q_j \right\Vert_2^2 \\
    &\geq \sum_{i}\left(\sum_{j}\rho_i\pi\left(j|i\right)\sqrt{\sigma_1}V_{j,1}\right)^2 \\
    &= \frac{\sigma_1}{n}\left(\sum_i \rho_i^2\right).
    \end{aligned}, $$
In particular, $$K_{M,\pi} \geq \frac{\max\left(\sum_i \rho_i^2,\sum_j \pi_j^2\right)}{\left(m \vee n\right)\omega_{\max}}. $$ \\ \\
Lastly, notice that the supremum is largest when $\pi$ always selects the same arm, independently of the context\footnote{Intuitively, when the behavior policy is uniform, the distribution shift is maximized when the target policy is constant.}. In this case, there is a reward matrix $M \in \mathcal{H}$ such that we have the sampling lower bound $$T \geq \frac{\sigma^2\operatorname{kl}\left(\delta,1-\delta\right)}{2\left(m \vee n\right)\omega_{\max}\varepsilon^2}.$$ Finally, notice that $m=\Theta\left(n\right)$ entails $m \vee n=\Theta\left(\frac{mn}{m+n}\right).$
\end{proof}
\newpage
\section{Proofs of results from Section \ref{sec:bpi}: Best Policy Identification}\label{app:bpi}

\subsection{Proof of Theorem \ref{thm:bpi-bound}}

In this appendix, we denote by $\Pi$ the set of deterministic policies (in other words, the set of functions from $[m]$ to $[n]$). To ease the notation, we also denote by $\hat{v}^{\pi}_{\rs}$ the \rspe\ estimator of $v^{\pi}$ for a given policy $\pi$, and by $\hat{\pi}_{\rs}$ the policy learned by \rsbpi. 

We first derive an upper bound on the max-norm error of the reward matrix estimator $\widebar{M}$ utilized in our PE and BPI algorithms, which is defined by \begin{align}\label{eq:def_of_Mbar}\widebar{M}_{i,j}:=\phi_{i,j}^T\hat{\theta}.
\end{align}


\begin{proposition}\label{prop:rs-max-error}
Assume that the regularization parameter of $\hat{\Lambda}_{\tau}$ verifies $\tau \leq \frac{r}{\left(m \wedge n\right)}\frac{\omega_{\min}}{\omega_{\max}}\log\left(16d/\delta\right)$. With probability larger than $1-\delta$, $$\Vert M-\widebar{M}\Vert _{\max} \lesssim \varepsilon$$
as soon as $$T \gtrsim \frac{\sigma^2\mu^2r\left(m+n\right)}{\omega_{\min}mn\,\varepsilon^2}\log\left(\frac{m+n}{\delta}\right) + \frac{K_0}{\varepsilon}\log^{3}\left( \frac{(m+n)T}{\delta}\right)$$

where $K_0$ has been
defined in \eqref{eq:def_of_K0}. \end{proposition}

\begin{proof}[Proof of Proposition \ref{prop:rs-max-error}] By replacing $\phi_{\pi}$ by $\phi_{i,j}$ in the error analysis performed in Appendix \ref{app:pe}, we directly obtain that for every (context, arm) pair $(i,j)$, with probability larger than $1-\frac{\delta}{mn}$, $$ \vert M_{i,j}-\phi_{i,j}^T\hat{\theta} \vert \leq \varepsilon $$
as soon as $$T \gtrsim \frac{\sigma^2\mu^2r\left(m+n\right)}{\omega_{\min}mn\,\varepsilon^2}\log\left(\frac{emn}{\delta}\right) + \frac{K_0}{\varepsilon}\log^{3}\left( \frac{mn(m+n)T}{\delta}\right).$$

Note that $\log\left(\frac{emn}{\delta}\right) \lesssim \log\left(\frac{m+n}{\delta}\right)$ and $\log^3\left(\frac{mn\left(m+n\right)T}{\delta}\right) \lesssim \log^3\left(\frac{\left(m+n\right)T}{\delta}\right)$. A union bound ensures that with probability larger than $1-\delta,$ \begin{align*}\Vert M-\widebar{M}\Vert _{\max} &= \max_{i,j} \vert M_{i,j}-\phi_{i,j}^T\hat{\theta} \vert \leq \varepsilon\end{align*}

under the stated condition on the sample complexity.
\end{proof}

We also wish to note that, although this will not be used in the proof of Theorem \ref{thm:bpi-bound}, the same argument can be utilized to derive a max-norm version of the PE error bound \eqref{eq:sharp-minimax-bound}.
Namely, under the assumptions of Proposition \ref{prop:rs-pe-random-error},  \begin{align}\label{eq:sharp-minimax-bound-max-norm}
 \Vert M-\widebar{M} \Vert_{\max} \leq \mu\sqrt{\frac{r\left(m+n\right)}{\omega_{\min}mn}}\sqrt{\frac{2\sigma^2\log(16mn/\delta)}{\left(1-\alpha\right)T}} +K_0\frac{\log^3\left(mn\left(m+n\right)T/\delta\right)}{T}\end{align} with probability larger than $1-\delta,$ where $K_0$ has been defined in \eqref{eq:def_of_K0}.

We are now ready to prove Theorem \ref{thm:bpi-bound}.
\begin{proof}[Proof of Theorem \ref{thm:bpi-bound}] 

We provide a detailed proof of the second statement of the theorem first. Let us recall  that \begin{align}\label{eq:max-policy-bound}\max_{\pi \in \Pi} \vert v^{\pi}-\hat{v}^{\pi}_{\rspe}\vert \  \leq \max_{\pi \in \Pi} \max_{i,j} \vert M_{i,j}-\widebar{M}_{i,j}\vert \ \sum_{i,j}\omega^{\pi}_{i,j}  \leq \Vert M-\widebar{M}\Vert _{\max}.\end{align}
Note that for a homogeneous context distribution, we have $\omega_{\min}=\Theta\left(\frac{1}{mn}\right)$ since the learner samples the arms uniformly. Consequently, when  $T \gtrsim \frac{\sigma^2\mu^2r\left(m+n\right)}{\varepsilon^2}\log\left(\frac{m+n}{\delta}\right) + \frac{K_0}{\varepsilon}\log^{3}\left( \frac{(m+n)T}{\delta}\right)$, Proposition \ref{prop:rs-max-error} yields $v^{*}-v^{\hat{\pi}}_{\rspe} \leq \varepsilon$ with probability larger than $1-\delta$.
Indeed, $v^{*}-v^{\hat{\pi}_{\rs}} \leq \vert v^{\hat{\pi}_{\rs}}-\hat{v}^{\hat{\pi}_{\rs}}_{\rs}\vert  + \vert \hat{v}^{\hat{\pi}_{\rs}}_{\rs}-v^{*}\vert .$
Under the event $\Vert M-\widebar{M} \Vert_{\max} \leq \varepsilon /2$,
The first term can be directly controlled by $\varepsilon/2$ with \eqref{eq:max-policy-bound}. For the second term, note that $\hat{\pi}_{\rs}=\arg \max_{\pi} \hat{v}^{\pi}_{\rs}$. Consequently, under the same event,
$\hat{v}^{\hat{\pi}_{\rs}}_{\rs} \geq \hat{v}^{\pi^\star}_{\rs} \geq v^{*}-\varepsilon/2$, and
$v^{*} \geq v^{\hat{\pi}_{\rs}} \geq \hat{v}^{\hat{\pi}_{\rs}}_{\rs}-\varepsilon/2$.
By Proposition \ref{prop:rs-max-error}, with probability larger than $1-\delta$, we have $$v^{*}-v^{\hat{\pi}_{\rs}} \leq \varepsilon \text{ for } T \gtrsim \frac{\sigma^2\mu^2r\left(m+n\right)}{\varepsilon^2}\log\left(\frac{m+n}{\delta}\right) + \frac{K_0}{\varepsilon}\log^{3}\left( \frac{(m+n)T}{\delta}\right)$$
for \begin{align}\label{eq:def_of_K0_2} K_0 \lesssim \frac{L^2\mu^{6}\kappa^{4}r^{3}\left(m+n\right)^{2}\sqrt{mn}}{\Vert M \Vert_{\max}\left(m \wedge n\right)}.\end{align}

We finally recall that from Remark \ref{remark:homogeneous-scaling}, $K_0 \lesssim \sigma mn$ when $M$ is also homogeneous and $\Vert M \Vert_{\max} = \Theta(\sigma).$

We can of course derive the corresponding result for \dsmbpi\ in the same way. More specifically, we obtain that $v^\star-v^{\hat{\pi}_{\dsmbpi}} \leq \varepsilon$ with probability larger than $1-\delta$ as soon as \begin{align*}
    T \gtrsim  \frac{L^2\mu^{6} \kappa^4 r^3 (m+n)mn}{(m\wedge n)^2 \varepsilon^2}\log^3\!\left(  \frac{(m+n)T}{\delta}\right)\!
\end{align*} by combining the max-norm guarantee on $\widehat{M}$ of Proposition \ref{prop:dsm-max-error} with the previous arguments\footnote{to get the exact statement of the theorem, note that  $\left(m+n\right)mn < \left(m+n\right)^3$.}.
\end{proof}

\subsection{General context distribution case}\label{subsec:general-context}
We now explain how to drop the homogeneous context assumption in Theorem \ref{thm:bpi-bound} while retaining the optimal scaling up to logarithmic factors with the technique highlighted in \cite{lee2023context}. The idea is to split the set of contexts into subsets on which the context distribution is homogeneous and to apply \rsbpi\ or \dsmbpi\ on each subset. Let \begin{align}\label{eq:def_of_L} \mathcal{L}=\ceil{\log_2\left(m/\varepsilon\right)}\end{align} and define $\mathcal{C}_{\mathcal{L}}=\left\{i \in [m], 2^{-l-1} < \rho_i \leq 2^{-l}\right\}$ for $0 \leq l \leq \mathcal{L}-1$, and $\mathcal{C}_\mathcal{L}=\left\{i \in [m], \rho_i \leq 2^{-\mathcal{L}}\right\}$. On each $\mathcal{C}_l$ for $0 \leq l \leq \mathcal{L}-1$, the distribution defined by the normalized context weights is homogeneous. On $\mathcal{C}_\mathcal{L}$, the context weights are small enough to not meaningfully contribute to the BPI error. We summarize below the algorithm structure for \rsbpi, and we recall that $\alpha$ is the proportion of samples used in the first phase of the algorithm. Of course, a \dsmbpi\ version of this algorithm can be constructed in the same way.

\begin{algorithm}[H]
\caption{\textsc{\textbf{G}eneralized \textbf{RS}-\textbf{BPI}} $(\grsbpi)$ }\label{algo:GRSBPI}
\begin{algorithmic}
\STATE \textbf{Input}: Budget of rounds $T$, data splitting parameter $\alpha$, context distribution $\rho$, regularization parameter $\tau$
   \FOR{$t=1$ {\bfseries to} $T$}
   \STATE Observe a context $i_t$ according to $\rho$
   \STATE Select an arm $j_t$ uniformly at random
   \ENDFOR
\STATE Split $[m]$ into the subsets $\mathcal{C}_l$ defined above 
\FOR{$l=0,\dots,\mathcal{L}-1$}
\STATE Execute Algorithm \ref{algo:RSBPI} on $\mathcal{C}_l$ to learn $\hat{\pi}^l_{\rs}$ 
\ENDFOR
\STATE \textbf{Output:} $\hat{\pi}_{\rs}$ such that $\hat{\pi}_{\rs}=\hat{\pi}^l_{\rs}$ on $\mathcal{C}_l$ for $0 \leq l \leq \mathcal{L}-1$ and arbitrary on $\mathcal{C}_\mathcal{L}$. 
\end{algorithmic}
\end{algorithm}

To clarify, all samples are gathered at the start. Then, for each $l$, we set $T_l$ as the number of samples of the form $(i_t,j_t,r_t)$ for $i_t \in \mathcal{C}_l$, and we construct a reward estimator $\widebar{M}^l$ from the $T_l$ samples with the definition \eqref{eq:def_of_Mbar}. We then select $\hat{\pi}^l_{\rs}$ with $\hat{\pi}^l_{\rs}(i)=\argmax_{1 \leq j \leq n}\widebar{M}^l_{i,j}.$ This generalized version of \rsbpi\ satisfies the following guarantee.

\begin{theorem}\label{thm:bpi-general-bound}
Let $\varepsilon > 0$, $\delta \in \left(0,1\right)$, and define $\rho\left(\mathcal{C}_l\right):=\sum_{i \in \mathcal{C}_l} \rho_i$. Assume that $\varepsilon < \left(\displaystyle \min_{0 \leq l \leq \mathcal{L}-1} \rho\left(\mathcal{C}_l\right)\right)\Vert M\Vert_{\max}$. With the choices $T_1 = \lfloor \alpha T \rfloor$ and $\tau \le r(m\wedge n)^{-1} (\rho_{\min}/\rho_{\max}) \log(16d/\delta)$, $v^\star-v^{\hat{\pi}_{\rs}} \leq \varepsilon'$ with probability larger than $1-\delta$ as soon as \begin{align*}
    T \gtrsim \frac{\sigma^2\mu^2r\left(m+n\right)}{\varepsilon^2}\log\left(\frac{\mathcal{L}\left(m+n\right)}{\delta}\right) + \frac{K_0}{\varepsilon}\log^{3}\left( \frac{\mathcal{L}(m+n)T}{\delta}\right)\!, 
\end{align*}
where $K_0$ has been defined in \eqref{eq:def_of_K0_2}, $\mathcal{L}$ has been defined in \eqref{eq:def_of_L} and $\varepsilon'=\varepsilon\left(\mathcal{L}+2\Vert M\Vert _{\max}\right)$.
    
\end{theorem}

\begin{proof}[Proof of Theorem \ref{thm:bpi-general-bound}]
Let $0 \leq l \leq \mathcal{L}-1$ and $\delta_l=\delta/\mathcal{L}$. From Lemma 9 in \cite{lee2023context}, we know that with high probability, at least $$\frac{T\rho\left(\mathcal{C}_l\right)}{2} \gtrsim \frac{\sigma^2\mu^2r\left(m+n\right)}{\varepsilon_l^2}\log\left(\frac{m+n}{\delta_l}\right) + \frac{K_0}{\varepsilon_l}\log^{3}\left( \frac{(m+n)T}{\delta_l}\right)$$ out of the $T$ samples correspond to a context in $\mathcal{C}_l$ for $\varepsilon_l:=\varepsilon/\rho\left(\mathcal{C}_l\right).$ 
Since $\varepsilon_l \leq \Vert M \Vert_{\max}$, by Theorem \ref{thm:bpi-bound}, Algorithm \ref{algo:RSBPI} on $\mathcal{C}_l$ yields a policy $\hat{\pi}^l_{\rs}$ such that $$\sum_{i \in \mathcal{C}_l} \frac{\rho_i}{\rho\left(\mathcal{C}_l\right)}\left(M_{i,\hat{\pi}^l_{\rs}(i)}-M_{i,\pi^\star(i)}\right) \leq \frac{\varepsilon}{\rho\left(\mathcal{C}_l\right)}$$ with probability larger than $1-\delta_l.$
A union bound ensures that with probability larger than $1-\delta$, the overall BPI error is bounded as follows:
$$\begin{aligned} \sum_{i=1}^{m} \rho_i\left(M_{i,\hat{\pi}_{\rs}(i)}-M_{i,\pi^\star(i)}\right) &= \sum_{l=0}^{\mathcal{L}-1}\sum_{i \in \mathcal{C}_l} \rho_i\left(M_{i,\hat{\pi}^l_{\rs}(i)}-M_{i,\pi^\star(i)}\right)+\sum_{i \in \mathcal{C}_\mathcal{L}} \rho_i\left(M_{i,\hat{\pi}^l_{\rs}\left(i\right)}-M_{i,\pi^\star(i)}\right) \\
&\leq \varepsilon \mathcal{L}+ 2m2^{-\mathcal{L}}\Vert M\Vert _{\max} \\
&\leq \varepsilon \mathcal{L}+2\varepsilon\Vert M\Vert _{\max} \\
&= \varepsilon'.
\end{aligned}$$
\end{proof}
\newpage
\section{Proof of results from Section \ref{sec:regret}: Regret Minimization}
\label{app:regret_min}

In this appendix, we present the regret analysis of Algorithm \ref{algo:ESRED}. We postpone the proof of Theorem \ref{thm:main_ESRED_context} to Section \ref{subsec:appendix_proof_main_regret}, and first repeat the analysis of SupLinUCB algorithm given in \cite{takemura2021parameter}, culminating in Theorem \ref{thm:mainthm_LowSupLinUCB_misspec}, which we then use to prove our upper bound on the regret of \rsslin\ in Section \ref{subsec:appendix_proof_main_regret}. 

\subsection{Misspecified linear bandits}
\label{subsec:misspecified_linear_regret}
We assume that in each round $t$, the learner is provided with a context set of $K$ actions $\cX_t = \lbrace x_{t,1}, \dots, x_{t,K}\rbrace$ with $x_{t,a}\in \RR^d$ and $\|x_{t,a}\|_2\leq 1$ for $a \in[K]$. The context sets are drawn according to some distribution $\rho$ and are independent across rounds. 


The reward of action $a$ in round $t$ is expressed  as  $r_{t,a}= x_{t,a}^\top\theta +  \epsilon_{t,a}+ \xi_t $. We assume that the misspecification is bounded, i.e., $\sup_{t,a} |\epsilon_{t,a}| \le \epsilon_{\max}$ and that $\max_{t,a}\vert r_{t,a}\vert = r_{\max}$. The noise process $\{\xi_t\}_{t\ge 1}$ is i.i.d. with $\sigma$-subgaussian distribution. The regret up to round $T$ of an algorithm $\pi$ selecting $a_t^\pi$ in round $t$ is defined by:$$
{\cal R}^\pi(T) = \sum_{t=1}^T \left( r_{t,a_t^\star}- r_{t,a_t^\pi}\right),
$$
where $a_t^\star = \argmax_{a \in [A]} \EE[r_{t,a}]$. To solve the misspecified linear bandit problem, we apply Algorithm 1 in \cite{takemura2021parameter}, a variant of \textsc{SupLinUCB} \citep{chu2011contextual}. The pseudo-code is presented in Algorithm \ref{algo:LowSupLinUCB_reg}.

\begin{algorithm}[h]
   \caption{\textsc{SupLinUCB} (Algorithm 1 in \cite{takemura2021parameter})}
   \label{algo:LowSupLinUCB_reg}
\begin{algorithmic}
   \STATE {\bfseries Input:} number of rounds $T$, threshold $\beta$ and regularization matrix $\Lambda$
   \STATE $J = \lceil \log_2 (T/d)/2\rceil+1$
   \STATE $\Psi_1^j = \emptyset ,\ \forall j\in[J]$ 
   \FOR{$t=1$ {\bfseries to} $T$}
    \STATE $j = 1$, $\cA_1 = [K]$
    \REPEAT
    \STATE $V_{t} = \Lambda + \sum_{\tau\in\Psi_t^j} x_{\tau,a_\tau} x_{\tau,a_\tau}^\top$
    \STATE $\htheta_{t} = V_{t}^{-1} \sum_{\tau\in\Psi_t^j} x_{\tau,a_\tau} r_{\tau,a_\tau}$ 
    \FOR{$a \in \cA_j$}
    \STATE $\hat{r}_{t,a}^j = \langle \htheta_{t},x_{t,a}\rangle $
    \STATE $w_{t,a}^j =  \| x_{t,a}\|_{V_{t}^{-1}}\beta$ 
    \ENDFOR
    \IF{$w_{t,a}^j \leq \beta\sqrt{d/T}$ for all $a\in \cA_j$}
    \STATE $a_t = \argmax_{a\in \cA_j} (\hat{r}_{t,a}^j+w_{t,a}^j)$
    \STATE $\Psi_{t+1}^{j'} \leftarrow \Psi_t^{j'}$ for all $j'\in[J]$
    \ELSIF{$w_{t,a}^j \leq \beta 2^{-j}$ for all $a\in\cA_j$}
    \STATE $\cA_{j+1} \leftarrow  \{ a\in \cA_j: (\hat{r}_{t,a}^j+w_{t,a}^j) \geq \max_{a'\in\cA_j} (\hat{r}_{t,a'}^j+w_{t,a'}^j) - 2^{1-j}\beta \}$
    \STATE $j \leftarrow j+1$
    \ELSE
    \STATE Choose $a_t\in \cA_j$ s.t. $w_{t,a_t}^j > \beta 2^{-j}$
    \STATE $\Psi_{t+1}^{j'} \leftarrow \begin{cases}
        \Psi_t^{j'}\cup \{t\}, &\mathrm{if}\ j' = j\\
        \Psi_t^{j'}, &\mathrm{else}
    \end{cases}$
    \ENDIF
    \UNTIL{an action $a_t$ is found.}
   \ENDFOR
\end{algorithmic}
\end{algorithm}

\begin{theorem}[Simplified version of Theorem 1 in \cite{takemura2021parameter}]
\label{thm:mainthm_LowSupLinUCB_misspec}
Let $B>0$ be such that $\Vert \theta\Vert_2 \leq B$, and let $\Lambda = (\sigma^2/B^2) I_d$. Moreover, let $\delta \in (0,1)$, $T > 0$ and define the threshold 
\begin{align}
    \beta(\delta) =\sigma (1+\sqrt{2\log \left( \frac{TK}{\delta} \ceil{\frac{1}{2}\log\frac{T}{d}} \right) }).
\end{align}
Then, there exists a universal constant $C>0$ such that the regret of the algorithm $\pi=$ \textsc{SupLinUCB} satisfies:
\begin{align*}
 \mathcal{R}^\pi(T) \leq C \left( \sigma \sqrt{d T \log K} \!+  \epsilon_{\max} \sqrt{d} T + d r_{\max} \right) \log(T+1)\log\left( 1+ \frac{TB^2}{d\sigma^2} \right), 
\end{align*}
with probability at least $1-\delta$.
\end{theorem}

The result of Theorem \ref{thm:mainthm_LowSupLinUCB_misspec} entails that Algorithm \ref{algo:LowSupLinUCB_reg} has an expected regret\footnote{The regret guarantee of Theorem \ref{thm:mainthm_LowSupLinUCB_misspec} is stated with high probability, but the result can be immediately used to bound the regret in expectation by choosing for example $\delta = 1/T$.}  of order $\widetilde{O}(\sqrt{dT \log(K) } + \epsilon_{\max}\sqrt{d} T)$, where $\widetilde{O}$ hides polylogarithmic factors in $d$ and $T$. Therefore, \textsc{SupLinUCB} enjoys tight dependence in the dimension in both the term $\widetilde{O}(\epsilon_{\max}\sqrt{d} T)$ that is due to the misspecification \citep{lattimore2020learning} and in the term $\sqrt{dT  \log(K) }$ that corresponds to the minimal regret without misspecification  \citep{auer2002using, chu2011contextual}.

\subsection{Proof of Theorem  \ref{thm:main_ESRED_context}}
\label{subsec:appendix_proof_main_regret}

Let us denote by $j_t^\star$ the optimal arm given context $i_t$, i.e., let $j_t^\star = \pi^\star(i_t)$ for all $t\in[T]$. We split the regret into the sum of two terms corresponding to the uniform exploration phase and a second phase where we apply Algorithm \ref{algo:ESRED}:
\begin{align*}
    R^\pi(T) = \underbrace{\sum_{t=1}^{T_1} \EE[M_{i_t, j^\star_t} - M_{i_t, j_t^\pi}]}_{\substack{ R_1^{\pi} \\ \textit{(Regret of phase 1)} }} + \underbrace{\sum_{t=T_1 +1}^T \EE[M_{i_t, j^\star_t} - M_{i_t, j_t^\pi}]}_{\substack{ R_2^{\pi} \\ \textit{(Regret of phase 2)} }}.
\end{align*}

\underline{\it Step 1: (Regret of phase 1).} Since $\vert M_{i,j} \vert \leq \Vert M\Vert_{\max}$ for all $(i,j)\in[m]\times [n]$, the first phase lasting for $T_1$ rounds accumulate a regret of at most:
\begin{align}\label{eq:reg_1_term}
    R_1^\pi \le 2T_1 \Vert M\Vert_{\max}. 
\end{align}

\underline{\it Step 2: (Regret of phase 2).} 
In the second phase, which lasts for $T_2 = T-T_1$ rounds, we can identify our observations $M_{i_t,j_t} + \xi_t = \phi_{i_t,j_t}^{\top} \theta + \epsilon_{i_t,j_t} + \xi_t$ with rewards $r_{t,j_t}$ from the previous section, where we define the set of arms for the context $i_t$ by setting $(x_{t,a})_{a\in[K]} = (\phi(i_t,j))_{j\in  [n]}$ with $K = n$, $x_{a^\star_t} = \phi(i_t,j^\star_t)$ and $d = r(m+n)-r^2$. Thus, the regret $R_2^\pi$ generated during this phase can be equivalently written as:\begin{align*}
    R_2^\pi
    = \sum_{t=T_1+1}^T \EE [r_{t,a^\star_t} - r_{t,a_t^\pi}]
    = \EE [{\cal R}^{\pi}(T_2)].
\end{align*}
Thus, we can apply Theorem \ref{thm:mainthm_LowSupLinUCB_misspec} to show that there exists a universal constant $C_2$ such that with probability at least $1-\delta$:
\begin{align}
    {\cal R}^\pi(T_2) \leq C_2 \sqrt{r (m+n)} \left(  \sigma \sqrt{ T \log\left(\frac{Tmn}{\delta}\right)} + \epsilon_{\max}T + \Vert M \Vert_{\max} \sqrt{r(m+n)} \right) \log^2 \left( T (m+n) \frac{L^2}{\sigma^2} \right)
    \label{eq:r2_regret_proof_secD}
\end{align}
where we used that $T_2 \leq T$. To bound the second term in the parenthesis, we use the upper bound on the misspecification $\epsilon_{\max}$ of Corollary \ref{corr:misspecification}. Indeed, we have that with probability at least $1-\delta$:
    \begin{align}
         \epsilon_{\max} \le  \frac{C L^2\kappa^3\mu^2 r (m+n) }{\sigma_r T_1\omega_{\min}(m\wedge n)}\log^3\left(\frac{(m+n)T_1}{\delta}\right), 
         \label{eq:eps_max_app_proof_regret}
    \end{align}
    provided that 
    $T_1 \geq c \frac{L^2 (m+n)}{\sigma_r^2 \omega_{\min}} \log^3\left( \frac{(m+n)T}{\delta}\right)$,
 for some universal constants $C, c > 0$. Define $\mathcal{E}$ as the event under which \eqref{eq:eps_max_app_proof_regret} holds. Next, we decompose $R_2^\pi$ as follows:
\begin{align*}
    R_2^\pi  = \sum_{t = T_1 + 1}^T  \EE\left[ r_{t,a^\star_t} - r_{t, a^\pi_t}  \right]  
    & \le 2 \Vert M \Vert_{\max} \PP(\mathcal{E}^c) T +  \sum_{t = T_1 + 1}^T  \EE\left[\left(r_{t,a^\star_t} - r_{t, a^\pi_t}\right) \mathds{1}_{\lbrace \mathcal{E} \rbrace }\right] \\
    & \le 2 \Vert M \Vert_{\max} \delta T  +  \sum_{t = T_1 + 1}^T  \EE\left[\left(r_{t,a^\star_t} - r_{t, a^\pi_t}\right) \mathds{1}_{\lbrace \mathcal{E} \rbrace }\right]. 
\end{align*}
Define $\mathcal{L}_{mnT} = \log(1+mnT)$. Substituting $\epsilon_{\max}$ with its upper bound in \eqref{eq:eps_max_app_proof_regret} and selecting $\delta=1/T$ in \eqref{eq:r2_regret_proof_secD} gives:
\begin{align}
    R_2^\pi
   \leq C_3 \sqrt{r(m+n)} \Big( \sigma \sqrt{T\mathcal{L}_{mnT} } +      
   T \frac{L^2\kappa^3\mu^2 r (m+n) }{\sigma_r T_1\omega_{\min}(m\wedge n)}\mathcal{L}_{mnT}^3 
   + \Vert M\Vert_{\max} \sqrt{r(m+n)} \Big)
   \log^2 \left( T (m+n)\frac{L^2}{\sigma^2} \right)
    \label{eq:R_pi_T_proof_noncon_red}
\end{align}
    provided that $T_1 \geq c \frac{L^2 (m+n)}{\sigma_r^2 \omega_{\min}} \log^3\left( (m+n)T^2 \right)$, for some universal constant $C_3, c > 0$ and $T_1 \leq T$.

    \underline{\it Step 3: (Optimizing $T_1$).} Observe that the upper bound of $R^{\pi}_{1}$ scales as $T_1$ (see \eqref{eq:reg_1_term}) while the second term in the upper bound of $R^{\pi}_2$ scales as $1/T_1$ (see \eqref{eq:R_pi_T_proof_noncon_red}). Therefore, to obtain a tight regret bound on $R^\pi(T)$, we need to balance these two terms that depend on $T_1$ with a proper choice of $T_1$. We can easily verify that an optimal choice is:  
    \begin{align*}
        T_1^\star =   \sqrt{\frac{C_3}{2}} \frac{\mu \kappa^{3/2} L}{\sqrt{\Vert M\Vert_{\max}}}  \frac{r^{3/4}(m+n)^{3/4} }{\sqrt{ \sigma_r \omega_{\min} (m\wedge n)}} \sqrt{T }\log^{3/2}\left( 1+mnT \right)\log \left( T (m+n) \frac{L^2}{\sigma^2}  \right). 
    \end{align*}
    Next, note that we can use that $\Vert M\Vert_{\max}/\sigma_r(M)\leq \mu^2 \kappa r/\sqrt{mn}$ from Lemma \ref{lem:spikiness} and the fact that $\omega_{\min} = \rho_{\min}/n$ to simplify further the last quantity and define:
    \begin{align}
    \label{eq:T1_def_regret_app}
        T_1 =   \sqrt{\frac{C_3}{2}} \frac{\mu^2 \kappa^{2} L}{\Vert M\Vert_{\max}}  \frac{r^{5/4}(m+n)^{3/4} (mn)^{1/4} }{\sqrt{  m\rho_{\min} (m\wedge n)} } \sqrt{T }\log^{3/2}\left( 1+mnT \right)\log \left( T (m+n) \frac{L^2}{\sigma^2}  \right). 
    \end{align}
    This choice of $T_1$ entails that the regret can be upper bounded as follows:
    \begin{align*}
     R^\pi(T)  = \widetilde{O}\left(\mu^2 \kappa^{2} r^{5/4} L \frac{(m+n)^{3/4}(mn)^{1/4}}{\sqrt{m \rho_{\min}(m\wedge n)}} \sqrt{T} \right).
    \end{align*}


\underline{\it Step 4: (Checking conditions on $T_1$).} It remains to check whether our choice of $T_1$ verifies $T_1 \geq c \frac{L^2 (m+n)}{\sigma_r^2 \omega_{\min}} \log^3\left( \frac{(m+n)T}{\delta}\right)$ and $T_1\leq T$. These conditions are satisfied if 
\begin{align*}
    T = \widetilde{\Omega} \left( \mu^4 \kappa^4 r^{5/2} L^2 \frac{(m+n)^{3/2} \sqrt{mn}}{\Vert M\Vert_{\max}^2 m\rho_{\min} (m\wedge n)} \right).
\end{align*}
Now, note that when the above condition does not hold, i.e.,
\begin{align*}
    T = \widetilde{O} \left( \mu^4 \kappa^4 r^{5/2} L^2 \frac{(m+n)^{3/2} \sqrt{mn}}{\Vert M\Vert_{\max}^2 m\rho_{\min} (m\wedge n)} \right),
\end{align*}
then the trivial bound $R^\pi(T) \le 2 T \Vert M \Vert_{\max}$, gives us after a few simple computations the upper bound 
\begin{align*}
    R^\pi(T) \le 2 \Vert M \Vert_{\max} T = \widetilde{O}\left(\mu^2 \kappa^{2} r^{5/4} L \frac{(m+n)^{3/4}(mn)^{1/4}}{\sqrt{m \rho_{\min}(m\wedge n)}} \sqrt{T} \right).  
\end{align*}


If, furthermore, $M$ and $\rho$ are homogeneous, we obtain, 
for any $T\geq 1$:
\begin{align*}
    R^\pi(T) = \widetilde{O}\left( L (m+n)^{3/4}\sqrt{T} \right). 
\end{align*}



\newpage
\section{Reduction to almost low-dimensional linear bandits and its limitations}
\label{subsec:main_reduction_almost_lowd}

We start by observing that our low-rank bandit settings can be viewed as an instance of the bilinear bandit problem \citep{jun2019bilinear}, where for any $i,j \in [m]\times[n]$, $M_{i,j} = e_i^\top M e_j$, and we may think of $e_i$ and $e_j$ as the left and right feature respectively. 
Similarly to Section \ref{subsec:reductions}, we define extended vectors $\phi_{i,j}^\ext,\theta^\ext\in\mathbb{R}^{mn}$ as follows:
\begin{align*}
    \phi_{i,j}^\ext = \begin{bmatrix}\textup{vec}( \widehat{U}^\top e_i e_j^\top \widehat{V} ) \\
    \textup{vec}( \widehat{U}^\top e_i e_j^\top \widehat{V}_\perp  ) \\
    \textup{vec}( \widehat{U}_\perp^\top e_i e_j^\top \widehat{V} ) \\
    \textup{vec}( \widehat{U}_\perp^\top e_i e_j^\top \widehat{V}_\perp )
    \end{bmatrix}, \quad \theta^\ext   = \begin{bmatrix}
        \textup{vec}(\widehat{U}^\top M \widehat{V} ) \\
        \textup{vec}(\widehat{U}^\top M  \widehat{V}_\perp ) \\
        \textup{vec}(\widehat{U}_\perp^\top  M \widehat{V}) \\
        \textup{vec}(\widehat{U}_\perp^\top  M  \widehat{V}_\perp) 
    \end{bmatrix}\!.
\end{align*}
We note that according to \eqref{eq:M_decomposition}, $M_{i,j} = \langle \phi_{i,j}^\ext , \theta^\ext \rangle$. Moreover, one can readily show that, for any $(i,j) \in [m]\times[n]$. 
\begin{align}
    \Vert \phi_{i,j}^\ext\Vert_2 \le 1, \ \ \Vert \theta^\ext \Vert_2 = \Vert M \Vert_\F \le \sqrt{nm} \Vert M \Vert_{\max}. 
    \label{eq:bound_psi_theta}
\end{align}
We can also write $\theta^\ext = \begin{bmatrix} \theta^\top & \theta_ \ell^\top \end{bmatrix}^\top$, where we denote $\theta_\ell = \textup{vec}(\widehat{U}^\top_\perp M \widehat{V}_\perp ) \in \RR^{(m-r)\times(n-r)}$. Observe that 
 \begin{align}
     \Vert \theta_\ell \Vert_2 = \Vert \widehat{U}_\perp^\top  M  \widehat{V}_\perp \Vert_F
     \leq \Vert \widehat{U}_\perp^\top  U\Vert_F \Vert M\Vert_\op \Vert \widehat{V}_\perp^\top  V\Vert_F \leq \Vert UU^\top - \widehat{U}\widehat{U}^\top \Vert_{F} \Vert V V^\top - \widehat{V} \widehat{V}^\top  \Vert_{F}  \sqrt{nm}\Vert M\Vert_{\max}.
     \label{eq:norm_theta_ell}
 \end{align}
 
Thus, one may use guarantees on the subspace recovery in the Frobenius norm such as Lemma \ref{lemma:subspace_Frobenius} to show that the impact of $\theta_\ell$ in the reward is much smaller than that of $\theta$. 
We have with probability at least $1-\delta$:
\begin{align*}
    \Vert \theta_\ell \Vert_2 
    \lesssim   \frac{L^2 \kappa r  (m+n) }{T_1 \omega_{\min} \sigma_r}\log^{3}\left( \frac{(m+n)T_1}{\delta}\right)
\end{align*}
provided that $T_1 \gtrsim \frac{1}{\omega_{\min} (m\wedge n)} \log^3\left( \frac{(m+n)T_1}{\delta}\right)$. 

From these observations, together with the reformulation $M_{i,j} = \langle \phi_{i,j}^\ext , \theta^\ext \rangle$, one can claim a reduction from a low-rank bandit model to an almost low-dimensional linear bandit \citep{valko2014spectral, jun2019bilinear, kocak2020spectral, kang2022efficient}. More precisely, the sub-vector corresponding to the last $(m-r)(n-r)$ components of $\theta$ (i.e., those corresponding to $\theta_\ell$) do not have a significant impact on the rewards, hence reducing the effective dimension of this linear bandit to $d= r(m+n) - r^2$.

We can derive an analogue of Theorem \ref{thm:main_ESRED_context} with a similar regret analysis:

\begin{proposition}
\label{prop:almost_low_d_regret}
    Let $B_2$ be an upper bound of $\Vert \theta^\ext \Vert_2$ and let $B_{\ell}$ be an upper bound of $\Vert \theta_{\ell} \Vert_2$. Define  $T_1 = \widetilde{\Theta} (\mu^2 \kappa r^{3/2}\sqrt{\frac{m+n}{\omega_{\min}\sqrt{mn}}})\sqrt{T})$, and let $\lambda = B_2^{-2}$, $\lambda_\perp = \frac{T}{d\log(1+\frac{T}{\lambda})}$ and:
    \begin{align*}
        \Lambda  = \mathrm{diag}(\overbrace{\lambda,\dots,\lambda}^{d},\overbrace{\lambda_\perp,\dots,\lambda_\perp}^{mn-d}),\qquad  
        \beta =   \sigma\sqrt{2\log \left( \frac{10 Tmn}{\delta}\log(1+T) \right) } + B_2\sqrt{\lambda} + B_\ell \sqrt{\lambda_\perp}.
    \end{align*}
    Then, Algorithm \ref{algo:ESRED} with \textsc{SupLinUCB} parameters $T-T_1,\beta(\delta),\Lambda$ has an expected regret that satisfies $R^\pi_T = \widetilde{O}\left( L \mu\kappa r \frac{\sqrt{m+n}(mn)^{1/4}}{\sqrt{m\rho_{\min}}} \sqrt{T}\right)$. 
\end{proposition}
The proof of the proposition only differs slightly from the proof of Theorem \ref{thm:main_ESRED_context}. Instead of Theorem \ref{thm:mainthm_LowSupLinUCB_misspec}, we combine the analysis of \textsc{SupLinUCB}\footnote{Here we used \textsc{SupLinUCB} as given in \cite{chu2011contextual} for the analysis.} from \cite{chu2011contextual}, Lemma 38 in \cite{kocak2020spectral}, Lemma 3 \cite{jun2019bilinear}, and Lemma 11 in \cite{abbasi2011improved} to show that the second phase of this algorithm has regret as follows:
\begin{align*}
    {\cal R}^{\pi}(T) \lesssim \sqrt{dT} \left( r_{\max} + \sigma \sqrt{\log  \left( eTK/\delta \right) }  + B_2 \sqrt{\lambda} + B_\ell \sqrt{\lambda_\perp} \right)\sqrt{\log\left( 1+T/\lambda\right) \log T}
\end{align*}
with probability at least $1-\delta$. Proposition \ref{prop:almost_low_d_regret} follows straightforwardly from this result and our bounds $B_2,B_{\ell}$ derived above. Note that in the homogeneous setting (see Definition \ref{def:homogeneous_M}), the regret above scales as $\widetilde{O}( (m+n) \sqrt{T})$.

The reduction presented above has the following limitations: 

(i) First, it does not take advantage of the delocalization of the subspace recovery error, i.e., of the fact that this error spreads out along $m+n$ dimensions. This delocalization property is quantified through our subspace recovery error guarantees in the two-to-infinity norm, and it has subtle but important implications. Indeed, it can be leveraged to perform a reduction to an almost low-dimensional linear bandit that is also \emph{almost sparse}, in the sense that $\Vert \theta_\ell \Vert_\infty \lesssim \Vert \theta_\ell \Vert_2/\sqrt{mn}$. A reduction to a \emph{misspecified} linear bandit leverages this structure and allows us to devise algorithms with tighter regret guarantees.

(ii) Then, since the problem can be reduced to an almost-sparse linear bandit, using algorithms such as \textsc{LowDim-OFUL} \citep{jun2019bilinear, kang2022efficient} does not give tight regret guarantees in our setting. Indeed, \textsc{LowDim-OFUL} relies on least squares estimators obtained with a weighted Euclidean norm regularization. This type of regularization accounts for the low-dimensionality of the problem, but fails at exploiting the sparsity of the problem. To this aim, one should use an $\ell_1$-norm regularization instead. In fact, even adapting algorithms for linear bandits such as \textsc{SupLinUCB} \citep{chu2011contextual} that have typically tighter dimension dependencies than \textsc{OFUL} does not yield tighter bounds. As we state in Proposition \ref{prop:almost_low_d_regret}, the corresponding regret upper bound scales as $\widetilde{O}((m+n)\sqrt{T})$ in the homogeneous case, compared to $\widetilde{O}((m+n)^{3/4}\sqrt{T})$ achieved using Algorithm \ref{algo:ESRED} in Section \ref{sec:regret}.

(iii) Finally, a reduction to misspecified linear bandit is more appealing computationally because the resulting feature vectors are of dimension $r(m+n) - r^2$. In the case of a reduction to an almost low-dimensional linear bandit, these vectors are of dimension $mn$. This means that instead of having to invert\footnote{Indeed, almost all existing algorithms for linear and contextual linear bandits require inverting the features matrix during a regression step that involves least squares estimation \citep{lattimore2020bandit}.} $mn \times mn$ matrices, we will only have to invert $(r(m+n)-r^2) \times (r(m+n)-r^2)$ matrices. 

\newpage
\section{Additional discussions}

\subsection{Knowledge of the context distribution}\label{subsec:knowledge-of-context}
We decided to assume knowledge of $\rho$ to simplify the exposition of our work, as dealing with unknown $\rho$ does not add any insight to our work nor does it harm any of our guarantees, as we show in this section.
After $T$ observations, a natural estimator $\hat \rho_T$ of $\rho$ is: 
$$
    \forall i \in [m], \qquad \hat{\rho}_{i,T} = \frac{1}{T} \sum_{t=1}^T \mathds{1}\lbrace i_t = 1 \rbrace. 
$$
Next, an immediate application of Bernstein's inequality gives 
$$
    \mathbb{P}( \vert \hat{\rho}_{i,T} - \rho_i \vert > \varepsilon \rho_i  ) \le 2 \exp\left(  -\frac{T\rho_i^2 \varepsilon^2}{2\sigma_i^2 +  2\varepsilon\rho_i/3 }\right) \le 2 \exp\left( -\frac{T\rho_i^2\varepsilon^2}{2\rho_i +  2\rho_i\varepsilon/3 }\right)  \le 2 \exp\left(  - T\rho_i \left(\frac{ \varepsilon^2}{2 +  2\varepsilon/3 } \right)\right)
$$
 where we used 
 $\sigma_i^2 := \mathbb{E}[ (\mathds{1}\lbrace i_t  = i\rbrace  - \rho_i)^2] = \rho_i - \rho_i^2 \le \rho_i$. Using a union bound and choosing for instance $\varepsilon = 1/2$, we obtain that 
 $$
     \mathbb{P}( \exists i \in [m], \; \vert \hat{\rho}_{i,t} - \rho_i \vert > \rho_{i}/2 ) \le 2\sum_{i=1}^m  \exp\left(  -\frac{3T\rho_{i}}{ 16 }\right) \le 2m \exp\left(  -\frac{3T\rho_{\min}}{ 16 }\right).
$$
Thus, for all $\delta \in (0,1)$, 
$$
 \mathbb{P}( \forall i \in [m], \; \vert \hat{\rho}_{i,t} - \rho_i \vert \le \rho_{i}/2 ) \ge 1- \delta
$$
provided that
$$
T \ge \frac{16}{3\rho_{\min}} \log\left(\frac{2m}{\delta}\right).
$$
In view of the above result, estimating $\rho$ always requires fewer samples than required in Theorem \ref{thm:recovery-two-to-infinity-norm}. To see that, 
 observe that by Lemma \ref{lem:spikiness} and the inequality $\omega_{\min} \leq \rho_{\min}/n$, $$
\frac{L^2 (m+n)}{\sigma_r^2 \omega_{\min}} \log^3\left( \frac{(m+n)T}{\delta}\right) \gtrsim \frac{m+n}{ \omega_{\min}mn} \log\left( \frac{2m}{\delta}\right) \gtrsim \frac{1}{ \rho_{\min}} \log\left( \frac{2m}{\delta}\right).$$  Consequently, $\rho$ can be assumed to be known without loss of generality.
 
\subsection{Low-rank matrix bandits}
Our framework can be applied to address a related and simpler problem known as low-rank matrix bandits. In this scenario, 
the learner must choose two arms, $i_t\in[m]$ and $j_t\in[n]$, at each time step $t\in T$. In contrast, in the contextual low-rank bandit setting, the row $i_t$ is sampled according to a distribution $\rho$ that is not chosen by the learner.

Our algorithms can be straightforwardly adapted to this setting. For instance, for regret minimization, we could utilize a two-phase algorithm. In the first phase, for $t\le T_1$, the arm pair $(i_t,j_t)$ is chosen uniformly at random from $[m]\times [n]$. In the second phase, for $t=T_1+1,\dots,T$, $(i_t,j_t)$ is selected by \textsc{SupLinUCB} with a fixed set of feature vectors ${\cal X}=\{ \phi(i,j), i\in [m],j\in [n]\}$ where $\phi(i,j)$ is constructed from the estimated singular subspaces according to \eqref{eq:phi}. The parameters used in  \textsc{SupLinUCB} are the same as for contextual low-rank bandits, except that we set $\omega_{\min} = 1/(mn)$. 

Similarly to Theorem \ref{thm:main_ESRED_context}, it can be proven that such an algorithm would achieve a regret of at most $\widetilde{O}(L\mu^2 \kappa^2 r^{5/4}  (m+n)^{3/4}(mn)^{1/4} \sqrt{T /(m\wedge n)})$ under an appropriate condition on $T$ (see Theorem \ref{thm:main_ESRED_context}). When $M$ is homogeneous, the regret scales as $\widetilde{O}(L(m+n)^{3/4} \sqrt{T})$.


Most existing work on low-rank bandits lacks minimax regret guarantees. \cite{katariya2017stochastic} introduced \textsc{Rank1Elim}, achieving regret $\widetilde{O}((m+n)\log(T)/\tilde{\Delta}_{\min})$\footnote{Here $\tilde{\Delta}_{\min}$ corresponds to a notion of minimum gap different from the standard one - see \cite{katariya2017stochastic}.} for a rank-1 reward matrix. \cite{kveton2017stochastic} extended this result to rank $r$, with additional strong assumptions on the reward matrix in their \textsc{LowRankElim} algorithm. \cite{trinh2020solving} provided the first algorithm with asymptotically optimal instance-dependent regret for rank-$1$ bandits. 
\cite{bayati2022speed} proposed an algorithm enjoying a minimax regret guarantee, but their upper bounds scale at least with $mn$ through the constant $C_2$ in the worst case (see Theorem 2 in \citep{bayati2022speed}), and requires tuning a filtering resolution $h$, making their guarantees hard to evaluate. \cite{stojanovic2024spectral} presented \textsc{SME-AE} algorithm, leveraging the entry-wise matrix estimation guarantees to obtain a regret upper bound of order $\widetilde{O}((m+n)(\bar{\Delta}/\Delta_{\min}^2)\log^3(T))$ where $\bar{\Delta}$ is the average reward gap, and $\Delta_{\min}$ is the minimum reward gap.

While these gap-dependent bounds require strong assumptions, e.g., \cite{stojanovic2024spectral} assume that $\Delta_{\max}/\Delta_{\min} \le \zeta$ where $\Delta_{\max}$ is the maximum reward gap, we provide the first algorithm with a minimax regret of order $\widetilde{O}((m+n)^{3/4} \sqrt{T})$ without such assumptions.




\subsection{Improved max-norm guarantees}\label{subsec:max-norm-improvements}

Our two-phase algorithm structure also yields a generic method to estimate a low-rank matrix in more general settings. The low-rank matrix estimator $\widebar{M}$ defined by $\widebar{M}_{i,j}=\phi_{i,j}^T\hat{\theta}$ (see \eqref{eq:phi} and \eqref{eq:rspe-lse} for a definition of these terms) enjoys a sharper max-norm guarantee than the standard estimator $\widehat{M}$ defined in \eqref{eq:SVD_hat_M}. Indeed, recall that from Proposition \ref{prop:dsm-max-error}, $$\Vert M-\widehat{M}\Vert_{\max} \lesssim L\sqrt{\frac{\mu^6\kappa^4r^3(m+n)}{T \omega_{\min}(m\wedge n)^2} \log^3\left( \frac{(m+n)T}{\delta}\right)}$$ with probability larger than $1-\delta$ when $T \gtrsim \frac{L^2\kappa^2(m+n)}{\sigma_r^2\omega_{\min}}\log^3\left(\frac{(m+n)T}{\delta}\right)$.

Additionally, from \eqref{eq:sharp-minimax-bound-max-norm},  $$\Vert M-\widebar{M}\Vert_{\max} \lesssim \sigma \mu\sqrt{\frac{r\left(m+n\right)}{T\omega_{\min}mn}\log\left(\frac{m+n}{\delta} 
 \right)}+\widetilde{O}\left(\frac{K}{T}\right)$$ with probability larger than $1-\delta$ under the assumptions of Proposition \ref{prop:rs-pe-random-error}, where $K$ depends polynomially on the model parameters. When $T$ is large enough to ensure that the first term in the error bound is larger than the second, the scaling in $\mu$ and $r$ of the max-norm error bound is improved from $\mu^{3}r^{3/2}$ to $\mu r^{1/2}$. Furthermore, the dependency in $\kappa$ and $\Vert M\Vert_{\max}$ is removed in the first term (recall that $L=\Vert M \Vert_{\max} \vee \sigma$). This may appear suprising, but note that the term $\widetilde{O}\left(\frac{K}{T}\right)$ still depends on $\kappa$ and $\Vert M\Vert_{\max}$ through $K$: more samples are required to achieve the stated scaling when $\kappa$ and $\Vert M\Vert_{\max}$ are large.

\subsection{Removing the dependence in $\Vert M\Vert_{\max}$ in the regret upper bound}
\label{app:debiased_regret}

We assume that $m= \Theta(n)$ for simplicity as in \cite{chen2020noisy}. Let $\widehat{M} := Z_{\mathrm{cvx},r}$ from Theorem 2 in \cite{chen2020noisy} and let $\widehat{M} = \widehat{U} \widehat{\Sigma} \widehat{V}^\top$. Then, we have:

\begin{lemma}
    Under the condition $\Vert M - \widehat{M} \Vert_\op \leq \sigma_r(M)/4$, it holds that:
\begin{align*}
    \max\left( \Vert U - \widehat{U} (\widehat{U}^\top U)\Vert_{2 \to \infty}, \Vert V - \widehat{V} (\widehat{V}^\top V)\Vert_{2 \to \infty} \right)\lesssim \frac{1}{\sigma_r(M)} \left( \sqrt{n} \Vert M - \widehat{M} \Vert_{\max} + \mu\sqrt{\frac{r}{n}} \Vert M - \widehat{M} \Vert_\op  \right)
\end{align*}
\label{lemma:vectors_perturbation_from_max}
\end{lemma}

\begin{proof} 
    Using repeatedly that $\Vert AB\Vert_{2\to\infty} \leq \Vert A\Vert_{2\to\infty} \Vert B\Vert_\op$ for appropriate $A,B$, we obtain:
    \begin{align*}
    \Vert U - \widehat{U} \widehat{U}^\top U \Vert_{2\to\infty} &= \Vert M V\Sigma^{-1} - \widehat{U} \widehat{U}^\top U\Vert_{2\to\infty} \leq \frac{1}{\sigma_r(M)} \Vert MV - \widehat{U} \widehat{U}^\top U \Sigma \Vert_{2\to\infty} \\
    &\leq  \frac{1}{\sigma_r(M)} \Vert M V - \widehat{M}V \Vert_{2\to\infty} + \frac{1}{\sigma_r(M)} \Vert \widehat{M}V - \widehat{U} \widehat{U}^\top U \Sigma \Vert_{2\to\infty}
\end{align*}
Note that the first term is bounded by $\sqrt{n}\Vert M - \widehat{M} \Vert_{\max}$ since $\Vert M V - \widehat{M}V \Vert_{2\to\infty} \leq \Vert M - \widehat{M} \Vert_{2\to\infty}\Vert V \Vert_\op \leq \sqrt{n}\Vert M - \widehat{M} \Vert_{\max}$. Before we bound the second term, note that:
\begin{align*} 
    \widehat{U} \widehat{U}^\top U \Sigma = \widehat{U} \widehat{U}^\top MV  =  \widehat{U} \widehat{U}^\top (M-\widehat{M})V + \widehat{U} \widehat{U}^\top \widehat{M} V
\end{align*}
and $\widehat{U} \widehat{U}^\top \widehat{M} V = \widehat{U} \widehat{\Sigma} \widehat{V}^\top V = \widehat{M}V$.
Thus:
\begin{align*}
    \Vert \widehat{M}V - \widehat{U} \widehat{U}^\top U \Sigma \Vert_{2\to\infty} = \Vert \widehat{U} \widehat{U}^\top (M-\widehat{M})V \Vert_{2\to\infty} \leq \Vert  \widehat{U} \Vert_{2\to\infty} \Vert M-\widehat{M} \Vert_\op.
\end{align*}
Lastly, we bound $\Vert  \widehat{U} \Vert_{2\to\infty}$ as follows:
\begin{align*}
    \Vert  \widehat{U} \Vert_{2\to\infty} \leq \Vert  \widehat{U} \widehat{U}^\top U \Vert_{2\to\infty} \Vert ( \widehat{U}^\top U)^{-1} \Vert_\op \leq (\Vert U\Vert_{2\to\infty} + \Vert U - \widehat{U} \widehat{U}^\top U \Vert_{2\to\infty}) \Vert ( \widehat{U}^\top U)^{-1} \Vert_\op
\end{align*}
and similarly to the equation between (49) and (50) in \cite{stojanovic2024spectral} we have:
\begin{align*}
    \Vert ( \widehat{U}^\top U)^{-1} \Vert_\op = \frac{1}{\sigma_{r}(\widehat{U}^\top U)} \leq \frac{1}{1-\Vert \widehat{U}^\top U - \mathrm{sgn}(\widehat{U}^\top U)\Vert_\op } \leq \frac{1}{1- \frac{2\Vert M-\widehat{M}\Vert_\op^2}{\sigma_r^2(M)}}
\end{align*}
Under the condition $\Vert M - \widehat{M} \Vert_\op\leq \sigma_r(M)/4$ we consequently have:
\begin{align*}
    \Vert U - \widehat{U} \widehat{U}^\top U \Vert_{2\to\infty} \lesssim \frac{1}{\sigma_r(M)} \left( \Vert M - \widehat{M}\Vert_{2\to\infty} + \Vert U\Vert_{2\to\infty} \Vert M - \widehat{M}\Vert_\op  \right)
\end{align*}
\end{proof}

Then, as a consequence of Lemma \ref{lemma:vectors_perturbation_from_max} and Theorem 2 from \cite{chen2020noisy}, we obtain analogously to Theorem \ref{thm:recovery-two-to-infinity-norm}: 
\begin{corollary}
    Let us denote $\epsilon_{\textup{Sub-Rec}} := \max( d_{2\to\infty}(U,\widehat{U}), d_{2\to\infty}(V,\widehat{V}))$. For any $\delta \in (0,1)$, the following event: 
    \begin{align*}
    \epsilon_{\textup{Sub-Rec}} = \widetilde{O} \left( \frac{n\sigma}{\sigma_r} \frac{1}{\sqrt{T}} \mu^3 \kappa^{5/2} r^{3/2} \right)
    \end{align*}
    holds with probability $\geq 1-O(n^{-3})$, provided that 
    $
    T \gtrsim \kappa^4 \mu^2 r n \log n \left( \mu^2 r \log^2 n + \frac{\sigma^2 n^2}{\sigma_r^2} \right).
    $ 
    \label{corr:chen_epsilon_sub}
\end{corollary}
Considering the setting of Corollary \ref{corr:chen_epsilon_sub}, the bound from Theorem \ref{thm:recovery-two-to-infinity-norm} scales with $\widetilde{O} \left( \frac{nL}{\sigma_r} \frac{1}{\sqrt{T}} \mu  \kappa r^{1/2} \right)$. Note that even though the dependence on $L=\Vert M \Vert_{\max} \vee \sigma$ is reduced to $\sigma$ in Corollary \ref{corr:chen_epsilon_sub}, the scaling in $\mu,\kappa$ and $r$ is worse than in Theorem \ref{thm:recovery-two-to-infinity-norm}.

To derive the corresponding regret bound, we can rewrite the requirement on the number of samples from Corollary \ref{corr:chen_epsilon_sub} as $T_1 \gtrsim  n \left(1+ \frac{\sigma^2}{\Vert M\Vert_{\max}^2}\right) \mathrm{poly}(\kappa,\mu,r,\log n)$. Next, similarly to \eqref{eq:T1_def_regret_app} in the proof of Theorem \ref{thm:main_ESRED_context} we obtain that the regret is minimized for:
\begin{align*}
    T_1^\star \gtrsim \frac{\sigma}{\Vert M\Vert_{\max}}\sqrt{T} n^{3/4} \mathrm{poly}(\kappa,\mu,r,\log n)
\end{align*}
Combining this with the requirement on $T_1$, we set:
\begin{align*}
    T_1 := \max \left\{\frac{\sigma}{\Vert M\Vert_{\max}}\sqrt{T} n^{3/4}, n \left(1+ \frac{\sigma^2}{\Vert M\Vert_{\max}^2}\right) \right\} \mathrm{poly}(\kappa,\mu,r,\log n)
\end{align*}

Next, following the arguments from the proof of Theorem \ref{thm:main_ESRED_context} in Appendix \ref{subsec:appendix_proof_main_regret} we obtain that combining our analysis with the guarantee from Corollary \ref{corr:chen_epsilon_sub} achieves:
\begin{align*}
     R^\pi(T)  \lesssim \max \left\{\sigma \sqrt{T} n^{3/4}, n \left(\Vert M\Vert_{\max}+ \sigma \frac{\sigma}{\Vert M\Vert_{\max}}\right) \right\} \mathrm{poly}(\kappa,\mu,r,\log n)  ,
\end{align*}
Comparing this result to the regret upper bound from Theorem \ref{thm:main_ESRED_context}, namely
\begin{align*}
    R^\pi(T)  = \widetilde{O}\left( (\Vert M\Vert_{\max}\vee \sigma )  \mu^2 \kappa^2 r^{5/4}  n^{3/4} \sqrt{T } \right),
\end{align*}
we see that in the regime $\sigma \lesssim \Vert M\Vert_{\max}$, the regret upper bound of Theorem \ref{thm:main_ESRED_context} scales with $\Vert M\Vert_{\max} n^{3/4} \sqrt{T}$, whereas the regret upper bound obtained by using Corollary \ref{corr:chen_epsilon_sub} scales with $\sigma n^{3/4} \sqrt{T} + \Vert M\Vert_{\max} n$ up to $\mathrm{poly}(\kappa,\mu,r,\log n)$ terms.

Finally, we note that we could also use the improved 
max-norm bound recalled in Section \ref{subsec:max-norm-improvements} to estimate the subspace recovery error related to $\widebar{M}$
and that, for $T$ large enough, it would have a milder dependence in $\mu, \kappa$ and $r$ compared to Corollary \ref{corr:chen_epsilon_sub}.
\newpage
\section{Numerical experiments}\label{app:experiments}

\textbf{Notation.} In this appendix, we denote by $\widebar{M}$ the reward matrix estimator that \rspe\ and \rsbpi\ leverage. It is defined by $$\widebar{M}_{i,j}=\phi_{i,j}^T\hat{\theta}$$ for $\left(i,j\right) \in [m] \times [n]$. We also denote by \dsmbpi\ the algorithm that outputs the policy $\hat{\pi}_{ \dsmbpi}$ defined by $\hat{\pi}_{\dsmbpi}\left(i\right)=\arg \max_{1 \leq j \leq n} \widehat{M}_{i,j}. $
We finally note that \dsmpe\ (resp. \dsmbpi) is denoted by $\operatorname{DSM-PE}$  (resp. $\operatorname{DSM-BPI}$) on the graphs.

We perform numerical experiments on synthetic data with a uniform context distribution. Unless specified otherwise, the behavior policy is uniform, the target policy is chosen as the best policy: $\pi(i)=\arg \max_{j} M_{i,j}$ (ties are broken arbitrarily), and we generate noisy entries $M_{i_t,j_t}+\xi_t$ where $\xi_t \sim \mathcal{N}\left(0,1\right)$ is standard Gaussian, and where $M = PDQ$ for two invertible matrices $P \in \mathbb{R}^{m \times m}$, $Q \in \mathbb{R}^{n \times n}$, and $D \in \mathbb{R}^{m \times n}$ defined by $D_{i,j}=\indicator_{i=j}\indicator_{i \leq r}$ (note that $M$ is consequently of rank $r$). $P$ and $Q$ are initially generated at random with uniform entries in $\left[0,1\right]$ and their diagonal elements are replaced by the sum of the corresponding row to ensure invertibility. All experiments are performed $50$ times and the shaded regions are the corresponding $\left[5\%,95\%\right]$ confidence intervals\footnote{The code used in the experiments can be accessed at \url{https://github.com/wilrev/LowRankBanditsTwoToInfinity}.}.




\subsection{Policy Evaluation}\label{subsec:pe-experiments}

\subsubsection{Choice of hyperparameters}

\textbf{Experiment 1: Impact of data splitting.}
For a regularization parameter of $\tau=10^{-4}$, we compare the performance of \rspe\ for $\alpha \in \left\{1/5,1/2,4/5\right\}$, where $\alpha$ is the proportion of samples used in the first phase of the algorithm. We also plot the error of \rspe\ when no data splitting is performed, that is to say when all of the samples are used in both phases of the algorithm. 
Finally, we plot the dominant term $\frac{\Vert \psi_{\pi}\Vert _{2}}{\sqrt{\omega_{\min}}}\sqrt{\frac{2\log(16/\delta)}{\left(1-\alpha\right)T}}$ in the instance-dependent upper bound on the PE error of \rspe\ for $\alpha=4/5$ and $\delta=10^{-2}$ (see Theorem \ref{thm:rs-pe-error-full}), which approaches its true error bound when $T$ grows large. 
Note that consequently, the error of \rspe\ can be below this dominant term. Additionally, the scaling in $1/\sqrt{1-\alpha}$ of this term suggests that using more samples in the second phase should yield better asymptotic performance. The results are presented in Figure \ref{fig:split}. 

\begin{figure}[H]
    \centering
    \includegraphics[width=0.66\columnwidth]{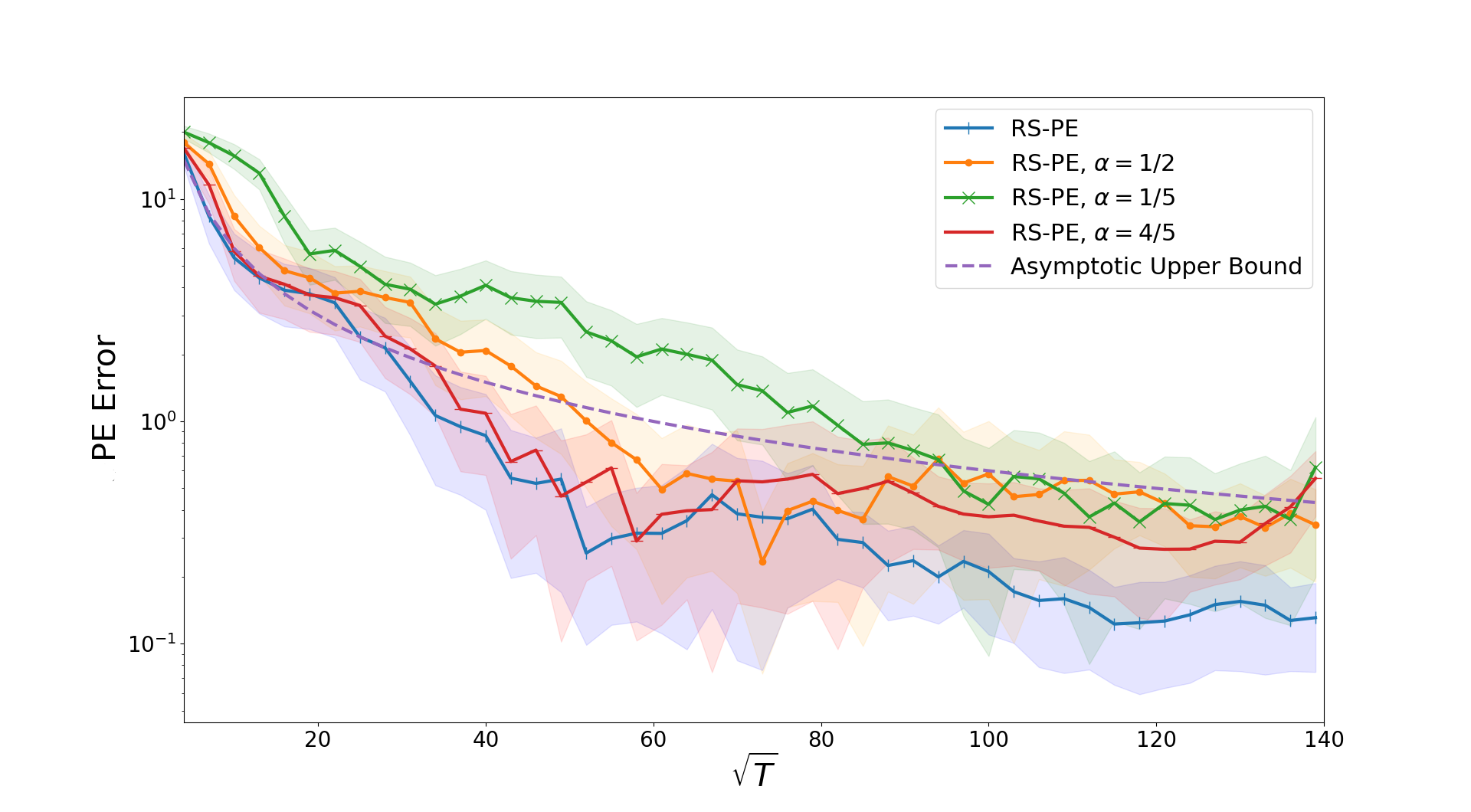}
    \caption{Impact of data splitting ($m=n=50$, $r=2$).}
    \label{fig:split}
\end{figure}

Empirically, using more samples in the first phase appears to yield faster initial performance, but this advantage is negated when the number of samples grows large. Finally, although this is not justified theoretically, not splitting the data (i.e., using the entire dataset for both phases) appears empirically more efficient.

\textbf{Experiment 2: Impact of the regularization parameter $\tau$.} 
We compare the performance of \rspe\ without data splitting for different values of the regularization parameter of $\hat{\Lambda}_{\tau}$, namely $\tau \in \left\{10^{-4},10^{-2},10^{-1}\right\}$. 
The results are presented in Figure \ref{fig:regu}. 

\begin{figure}[H]
    \centering
    \includegraphics[width=0.66\columnwidth]{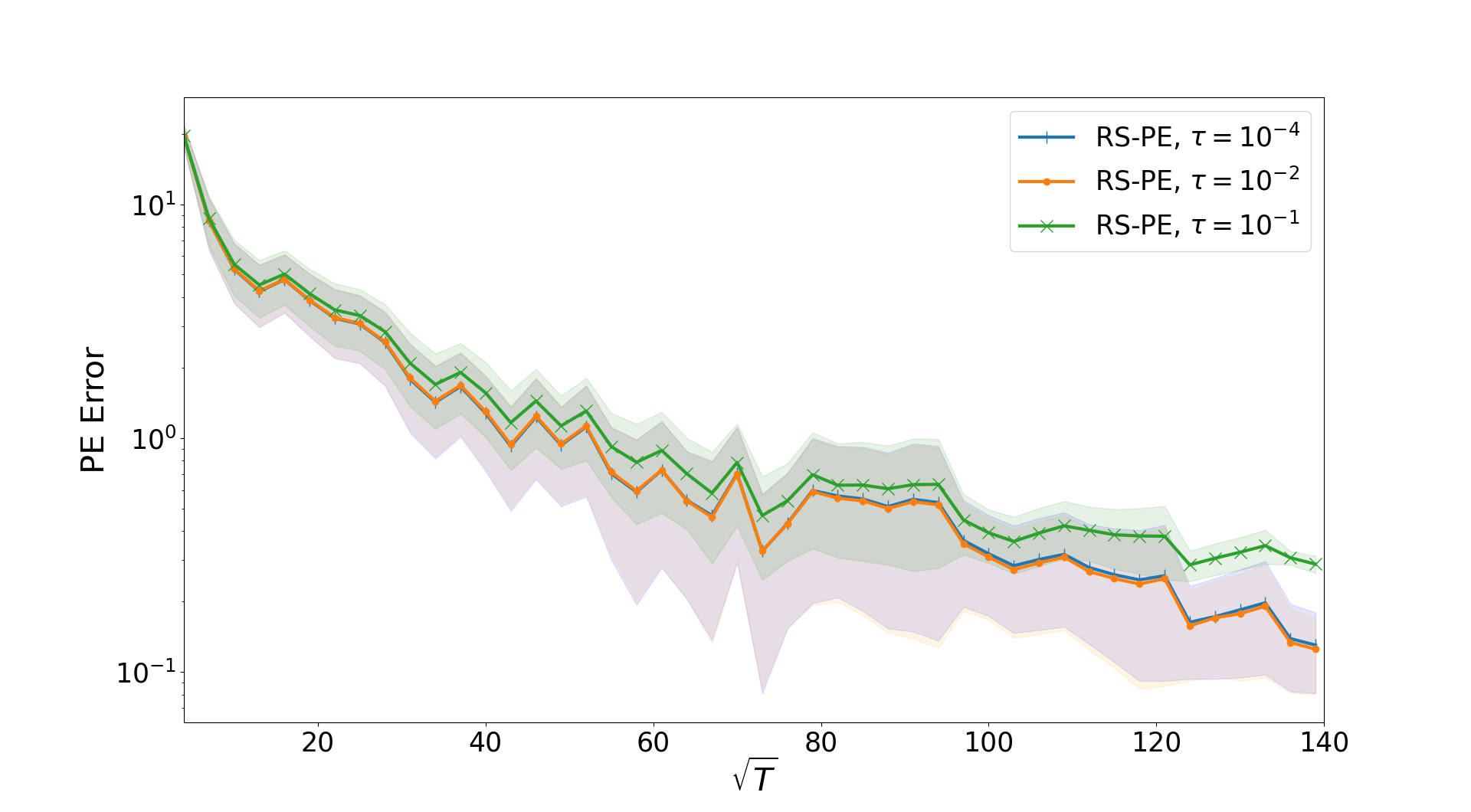}
    \caption{Impact of regularization ($m=n=50$, $r=2$)}
    \label{fig:regu}
\end{figure}
We note that the regularization parameter does not appear to impact the error significantly as long as it is chosen small enough.

Given the previous results, in all future experiments, we do not perform data splitting: all samples are used in both phases of the algorithms. Furthermore the regularization parameter of $\hat{\Lambda}_{\tau}$ will be chosen small: $\tau=10^{-4}$. 

\subsubsection{Comparison of \rspe\ with benchmark estimators.}

\textbf{Experiment 3: Scaling of the PE error with the sample size.}
We compare the \rspe\ estimator $\hat{v}_{\rspe}=\sum_{i,j} w^{\pi}_{i,j} \widebar{M}_{i,j}$ with the \dsmpe\ estimator $\hat{v}_{\dsmpe}=\sum_{i,j} w^{\pi}_{i,j} \widehat{M}_{i,j}$ and the \ips\ estimator \cite{wang2017optimal} which can also be defined as $\hat{v}_{\ips}=\sum_{i,j} w^{\pi}_{i,j} \widetilde{M}_{i,j}$. Similarly to \rspe\ and \dsmpe, \ips\ can be interpreted as a plug-in estimator, i.e., it only relies on an estimator of the reward matrix. It is thus a suitable benchmark to demonstrate the efficiency of the reward matrix estimator $\widebar{M}$ for the PE task. We also plot the asymptotic upper bound $\frac{\Vert \psi_{\pi}\Vert _{2}}{\sqrt{\omega_{\min}}}\sqrt{\frac{2\log(16/\delta)}{T}}$ of \rspe\ suggested by Theorem \ref{thm:rs-pe-error-full} for $\delta=10^{-2}$. More specifically, this is the value of the dominant term in the upper bound for $\alpha=0$, which one would expect to match the asymptotic behavior of the error of \rspe\ when no data splitting is performed. The results are presented in Figure \ref{fig:PE_comparison}.

\begin{figure}[H]
    \centering
    \includegraphics[width=0.66\columnwidth]{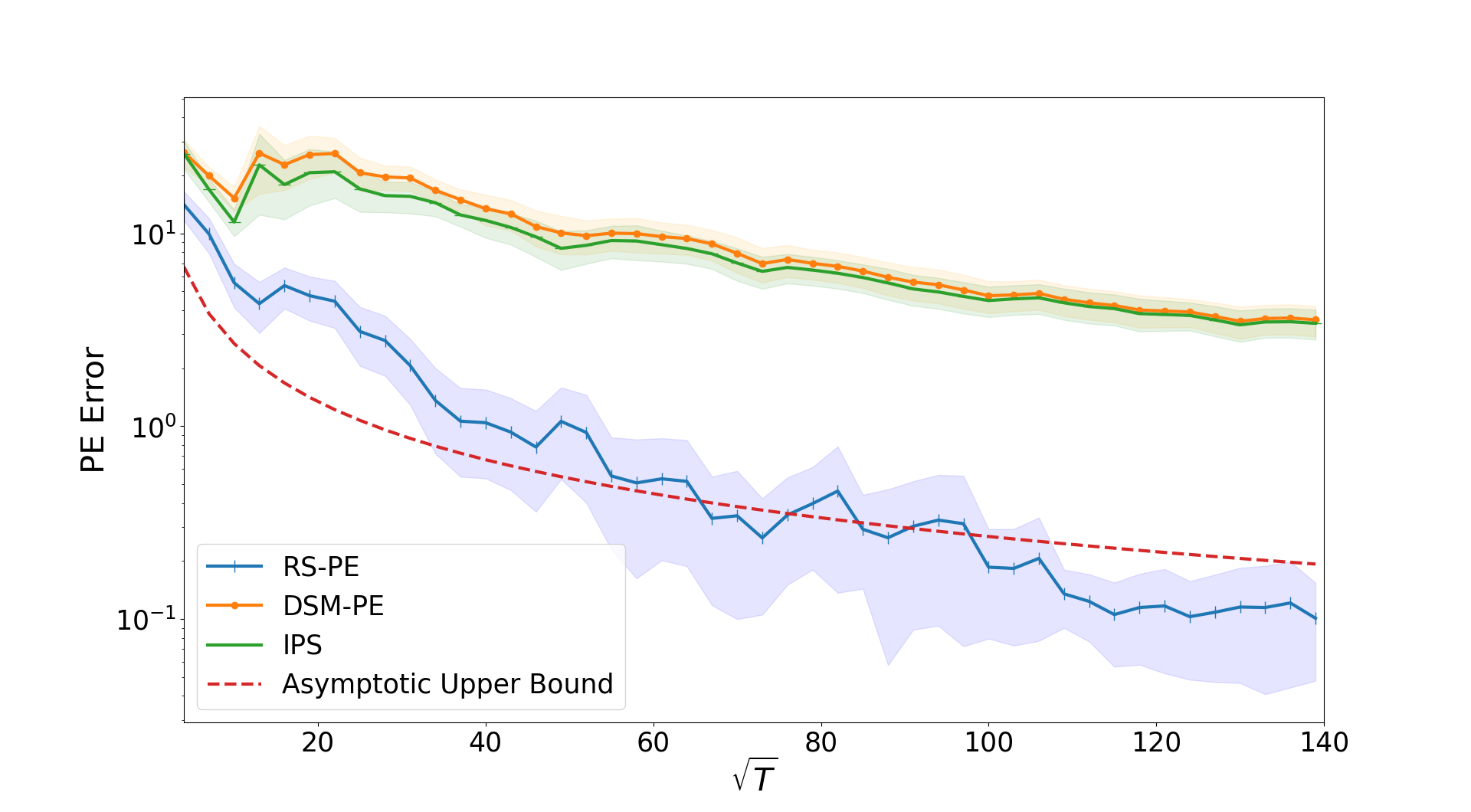}
    \caption{Sample size scaling of the PE error ($m=n=50$, $r=2$)}
    \label{fig:PE_comparison}
\end{figure}

We note that the \rspe\ estimator outperforms the \dsmpe\ and \ips\ estimators by a significant margin. The \dsmpe\ estimator appears to perform comparably to the \ips\ estimator, which is surprising since the latter does not leverage the low-rank structure. Furthermore, despite the restrictive theoretical condition on the number of samples to ensure that the higher-order term in the bound of Theorem  \ref{thm:rs-pe-error-full} is negligible, the asymptotic upper bound appears to closely match the behavior of the \rspe\ error when the number of samples is reasonably large. We finally note that the confidence intervals for \rspe\ only appear larger because of the logarithmic scale on the y-axis.


\textbf{Experiment 4: Scaling of the PE error with the matrix size.}
Note that when $m=n$,  Theorems \ref{thm:dsm-pe-error} and \ref{thm:rs-pe-error} suggest that the PE error of \rspe\ and \dsmpe\ scale with $\sqrt{m}$, while the error bounds of estimators that do not leverage the low-rank structure would typically scale with $m$ \cite{yin2020asymptotically}. We perform an experiment to determine if the matrix size scaling of the error is also improved experimentally. To isolate the dependency in the size of the matrix,  we do not generate a random reward matrix for each $m$. Rather, we ensure that the reward matrix retains the same incoherence parameter, condition number and max-norm for each $m$ by choosing $M \in \mathbb{R}^{m \times m}$ defined by $M_{i,j}=1$ for all $i,j \in [m].$ 
 Furthermore, for this choice of $M$ and our policy choices, it can be checked that the instance-dependent term in the \rspe\ guarantee of Theorem \ref{thm:rs-pe-error} is simply $\frac{\Vert \psi_{\pi}\Vert }{\sqrt{\omega_{\min}}}=\sqrt{m}$, which further suggests that we should expect the \rspe\ error to scale with $\sqrt{m}$ for this particular instance. The PE error of \rspe, \dsmpe\ and \ips\ are plotted as a function of $m \in \left\{1,\dots,300\right\}$ on a log-log scale for $T=10 \ 000$. The results are presented in Figure \ref{fig:PE_error_scaling}.

\begin{figure}[H]
    \centering
    \includegraphics[width=0.66\columnwidth]{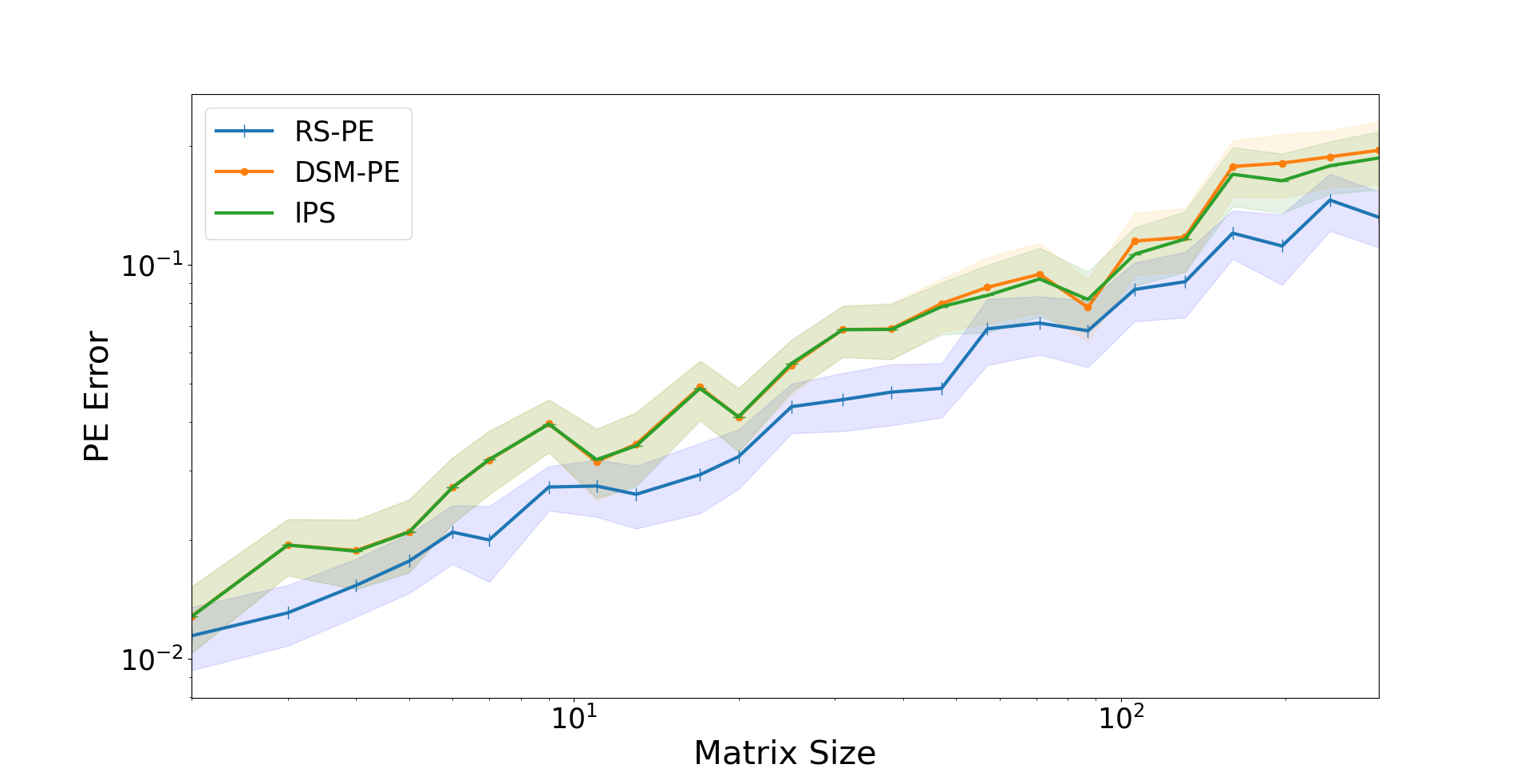}
    \caption{Matrix size scaling of the PE error ($T=10 \ 000$, $r=1$)}
    \label{fig:PE_error_scaling}
\end{figure}

 Surprisingly, all PE estimators appear to have a comparable scaling in $m$. Nonetheless, \rspe\ outperforms the two other estimators for every $m$, even though we have chosen a reward matrix for which $\Vert M\Vert_{\max}=\mu=\kappa=r=1$, so that the error guarantee of \dsmpe\ matches the \rspe\ one up to logarithmic factors. 


\textbf{Experiment 5: Scaling of the max-norm error with the matrix size.} 
Similarly to the PE error, when $m=n$, the max-norm error bounds of $\widehat{M}$ and $\widebar{M}$ summarized in Section \ref{subsec:max-norm-improvements} scale with $\sqrt{m}$. In contrast, the error bounds of matrix estimators that do not leverage the low-rank structure would typically scale with $m$. To determine if the scaling in the size of the matrix is also improved experimentally, we retain the same setting as Experiment 4, but we instead plot $\Vert M-\widebar{M}\Vert_{\max}$, $\Vert M-\widehat{M}\Vert _{\max}$ and $\Vert M-\widetilde{M}\Vert _{\max}$ as a function of $m$. The results are presented in Figure \ref{fig:max_norm_scaling}.

\begin{figure}[H]
    \centering
    \includegraphics[width=0.66\columnwidth]{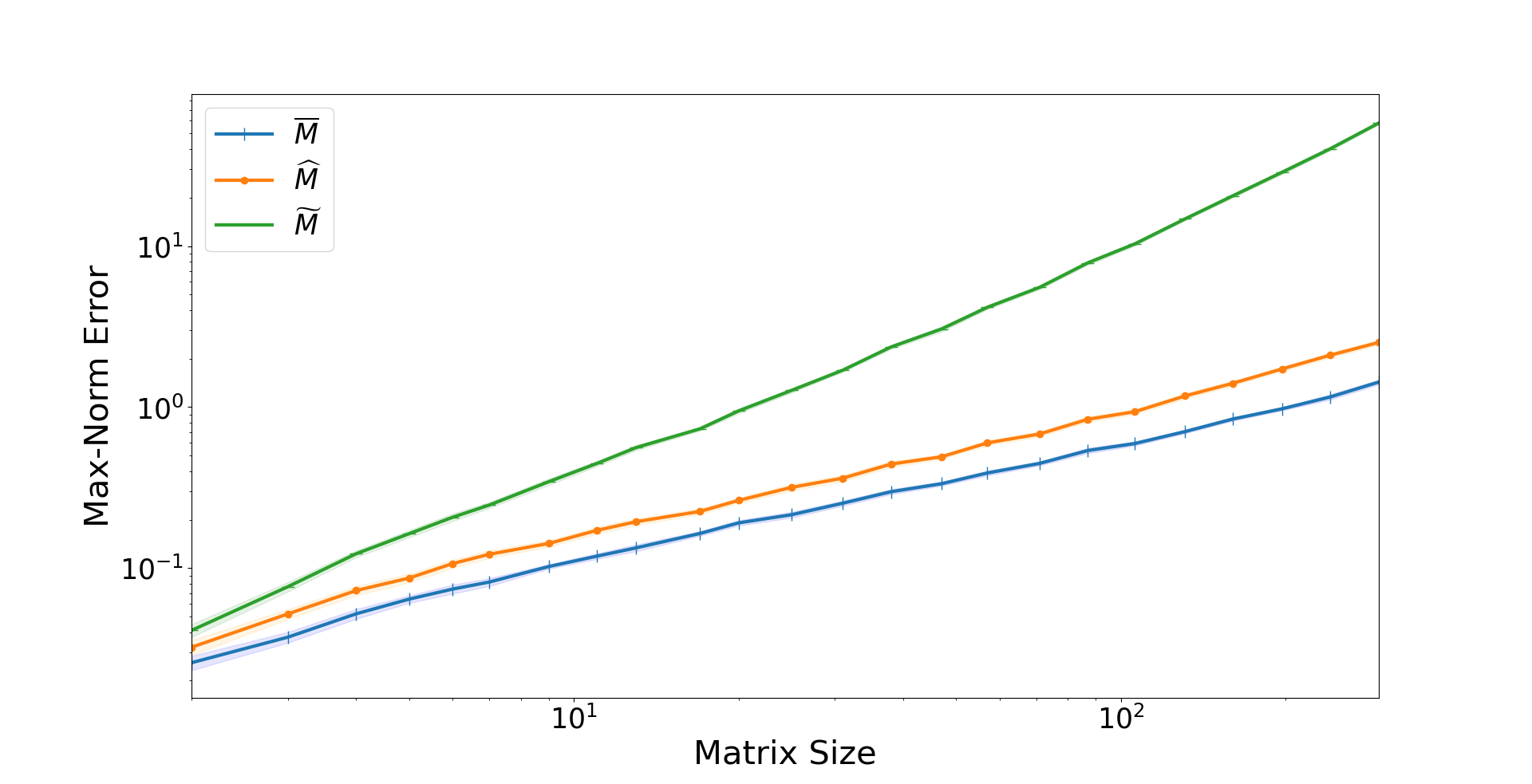}
    \caption{Matrix size scaling of the max-norm error ($T=10 \ 000$, $r=1$)}
    \label{fig:max_norm_scaling}
\end{figure}

We note that $\widehat{M}$ and $\widebar{M}$, the reward estimators that leverage the low-rank structure, have a max-norm error that scales noticeably better than $\widetilde{M}$. Furthermore, $\widebar{M}$ consistently outperforms $\widehat{M}$, even though their max-norm error bounds match up to logarithmic factors for our reward matrix choice (see Appendix \ref{subsec:max-norm-improvements} for a discussion on this matter).




\subsection{Best Policy Identification}\label{subsec:bpi-experiment}

\textbf{Experiment 6 : Comparison of \rsbpi\ with benchmark algorithms.}
We compare the value of the policy learned by \rsbpi, \dsmbpi, and a benchmark algorithm that corresponds to \dsmbpi\ without taking the rank-$r$ approximation of the estimated matrix. Specifically, this benchmark outputs the policy $\hat{\pi}$ defined by $\hat{\pi}\left(i\right)=\arg \max_{1 \leq j \leq m} \widetilde{M}_{i,j}$ for all $i \in [n].$ The results are presented in Figure \ref{fig:BPI_comparison}.

\begin{figure}[H]
    \centering
    \includegraphics[width=0.66\columnwidth]{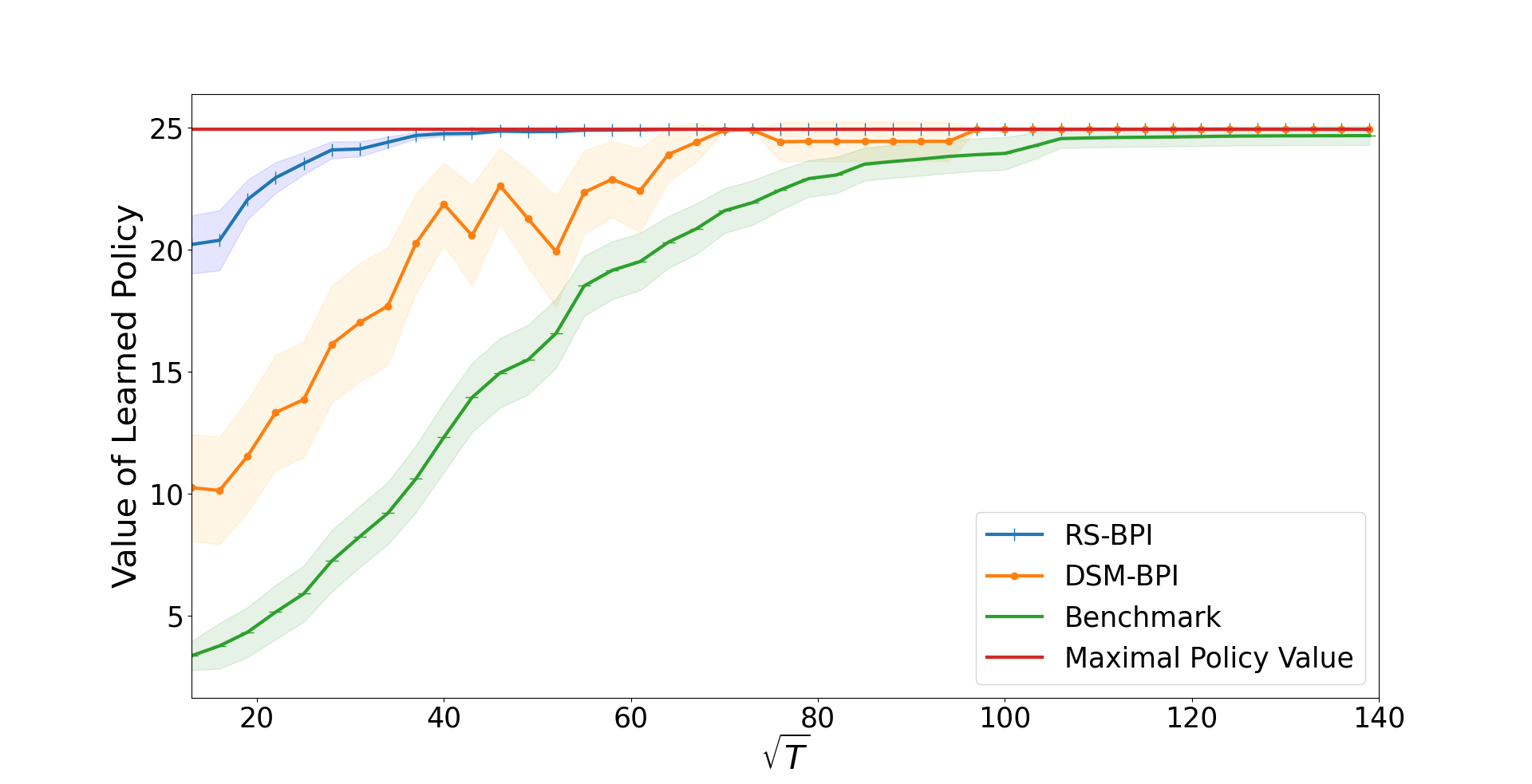}
    \caption{Value of the learned policy against the maximal policy value ($m=n=50, r=2$)}
    \label{fig:BPI_comparison}
\end{figure}

Both low-rank algorithms display improved performance compared to the benchmark, yet the value of the policy learned by \rsbpi\ appears to converge much quicker towards the value of the best policy than the ones of \dsmbpi\ and the benchmark algorithm. 

\end{document}